\pdfoutput=1
\documentclass{mscs}

\newif\ifexternalizetikz

\usepackage[utf8]{inputenc}
\usepackage[T1]{fontenc}
\usepackage{textcomp}
\usepackage{lmodern}
\usepackage{exscale} 

\usepackage{graphicx,color}
\usepackage{etoolbox}
\usepackage{xparse}
\usepackage{xspace}
\usepackage{environ}

\usepackage{latexsym}
\usepackage{amssymb,amsmath}
\usepackage{mathtools}
\usepackage{stmaryrd}
\usepackage{cmll}
\usepackage{bm}
\usepackage{bussproofs}
\usepackage{nicefrac}

\usepackage{microtype}
\usepackage{comment} 
\usepackage{wrapfig}
\usepackage{listings} 
\usepackage{fancybox}

\usepackage{tikz}
\usetikzlibrary{shapes,shapes.geometric}
\usetikzlibrary{shapes.gates.ee,shapes.gates.ee.IEC}
\usetikzlibrary{calc}
\usetikzlibrary{positioning}
\usetikzlibrary{cd}
\usetikzlibrary{decorations.pathmorphing}

\ifexternalizetikz
\usepackage{tikzcdx}
\usetikzlibrary{external}
\makeatletter
\newcommand{\tikzextname}[1]{%
\tikzset{external/figure name={\tikzexternal@realjob-#1-}}}
\makeatother
\fi
\makeatletter
\providecommand{\tikzextname}[1]{\@bsphack\@esphack}
\makeatother

\allowdisplaybreaks[1] 

\usepackage[all,2cell]{xy}
\UseAllTwocells
\SelectTips{cm}{}  
\newdir{pb}{:(1,-1)@^{|-}}
\newcommand{\pb}[1]{\save[]+<17 pt,0 pt>:a(#1)\ar@{pb{}}[]\restore}

\renewcommand{\arraystretch}{1.3}
\newtheorem{theorem}{Theorem}[section]
\newtheorem{proposition}[theorem]{Proposition}

\newtheorem{corollary}[theorem]{Corollary}

\newtheorem{definition}[theorem]{Definition}
\newtheorem{remark}[theorem]{Remark}
\newtheorem{example}[theorem]{Example}

\newenvironment{myproof}{\begin{proof}}{\end{proof}}
\newcommand{\qed}{\hspace*{\fill}\usebox{\proofbox}}

\newlength\stateheight
\newlength\minimumstatewidth
\tikzset{width/.initial=\minimummorphismwidth}
\tikzset{colour/.initial=white}

\newif\ifblack\pgfkeys{/tikz/black/.is if=black}
\newif\ifwedge\pgfkeys{/tikz/wedge/.is if=wedge}
\newif\ifvflip\pgfkeys{/tikz/vflip/.is if=vflip}
\newif\ifhflip\pgfkeys{/tikz/hflip/.is if=hflip}
\newif\ifhvflip\pgfkeys{/tikz/hvflip/.is if=hvflip}

\pgfdeclarelayer{foreground}
\pgfdeclarelayer{background}
\pgfsetlayers{background,main,foreground}

\def\thickness{0.4pt}

\makeatletter
\pgfkeys{%
  /tikz/on layer/.code={
    \pgfonlayer{#1}\begingroup
    \aftergroup\endpgfonlayer
    \aftergroup\endgroup
  },
  /tikz/node on layer/.code={
    \gdef\node@@on@layer{%
      \setbox\tikz@tempbox=\hbox\bgroup\pgfonlayer{#1}\unhbox\tikz@tempbox\endpgfonlayer\pgfsetlinewidth{\thickness}\egroup}
    \aftergroup\node@on@layer
  },
  /tikz/end node on layer/.code={
    \endpgfonlayer\endgroup\endgroup
  }
}
\def\node@on@layer{\aftergroup\node@@on@layer}
\makeatother

\makeatletter
\pgfdeclareshape{state}
{
    \savedanchor\centerpoint
    {
        \pgf@x=0pt
        \pgf@y=0pt
    }
    \anchor{center}{\centerpoint}
    \anchorborder{\centerpoint}
    \saveddimen\overallwidth
    {
        \pgf@x=3\wd\pgfnodeparttextbox
        \ifdim\pgf@x<\minimumstatewidth
            \pgf@x=\minimumstatewidth
        \fi
    }
    \savedanchor{\upperrightcorner}
    {
        \pgf@x=.5\wd\pgfnodeparttextbox
        \pgf@y=.5\ht\pgfnodeparttextbox
        \advance\pgf@y by -.5\dp\pgfnodeparttextbox
    }
    \anchor{A}
    {
        \pgf@x=-\overallwidth
        \divide\pgf@x by 4
        \pgf@y=0pt
    }
    \anchor{B}
    {
        \pgf@x=\overallwidth
        \divide\pgf@x by 4
        \pgf@y=0pt
    }
    \anchor{text}
    {
        \upperrightcorner
        \pgf@x=-\pgf@x
        \ifhflip
            \pgf@y=-\pgf@y
            \advance\pgf@y by 0.4\stateheight
        \else
            \pgf@y=-\pgf@y
            \advance\pgf@y by -0.4\stateheight
        \fi
    }
    \backgroundpath
    {
       \begin{pgfonlayer}{foreground}
        \pgfsetstrokecolor{black}
        \pgfsetlinewidth{\thickness}
        \pgfpathmoveto{\pgfpoint{-0.5*\overallwidth}{0}}
        \pgfpathlineto{\pgfpoint{0.5*\overallwidth}{0}}
        \ifhflip
            \pgfpathlineto{\pgfpoint{0}{\stateheight}}
        \else
            \pgfpathlineto{\pgfpoint{0}{-\stateheight}}
        \fi
        \pgfpathclose
        \ifblack
            \pgfsetfillcolor{black!50}
            \pgfusepath{fill,stroke}
        \else
            \pgfusepath{stroke}
        \fi
       \end{pgfonlayer}
    }
}

\pgfdeclareshape{my ground}
{
  \inheritsavedanchors[from=rectangle ee]
  \inheritanchor[from=rectangle ee]{center}
  \inheritanchor[from=rectangle ee]{north}
  \inheritanchor[from=rectangle ee]{south}
  \inheritanchor[from=rectangle ee]{east}
  \inheritanchor[from=rectangle ee]{west}
  \inheritanchor[from=rectangle ee]{north east}
  \inheritanchor[from=rectangle ee]{north west}
  \inheritanchor[from=rectangle ee]{south east}
  \inheritanchor[from=rectangle ee]{south west}
  \inheritanchor[from=rectangle ee]{input}
  \inheritanchor[from=rectangle ee]{output}
  \anchorborder{\csname pgf@anchor@my ground@input\endcsname}

  \backgroundpath{
    \pgf@process{\pgfpointadd{\southwest}{\pgfpoint{\pgfkeysvalueof{/pgf/outer xsep}}{\pgfkeysvalueof{/pgf/outer ysep}}}}
    \pgf@xa=\pgf@x \pgf@ya=\pgf@y
    \pgf@process{\pgfpointadd{\northeast}{\pgfpointscale{-1}{\pgfpoint{\pgfkeysvalueof{/pgf/outer xsep}}{\pgfkeysvalueof{/pgf/outer ysep}}}}}
    \pgf@xb=\pgf@x \pgf@yb=\pgf@y
    \pgfpathmoveto{\pgfqpoint{\pgf@xa}{\pgf@ya}}
    \pgfpathlineto{\pgfqpoint{\pgf@xa}{\pgf@yb}}
    \pgfmathsetlength\pgf@xa{.5\pgf@xa+.5\pgf@xb}
    \pgfmathsetlength\pgf@yc{.16666\pgf@yb-.16666\pgf@ya}
    \advance\pgf@ya by\pgf@yc
    \advance\pgf@yb by-\pgf@yc
    \pgfpathmoveto{\pgfqpoint{\pgf@xa}{\pgf@ya}}
    \pgfpathlineto{\pgfqpoint{\pgf@xa}{\pgf@yb}}
    \advance\pgf@ya by\pgf@yc
    \advance\pgf@yb by-\pgf@yc
    \pgfpathmoveto{\pgfqpoint{\pgf@xb}{\pgf@ya}}
    \pgfpathlineto{\pgfqpoint{\pgf@xb}{\pgf@yb}}
  }
}
\makeatother

\tikzset{inline text/.style =
  {text height=1.2ex,text depth=0.25ex,yshift=0.5mm}}
\tikzset{arrow box/.style =
  {rectangle,inline text,fill=white,draw,
    minimum height=5mm,yshift=-0.5mm,minimum width=5mm}}
\tikzset{dot/.style =
  {inner sep=0mm,minimum width=1mm,minimum height=1mm,
    draw,shape=circle}}
\tikzset{white dot/.style = {dot,fill=white,text depth=-0.2mm}}
\tikzset{scalar/.style = {diamond,draw,inner sep=1pt}}

\tikzset{copier/.style = {dot,fill,text depth=-0.2mm}}
\tikzset{discarder/.style = {my ground,draw,inner sep=0pt,
    minimum width=4.2pt,minimum height=11.2pt,anchor=input,rotate=90}}

\tikzset{xshiftu/.style = {shift = {(#1, 0)}}}
\tikzset{yshiftu/.style = {shift = {(0, #1)}}}
\tikzset{scriptstyle/.style={font=\everymath\expandafter{\the\everymath\scriptstyle}}}

\let\Pr\relax
\DeclareMathOperator{\Pr}{P}

\NewEnviron{ctikzpicture}[1][]{\vcenter{\hbox{\begin{tikzpicture}[#1]\BODY\end{tikzpicture}}}}
\ifexternalizetikz
\makeatletter
\newcommand{\@ctikz}[2][]{\vcenter{\hbox{\begin{tikzpicture}[#1]#2\end{tikzpicture}}}}
\RenewEnviron{ctikzpicture}[1][]{%
  \def\@@ctikz{\@ctikz[#1]}%
  \expandafter\@@ctikz\expandafter{\BODY}}
\makeatother
\fi

\DeclareMathSymbol{\mhyphen}{\mathalpha}{operators}{"2D}

\ifdefined\mathsl\else
  \DeclareMathAlphabet{\mathsl}{\encodingdefault}{\rmdefault}{\mddefault}{\sldefault}
  \SetMathAlphabet{\mathsl}{bold}{\encodingdefault}{\rmdefault}{\bfdefault}{\sldefault}
\fi

\DeclareFontFamily{U}{mathux}{\hyphenchar\font45}
\DeclareFontShape{U}{mathux}{m}{n}{
      <5> <6> <7> <8> <9> <10>
      <10.95> <12> <14.4> <17.28> <20.74> <24.88>
      mathux10
      }{}
\DeclareSymbolFont{mathux}{U}{mathux}{m}{n}

\DeclareMathSymbol{\bigovee}{1}{mathux}{"8F}
\DeclareMathSymbol{\bigperp}{1}{mathux}{"4E}

\newcommand{\longto}{\longrightarrow}

\newcommand{\NN}{\mathbb{N}}

\newcommand{\RR}{\mathbb{R}}
\newcommand{\pRR}{\RR_{\geq 0}}

\newcommand{\Set}{\mathbf{Set}}

\newcommand{\Meas}{\mathbf{Meas}}

\newcommand{\sfKrn}{\mathbf{sfKrn}}
\newcommand{\pKrn}{\mathbf{pKrn}}
\newcommand{\pKrnsb}{\pKrn_{\mathrm{sb}}}

\newcommand{\Kl}{\mathcal{K}\mspace{-2mu}\ell}

\newcommand{\id}{\mathrm{id}}

\newcommand{\Dst}{\mathcal{D}}
\newcommand{\Mlt}{\mathcal{M}}
\newcommand{\Giry}{\mathcal{G}}

\newcommand{\commacat}[2]{(#1\mathbin{\downarrow}#2)}

\DeclareMathOperator{\supp}{supp}

\DeclarePairedDelimiter{\abs}{\lvert}{\rvert}

\DeclarePairedDelimiter{\ket}{\lvert}{\rangle}

\newcommand{\setin}[3]{\{#1\in#2\;|\;#3\}}

\newcommand{\given}{\mathbin{|}}

\DeclarePairedDelimiterX{\paren}[1]{(}{)}{%
\renewcommand{\given}{\mathbin{}\mathclose{}\delimsize\vert\mathopen{}\mathbin{}}%
#1}

\DeclarePairedDelimiterX{\set}[1]{\lbrace}{\rbrace}
  {\mybrconvii{\mathrel{}\mathclose{}\delimsize\vert\mathopen{}\mathrel{}}{#1}}
\DeclarePairedDelimiterX{\braket}[1]{\langle}{\rangle}
  {\mybrconviii{\mathclose{}\delimsize\vert\mathopen{}}{#1}}
\NewDocumentCommand\mybrconvii{m >{\SplitArgument{1}{|}}m}{\mybrconvauxii{#1}#2}
\NewDocumentCommand\mybrconviii{m >{\SplitArgument{2}{|}}m}{\mybrconvauxiii{#1}#2}
\NewDocumentCommand\mybrconvauxii{mmm}{\IfNoValueTF{#3}{#2}{#2#1#3}}
\NewDocumentCommand\mybrconvauxiii{mmmm}{\IfNoValueTF{#3}{#2}{\mybrconvauxii{#1}{#2#1#3}{#4}}}

\newcommand{\cat}[1]{\mathbf{#1}}

\newcommand{\catC}{\cat{C}}

\makeatletter
\newcommand{\mapsfromfill@}{\arrowfill@\leftarrow\relbar{\relbar\mapsfromchar}}
\newcommand{\xmapsfrom}[2][]{\ext@arrow 3095\mapsfromfill@{#1}{#2}}
\makeatother

\makeatletter
\newcommand{\mathoverlap}[2]{\mathpalette\mathoverlap@{{#1}{#2}}}
\newcommand{\mathoverlap@}[2]{\mathoverlap@@{#1}#2}
\newcommand{\mathoverlap@@}[3]{\ooalign{$\m@th#1#2$\crcr\hidewidth$\m@th#1#3$\hidewidth}}
\makeatother

\newcommand{\klcirc}{\mathbin{\mathoverlap{\circ}{\cdot}}}

\newcommand{\krnto}{\rightsquigarrow}

\newcommand{\pklar}{\ar@{-^>}|-@{|}}

\makeatletter
\NewDocumentCommand\mathrotatebox{omm}{%
  \IfNoValueTF{#1}
  {\mathpalette\mathrotatebox@i{{#2}{#3}}}
  {\mathpalette\mathrotatebox@ii{{#1}{#2}{#3}}}%
}
\newcommand{\mathrotatebox@i}[2]{\mathrotatebox@i@{#1}#2}
\newcommand{\mathrotatebox@i@}[3]{\rotatebox{#2}{$\m@th#1#3$}}
\newcommand{\mathrotatebox@ii}[2]{\mathrotatebox@ii@{#1}#2}
\newcommand{\mathrotatebox@ii@}[4]{\rotatebox[#2]{#3}{$\m@th#1#4$}}
\makeatother

\newcommand{\dd}{\mathrm{d}} 

\newcommand{\one}{\mathbf{1}}
\newcommand{\indic}[1]{\one_{#1}}

\makeatletter
\NewDocumentCommand\mathraisebox{moom}{%
  \IfNoValueTF{#2}
  {\mathpalette\mathraisebox@i{{#1}{#4}}}
  {%
    \IfNoValueTF{#3}
    {\mathpalette\mathraisebox@ii{{#1}{#2}{#4}}}
    {\mathpalette\mathraisebox@iii{{#1}{#2}{#3}{#4}}}%
  }%
}
\newcommand{\mathraisebox@i}[2]{\mathraisebox@i@{#1}#2}
\newcommand{\mathraisebox@i@}[3]{\raisebox{#2}{$\m@th#1#3$}}
\newcommand{\mathraisebox@ii}[2]{\mathraisebox@ii@{#1}#2}
\newcommand{\mathraisebox@ii@}[4]{\raisebox{#2}[#3]{$\m@th#1#4$}}
\newcommand{\mathraisebox@iii}[2]{\mathraisebox@iii@{#1}#2}
\newcommand{\mathraisebox@iii@}[5]{\raisebox{#2}[#3][#4]{$\m@th#1#5$}}
\makeatother

\newcommand{\nrm}{\mathrm{nrm}}

\newcommand{\cmpr}[2]{\{#1|{\kern.2ex}#2\}}

\newcommand{\klafter}{\klcirc}

\DeclareSymbolFont{T1op}{T1}{cmr}{m}{n}
\SetSymbolFont{T1op}{bold}{T1}{cmr}{bx}{n}
\DeclareMathSymbol{\mathguilsinglleft}{\mathopen}{T1op}{'016}
\DeclareMathSymbol{\mathguilsinglright}{\mathclose}{T1op}{'017}

\newcommand{\marg}[2]{\ensuremath{{#1}.{#2}}}
\newcommand{\disint}[2]{\ensuremath{{#1}\mathrel{\sslash}{#2}}}
\newcommand{\extract}[3]{\ensuremath{{#1}.\big[\,{#2}\;\big|\;{#3}\,\big]}}
\newcommand{\no}[1]{#1^{\scriptscriptstyle \bot}} 

\renewcommand{\tt}{\ensuremath{tt}}
\newcommand{\intd}{{\kern.2em}\mathrm{d}{\kern.03em}}

\ifexternalizetikz\tikzexternaldisable\fi
\newsavebox\sbground
\savebox\sbground{%
  \begin{tikzpicture}[baseline=0pt]
    \draw (0,-.1ex) to (0,.85ex)
    node[ground IEC,draw,anchor=input,inner sep=0pt,
    minimum width=3.15pt,minimum height=8.4pt,rotate=90] {};
  \end{tikzpicture}%
}
\newcommand{\ground}{\mathord{\usebox\sbground}}

\newsavebox\sbcopier
\savebox\sbcopier{%
  \begin{tikzpicture}[baseline=0pt]
    \node[copier,scale=.7] (a) at (0,3.6pt) {};
    \draw (a) -- +(-90:.16);
    \draw (a) -- +(45:.19);
    \draw (a) -- +(135:.19);
  \end{tikzpicture}}
\newcommand{\copier}{\mathord{\usebox\sbcopier}}
\ifexternalizetikz\tikzexternalenable\fi

\DeclareFixedFont{\ttb}{T1}{txtt}{bx}{n}{10} 
\DeclareFixedFont{\ttm}{T1}{txtt}{m}{n}{10}  
\usepackage{color}
\definecolor{deepblue}{rgb}{0,0,0.5}
\definecolor{deepred}{rgb}{0.6,0,0}
\definecolor{deepgreen}{rgb}{0,0.5,0}

\newcommand\pythonstyle{\lstset{
backgroundcolor = \color{lightgray},
language=Python,
basicstyle=\ttm,
otherkeywords={self,>>>,...},             
keywordstyle=\ttb\color{deepblue},
emph={MyClass,__init__},          
emphstyle=\ttb\color{deepred},    
stringstyle=\color{deepgreen},
frame=tb,                         
showstringspaces=false            %
}}

\lstnewenvironment{python}[1][]
{
\pythonstyle
\lstset{#1}
}
{}

\newcommand\pythoninline[1]{{\pythonstyle\lstinline!#1!}}

\newcommand{\EfProb}{\textit{EfProb}\xspace}

\NewDocumentCommand{\equpto}{o}{%
  \IfNoValueTF{#1}{\sim}{\overset{#1}{\sim}}}
\newcommand{\Coupl}{\mathrm{Coupl}}
\newcommand{\Caus}{\mathrm{Caus}}

\newcommand{\cind}[3]{#1\Perp#2\mid#3}

\title[Disintegration and Bayesian Inversion]{%
  Disintegration and Bayesian Inversion via String Diagrams}

\author[K. Cho and B. Jacobs]{%
  K\ls E\ls N\ls T\ls A\ns
  C\ls H\ls O$^{\dagger}$\ns
  and\ns
  B\ls A\ls R\ls T\ns
  J\ls A\ls C\ls O\ls B\ls S$^{\ddagger}$\\
  $^{\dagger}$National Institute of Informatics\addressbreak
  2-1-2 Hitotsubashi, Chiyoda-ku, Tokyo 101-8430, Japan\addressbreak
  Email: {\normalfont\ttfamily cho@nii.ac.jp}\addressbreak
  $^{\ddagger}$Institute for Computing and Information Sciences,
  Radboud University\addressbreak
  P.O.Box 9010, 6500 GL Nijmegen, the Netherlands\addressbreak
  Email: {\normalfont\ttfamily bart@cs.ru.nl}}

\journal{} 
\date{31 August 2017; revised 7 December 2018}

\begin{document}


\maketitle

\begin{abstract} 
The notions of disintegration and Bayesian inversion are fundamental
in conditional probability theory. They produce channels, as
conditional probabilities, from a joint state, or from an already
given channel (in opposite direction).  These notions exist in the
literature, in concrete situations, but are presented here in abstract
graphical formulations. The resulting abstract descriptions are used
for proving basic results in conditional probability theory. The
existence of disintegration and Bayesian inversion is discussed for
discrete probability, and also for measure-theoretic probability --- via
standard Borel spaces and via likelihoods. Finally, the usefulness of
disintegration and Bayesian inversion is illustrated in several
examples.
\end{abstract}

\section{Introduction}
\tikzextname{s_intro}

The essence of conditional probability can be summarised informally in
the following equation about probability distributions:
$$\begin{array}{rcl}
\textit{joint}
& = &
\textit{conditional} \,\cdot\, \textit{marginal}.
\end{array}$$

\noindent A bit more precisely, when we have joint probabilities
$\Pr(x,y)$ for elements $x,y$ ranging over two sample spaces, the
above equation splits into two equations,
\begin{equation}
\label{eqn:conditionalprobability}
\begin{array}{rcccl}
\Pr(y \given x) \cdot \Pr(x)
& = &
\Pr(x,y)
& = &
\Pr(x\given y)\cdot \Pr(y),
\end{array}
\end{equation}

\noindent where $\Pr(x)$ and $\Pr(y)$ describe the marginals, which are
obtained by discarding variables. We see that conditional
probabilities $\Pr(y\given x)$ and $\Pr(x\given y)$ can be constructed in
two directions, namely $y$ given $x$, and $x$ given $y$. We also see
that we need to copy variables: $x$ on the left-hand-side of the
equations~\eqref{eqn:conditionalprobability}, and $y$ on the
right-hand-side.

Conditional probabilities play a crucial role in Bayesian probability
theory. They form the nodes of Bayesian
networks~\cite{Pearl88,BernardoS00,Barber12}, which reflect the
conditional independencies of the underlying joint distribution via
their graph structure. As part of our approach, we shall capture
conditional independence in an abstract manner.

The main notion of this paper is \emph{disintegration}. It is the
process of extracting a conditional probability from a joint
probability.  Disintegration, as we shall formalise it here, gives a
structural description of the above
equation~\eqref{eqn:conditionalprobability} in terms of \emph{states}
and \emph{channels}. In general terms, a \emph{state} is a probability
distribution of some sort (discrete, measure-theoretic, or even quantum) and
a \emph{channel} is a map or morphism in a probabilistic setting, like
$\Pr(y\given x)$ and $\Pr(x\given y)$ as used above. It can take the form
of a stochastic matrix, probabilistic transition system, Markov
kernel, conditional probability table (in a Bayesian network), or
morphism in a Kleisli category of a `probability
monad'~\cite{Jacobs17}.  A state is a special kind of channel, with
trivial domain. Thus we can work in a monoidal category of channels,
where we need discarding and copying --- more formally, a comonoid
structure on each object --- in order to express the above conditional
probability equations~\eqref{eqn:conditionalprobability}.

In this article we abstract away from interpretation details and will
describe disintegration pictorially, in the language of string
diagrams. This language can be seen as the internal language of
symmetric monoidal categories~\cite{Selinger2010} --- with comonoids
in our case. The essence of disintegration becomes: extracting a
conditional probability channel from a joint state.

Categorical approaches to Bayesian conditioning have appeared for
instance in~\cite{CulbertsonS2014,StatonYHKW16,ClercDDG2017} and
in~\cite{JacobsWW15a,JacobsZ16,Jacobs17}. The latter references use
effectus theory~\cite{Jacobs15d,ChoJWW15b}, a new comprehensive
approach aimed at covering the logic of both quantum theory and
probability theory, supported by a Python-based tool \EfProb, for
`effectus probability'. This tool is used for the (computationally
extensive) examples in this paper.

Disintegration, also known as regular conditional probability, is a
notoriously difficult operation in measure-theoretic probability, see
\textit{e.g.}~\cite{Pollard2002,Panangaden09,ChangP1997}: it may not
exist~\cite{Stoyanov2014}; even if it exists it may be determined only
up to negligible sets; and it may not be continuous or
computable~\cite{AckermanFR2011}.  Disintegration has been studied
using categorical language in~\cite{CulbertsonS2014}, which focuses on
a specific category of probabilistic mappings.
Our approach here is more axiomatic.

We thus describe disintegration as going from a joint state to a
channel. A closely related concept is \emph{Bayesian inversion}: it
turns a channel (with a state) into a channel in opposite direction.
We show how Bayesian inversion can be understood and expressed easily
in terms of disintegration --- and also how, in the other direction,
disintegration can be obtained from Bayesian inversion. Bayesian
inversion is taken as primitive notion in~\cite{ClercDDG2017}. Here we
start from disintegration. The difference is a matter of choice.

Bayesian inversion is crucial for backward inference. We explain it
informally: let $\sigma$ be a state of a domain/type $X$, and $c\colon
X \rightarrow Y$ be a channel; Bayesian inversion yields a channel
$d\colon Y \rightarrow X$. Informally, it produces for an element
$y\in Y$, seen as singleton/point predicate $\{y\}$, the conditioning
of the state $\sigma$ with the pulled back evidence $c^{-1}(\{y\})$.
A concrete example involving such `point observations' will be
described at the end of Section~\ref{sec:likelihood}. More generally,
disintegration and Bayesian inversion are used to structurally
organise state updates in the presence of new evidence in probabilistic
programming, see
\textit{e.g.}~\cite{GordonHNRG14,BorgstromGGMG13,StatonYHKW16,KatoenGJKO2015}.
See also \cite{ShanR17}, where disintegration
is handled via symbolic manipulation.

Disintegration and Bayesian inversion are relatively easy to define in discrete
probability theory. The situation is much more difficult in measure-theoretic
probability theory, first of all because point predicates $\{y\}$ do
not make much sense there, see also~\cite{ChangP1997}.
A common solution to the problem of the existence of
disintegration / Bayesian inversion
is to restrict ourselves to standard Borel spaces, as in~\cite{ClercDDG2017}.
We take this approach too.
There is still an issue that disintegration is determined only up to negligible sets.
We address this by defining `almost equality'
in our abstract pictorial formulation.
This allows us to present a fundamental result from~\cite{ClercDDG2017} abstractly
in our setting, see Section~\ref{sec:equality}.

Another common, more concrete solution is to
assume a \emph{likelihood}, that is, a probabilistic relation $X\times
Y \rightarrow \pRR$. Such a likelihood gives rise to probability
density function (pdf), providing a good handle on the situation,
see~\cite{Pawitan01}.  The technical core of
Section~\ref{sec:likelihood} is a generalisation of this
likelihood-based approach.

The paper is organised as follows. It starts with a brief introduction
to the graphical language that we shall be using, and to the
underlying monoidal categories with discarding and copying. Then,
Section~\ref{sec:disintegration} introduces both disintegration and
Bayesian inversion in this graphical language, and relates the two
notions. Subsequently, Section~\ref{sec:classifier} contains an
elaborated example, namely of naive Bayesian classification. A
standard example from the literature~\cite{WittenFH11} is redescribed
in the current setting: first, channels are extracted via
disintegration from a table with given data; next, Bayesian inversion
is applied to the combined extracted channels, giving the required
classification. This is illustrated in both the discrete and the
continuous version of the example.

Next, Section~\ref{sec:equality} is more technical and elaborates the
standard equality notion of `equal almost everywhere' in the current
setting. This is used for describing Bayesian inversion in a more
formal way,
following~\cite{ClercDDG2017}. Section~\ref{sec:conditionalindependence}
uses our graphical approach to review conditional independence and to
prove at an abstract level several known results, namely the equivalence
of various formulations of conditional independence, and the
`graphoid' axioms
from~\cite{VermaP1988,GeigerVP1990}. Section~\ref{sec:beyondcausal}
relaxes the requirement that maps are causal, so that `effects' can be
used as the duals of states for validity and conditioning. The main
result relates conditioning of joint states to forward and backward
inference via the extracted channels, in the style
of~\cite{JacobsZ16,JacobsZ18}; it is illustrated in a concrete
example, where a Bayesian network is seen as a graph in a Kleisli
category --- following~\cite{Fong12}. Finally,
Section~\ref{sec:likelihood} gives the likelihood formulation of
disintegration and inversion, as briefly described above.

\section{Graphical language}\label{sec:graphical}
\tikzextname{s_graphical}

The basic idea underlying this paper is to describe probability theory in terms of
\emph{channels}. A channel $f\colon X\to Y$ is a (stochastic) process
from a system of type $X$ into that of $Y$.  Concretely, it may be a
probability matrix or kernel.  Our standing assumption is that
types (as objects)
and channels (as arrows)
form a \emph{symmetric monoidal category}.
For the formal definition we refer to~\cite{MacLane1998}.  We
informally summarise that we have the following constructions.
\begin{enumerate}
\item
Sequential composition
$g\circ f\colon X\to Z$
for appropriately typed channels $f\colon X\to Y$ and $g\colon Y\to Z$.
\item
Parallel composition
$f\otimes g\colon X\otimes Z\to Y\otimes W$
for $f\colon X\to Y$ and $g\colon Z\to W$.
This involves composition of types $X\otimes Z$.
\item
Identity channels $\id_X\colon X\to X$, which `do nothing'.
Thus $\id\circ f = f = \id\circ f$.
\item
A unit type $I$, which represents `no system'.
Thus $I\otimes X\cong X\cong X\otimes I$.
\item
Swap isomorphisms $X\otimes Y\cong Y\otimes X$ and associativity
isomorphisms $(X\otimes Y)\otimes Z \cong X\otimes (Y\otimes Z)$, so
the ordering in composed types does not matter.
\end{enumerate}

The representation of such channels in the ordinary `formula' notation
easily becomes complex and thus reasoning becomes hard to follow.  A
graphical language known as \emph{string diagrams} offers a more
convenient and intuitive way of reasoning in a symmetric monoidal
category.

In string diagrams, types/objects are represented as wires
`\,\begin{tikzpicture}[baseline=-.6ex,y=1ex]
\draw (0,-1) -- (0,1);
\end{tikzpicture}\,', with information flowing bottom to top.
The composition of types is depicted by juxtaposition of wires, and
the unit type is `no diagram' as below.
\[
\vcenter{\hbox{%
\begin{tikzpicture}[font=\small]
\draw (0,0) to +(0,1);
\node at (0.55,0.5) {$X\otimes Y$};
\end{tikzpicture}}}
=\quad
\vcenter{\hbox{%
\begin{tikzpicture}[font=\small]
\draw (0,0) to +(0,1);
\draw (0.5,0) to +(0,1);
\node at (0.2,0.5) {$X$};
\node at (0.7,0.5) {$Y$};
\end{tikzpicture}}}
\qquad\qquad\qquad
\vcenter{\hbox{%
\begin{tikzpicture}[font=\small]
\draw (0,0) to +(0,1);
\node at (0.2,0.5) {$I$};
\end{tikzpicture}}}
=\quad
\vcenter{\hbox{%
\begin{tikzpicture}[font=\small]
\draw [gray,dashed] (0,0) rectangle (0.7,1);
\end{tikzpicture}}}
\]
Channels/arrows are represented by boxes with an input wire(s) and an
output wire(s), in upward direction.  When a box does not have input
or output, we write it as a triangle or diamond. For example, $f\colon
X\to Y\otimes Z$, $\omega\colon I\to X$, $p\colon X \to I$, and
$s\colon I\to I$ are respectively depicted as:
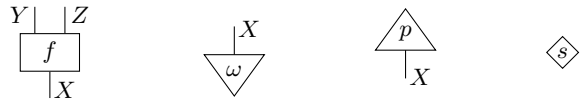
\[
\vcenter{\hbox{%
\begin{tikzpicture}[font=\small]
\node[arrow box] (v1) at (0,0) {\;\;$f$\;\;};
\draw (v1.north) ++(-0.2,0) -- +(0,0.35);
\draw (v1.north) ++(0.2,0) -- +(0,0.35);
\draw (v1.south) -- +(0,-0.35);
\node at (0.2,-0.5) {$X$};
\node at (-0.4,0.5) {$Y$};
\node at (0.4,0.5) {$Z$};
\end{tikzpicture}}}
\qquad\qquad
\vcenter{\hbox{%
\begin{tikzpicture}[font=\small]
\node[state] (omega) at (0,0) {$\omega$};
\node (X) at (0,0.5) {};
\draw (omega) to (X);
\node at (0.2,0.35) {$X$};
\end{tikzpicture}}}
\qquad\qquad
\vcenter{\hbox{%
\begin{tikzpicture}[font=\small]
\node[state,hflip] (p) at (0,0) {$p$};
\node (X) at (0,-0.5) {};
\draw (p) to (X);
\node at (0.2,-0.35) {$X$};
\end{tikzpicture}}}
\qquad\qquad
\vcenter{\hbox{%
\begin{tikzpicture}[font=\small]
\node[scalar] (c) at (-1,0.6) {$s$};
\end{tikzpicture}}}
\]
The identity channels are represented by `no box', \textit{i.e.}\ just wires,
and the swap isomorphisms are represented by crossing of wires:
\[
\vcenter{\hbox{%
\begin{tikzpicture}[font=\small]
\node[arrow box] (c) at (0,0) {$\id$};
\draw  (c) to (0, 0.5);
\draw (c) to (0,-0.5);
\node at (0.2,0.6) {$X$};
\node at (0.2,-0.6) {$X$};
\end{tikzpicture}}}
=\;\;
\vcenter{\hbox{%
\begin{tikzpicture}[font=\small]
\draw (0,0) to +(0,1);
\node at (0.2,0.5) {$X$};
\end{tikzpicture}}}
\qquad\qquad\qquad
\vcenter{\hbox{%
\begin{tikzpicture}[font=\small]
\draw (0,0) to (0.5,1);
\draw (0,1) to (0.5,0);
\node at (-0.2,0) {$X$};
\node at (0.7,1) {$X$};
\node at (0.7,0) {$Y$};
\node at (-0.2,1) {$Y$};
\end{tikzpicture}}}
\]
Finally, the sequential composition of channels is depicted by connecting
the input and output wires, and
the parallel composition is given by juxtaposition, respectively as below:
\[
\vcenter{\hbox{%
\begin{tikzpicture}[font=\small]
\node[arrow box] (c) at (0,0) {$g\circ f$};
\draw  (c) to (0, 0.7);
\draw (c) to (0,-0.7);
\end{tikzpicture}}}
\;\;=\;\;
\vcenter{\hbox{%
\begin{tikzpicture}[font=\small]
\node[arrow box] (f) at (0,-1.5) {$f$};
\node[arrow box] (g) at (0,-0.8) {$g$};
\draw (g.north) to +(0,0.2);
\draw (f) to (g);
\draw (f.south) to +(0,-0.2);
\end{tikzpicture}}}
\qquad\qquad\qquad
\vcenter{\hbox{%
\begin{tikzpicture}[font=\small]
\node[arrow box] (c) at (0,0) {$h\otimes k$};
\draw (c) to (0, 0.6);
\draw (c) to (0,-0.6);
\end{tikzpicture}}}
\;\;=\;\;
\vcenter{\hbox{%
\begin{tikzpicture}[font=\small]
\node[arrow box] (h) at (0,0) {$h$};
\node[arrow box] (k) at (0.7,0) {$k$};
\draw (h) to +(0, 0.6);
\draw (h) to +(0,-0.6);
\draw (k) to +(0, 0.6);
\draw (k) to +(0,-0.6);
\end{tikzpicture}}}
\]
The use of string diagrams is justified by the following `coherence'
theorem;
see \cite{JoyalS1991,Selinger2010} for details.

\begin{theorem}
A well-formed equation between composites of arrows
in a symmetric monoidal category follows
from the axioms of symmetric monoidal categories if and only if
the string diagrams of both sides are equal
up to isomorphism of diagrams.
\qed
\end{theorem}

We further assume the following structure in our category.
For each type $X$ there are a discarder $\ground_X\colon X\to I$
and a copier $\copier_X\colon X\to X\otimes X$.
They are required to satisfy the following equations:
\[
\vcenter{\hbox{%
\begin{tikzpicture}[font=\small]
\node[copier] (c) at (0,0) {};
\coordinate (x1) at (-0.3,0.3);
\node[discarder] (d) at (x1) {};
\coordinate (x2) at (0.3,0.5);
\draw (c) to[out=165,in=-90] (x1);
\draw (c) to[out=15,in=-90] (x2);
\draw (c) to (0,-0.3);
\end{tikzpicture}}}
\quad=\quad
\vcenter{\hbox{%
\begin{tikzpicture}[font=\small]
\draw (0,0) to (0,0.8);
\end{tikzpicture}}}
\quad=\quad
\vcenter{\hbox{%
\begin{tikzpicture}[font=\small]
\node[copier] (c) at (0,0) {};
\coordinate (x1) at (-0.3,0.5);
\coordinate (x2) at (0.3,0.3);
\node[discarder] (d) at (x2) {};
\draw (c) to[out=165,in=-90] (x1);
\draw (c) to[out=15,in=-90] (x2);
\draw (c) to (0,-0.3);
\end{tikzpicture}}}
\qquad\qquad\quad
\vcenter{\hbox{%
\begin{tikzpicture}[font=\small]
\node[copier] (c2) at (0,0) {};
\node[copier] (c1) at (-0.3,0.3) {};
\draw (c2) to[out=165,in=-90] (c1);
\draw (c2) to[out=15,in=-90] (0.4,0.6);
\draw (c1) to[out=165,in=-90] (-0.6,0.6);
\draw (c1) to[out=15,in=-90] (0,0.6);
\draw (c2) to (0,-0.3);
\end{tikzpicture}}}
\quad=\quad
\vcenter{\hbox{%
\begin{tikzpicture}[font=\small]
\node[copier] (c2) at (-0.5,0) {};
\node[copier] (c1) at (-0.2,0.3) {};
\draw (c2) to[out=15,in=-90] (c1);
\draw (c2) to[out=165,in=-90] (-1,0.6);
\draw (c1) to[out=165,in=-90] (-0.5,0.6);
\draw (c1) to[out=15,in=-90] (0.1,0.6);
\draw (c2) to (-0.5,-0.3);
\end{tikzpicture}}}
\qquad\qquad\quad
\vcenter{\hbox{%
\begin{tikzpicture}[font=\small]
\node[copier] (c) at (0,0.4) {};
\draw (c)
to[out=15,in=-90] (0.25,0.7)
to[out=90,in=-90] (-0.25,1.4);
\draw (c)
to[out=165,in=-90] (-0.25,0.7)
to[out=90,in=-90] (0.25,1.4);
\draw (c) to (0,0.1);
\end{tikzpicture}}}
\quad=\quad
\vcenter{\hbox{%
\begin{tikzpicture}[font=\small]
\node[copier] (c) at (0,0.4) {};
\draw (c)
to[out=15,in=-90] (0.25,0.7);
\draw (c)
to[out=165,in=-90] (-0.25,0.7);
\draw (c) to (0,0.1);
\end{tikzpicture}}}
\]
This says that $(\copier_X, \ground_X)$ forms a commutative comonoid
on $X$.  By the associativity we may write:
\[
\begin{ctikzpicture}[font=\small]
\node[copier] (c3) at (0,0) {};
\draw (0,-0.2) to (c3);
\draw (c3) to[out=165,in=-90] (-0.6,0.4);
\draw (c3) to[out=160,in=-90] (-0.4,0.4);
\draw (c3) to[out=150,in=-90] (-0.2,0.4);
\draw (c3) to[out=15,in=-90] (0.6,0.4);
\node[font=\normalsize] at (0.2,0.3) {$\dots$};
\end{ctikzpicture}
\;\;\coloneqq\;\;
\begin{ctikzpicture}[font=\small]
\node[copier] (c2) at (-0.1,0.2) {};
\node[copier] (c1) at (-0.3,0.4) {};
\node[copier] (c3) at (0.5,-0.1) {};
\draw (c2) to[out=165,in=-90] (c1);
\draw (c2) to[out=15,in=-90] (0.2,0.6);
\draw (c1) to[out=165,in=-90] (-0.5,0.6);
\draw (c1) to[out=15,in=-90] (-0.1,0.6);
\draw (0.5,-0.35) to (c3);
\draw (c3) to[out=15,in=-90] (0.85,0.6);
%
\node[font=\normalsize] at (0.5,0.25) {$\dots$};
\end{ctikzpicture}
\]
Moreover we assume that the comonoid structures $(\copier_X, \ground_X)$ are compatible
with the monoidal structure $(\otimes, I)$, in the sense that
the following equations hold.
\[
\vcenter{\hbox{%
\begin{tikzpicture}[font=\small]
\node[discarder] (d) at (0,0) {};
\draw (d) to +(0,-0.5);
\node at (0.6,-0.3) {$X\otimes Y$};
\end{tikzpicture}}}
=
\vcenter{\hbox{%
\begin{tikzpicture}[font=\small]
\node[discarder] (d) at (0,0) {};
\node[discarder] (d2) at (0.6,0) {};
\draw (d) to +(0,-0.5);
\draw (d2) to +(0,-0.5);
\node at (0.2,-0.3) {$X$};
\node at (0.8,-0.3) {$Y$};
\end{tikzpicture}}}
\qquad\quad
\vcenter{\hbox{%
\begin{tikzpicture}[font=\small]
\node[discarder] (d) at (0,0) {};
\draw (d) to +(0,-0.5);
\node at (0.2,-0.3) {$I$};
\end{tikzpicture}}}
=\;
\vcenter{\hbox{%
\begin{tikzpicture}[font=\small]
\draw [gray,dashed] (0,0) rectangle (0.45,0.65);
\end{tikzpicture}}}
\qquad\quad
\vcenter{\hbox{%
\begin{tikzpicture}[font=\small]
\node[copier] (c) at (0,0.4) {};
\draw (c)
to[out=15,in=-90] (0.25,0.7);
\draw (c)
to[out=165,in=-90] (-0.25,0.7);
\draw (c) to (0,0);
\node at (0.6,0.1) {$X\otimes Y$};
\end{tikzpicture}}}
=
\vcenter{\hbox{%
\begin{tikzpicture}[font=\small]
\node[copier] (c) at (0,0) {};
\node[copier] (c2) at (0.4,0) {};
\draw (c) to[out=15,in=-90] +(0.5,0.45);
\draw (c) to[out=165,in=-90] +(-0.4,0.45);
\draw (c) to +(0,-0.4);
\draw (c2) to[out=15,in=-90] +(0.4,0.45);
\draw (c2) to[out=165,in=-90] +(-0.5,0.45);
\draw (c2) to +(0,-0.4);
\node at (-0.2,-0.3) {$X$};
\node at (0.6,-0.3) {$Y$};
\end{tikzpicture}}}
\qquad\quad
\vcenter{\hbox{%
\begin{tikzpicture}[font=\small]
\node[copier] (c) at (0,0.4) {};
\draw (c)
to[out=15,in=-90] (0.25,0.7);
\draw (c)
to[out=165,in=-90] (-0.25,0.7);
\draw (c) to (0,0);
\node at (0.2,0.1) {$I$};
\end{tikzpicture}}}
=\;
\vcenter{\hbox{%
\begin{tikzpicture}[font=\small]
\draw [gray,dashed] (0,0) rectangle (0.45,0.65);
\end{tikzpicture}}}
\]
Note that we do \emph{not} assume that
these maps are natural. Explicitly, we do not
necessarily have $\copier\circ f = (f\otimes f)\circ \copier$
or $\ground\circ f=\ground$.

We will use these
symmetric monoidal categories throughout in the paper.
For convenience, we introduce a term for them.

\begin{definition}
\label{def:cd-cat}
A \emph{CD-category} is
a symmetric monoidal category $(\catC,\otimes,I)$ with
a commutative comonoid $(\copier_X,\ground_X)$
for each $X\in\catC$, suitably compatible as described above.
\end{definition}
Here `CD' stands for Copy/Discard.

\begin{definition}
An arrow $f\colon X\to Y$ in a CD-category
is said to be \emph{causal} if
\[
\vcenter{\hbox{%
\begin{tikzpicture}[font=\small]
\node[arrow box] (c) at (0,0) {$f$};
\node[discarder] (d) at (0,0.5) {};
\draw  (c) to (d);
\draw (c) to (0,-0.5);
\end{tikzpicture}}}
\;=\;
\vcenter{\hbox{%
\begin{tikzpicture}[font=\small]
\node[discarder] (d) at (0,0) {};
\draw (d) to (0,-0.5);
\end{tikzpicture}}}
\quad.
\]

\noindent A CD-category is \emph{affine} if all the arrows are causal,
or equivalently, the tensor unit $I$ is a final object.
\end{definition}

The term `causal' comes from (categorical) quantum
foundation~\cite{CoeckeK2017,DArianoCP2017}, and is related to
relativistic causality, see \textit{e.g.}~\cite{Coecke2016}.

We reserve the term `channel' for causal arrows.  Explicitly, causal
arrows $c\colon X\to Y$ in a CD-category are called \emph{channels}.
A channel $\omega\colon I\to X$ with input type $I$ is called a
\emph{state} (on $X$).
For the time being
we will only use channels (and states), and thus only consider affine
CD-categories.
Non-causal arrows will not appear until
Section~\ref{sec:beyondcausal}.

Our main examples of affine CD-categories are two Kleisli categories
$\Kl(\Dst)$ and $\Kl(\Giry)$, respectively, for discrete probability,
and more general, measure-theoretic probability.  We explain them in
order below. There are more examples, like the Kleisli category of the
non-empty powerset monad, or the (opposite of) the category of
commutative $C^*$-algebras (with positive unital maps), but they
are out of scope here.

\begin{example}
\label{ex:KlD}
  What we call a \emph{distribution} or a \emph{state} over a set
  $X$ is a finite subset $\{x_{1}, x_{2}, \ldots, x_{n}\} \subseteq
  X$, called the support where each element $x_{i}$ occurs with a
  multiplicity $r_{i} \in [0,1]$, such that $\sum_{i}r_{i}=1$. Such a
  convex combination is often written as $r_{1}\ket{x_1} + \cdots +
  r_{n}\ket{x_n}$ with $r_{i}\in[0,1]$. The ket notation $\ket{-}$ is
  meaningless syntactic sugar that is used to distinguish elements
  $x\in X$ from occurrences in such formal sums. Notice that a
  distribution can also be written as a function $\omega \colon X
  \rightarrow [0,1]$ with finite support $\supp(\omega) =
  \setin{x}{X}{\omega(x) \neq 0}$. We shall write $\Dst(X)$ for the
  set of distributions over $X$. This $\Dst$ is a monad on the
  category $\Set$ of sets and functions.

A function $f\colon X \rightarrow \Dst(Y)$ is called a \emph{Kleisli}
map; it forms a channel $X \rightarrow Y$. Such maps can be composed
as matrices, for which we use special notation $\klafter$.
$$\begin{array}{rcl}
(g \klafter f)(x)(z)
& = &
\sum_{y\in Y} f(x)(y) \cdot g(y)(z) \qquad 
   \mbox{for $g\colon Y \rightarrow  \Dst(Z)$ with $x\in X, z\in Z$.}
\end{array}$$

\noindent We write $1 = \{*\}$ for a singleton set, and $2 = 1+1 =
\{0,1\}$. Notice that $\Dst(1) \cong 1$ and $\Dst(2) \cong [0,1]$. We
can identify a state on $X$ with a channel $1 \rightarrow X$.

The monad $\Dst$ is known to be \emph{commutative}.  This implies that
finite products of sets $X\times Y$ give rise to a symmetric monoidal
structure on the Kleisli category $\Kl(\Dst)$.  Specifically, for two
maps $f\colon X\to \Dst(Y)$ and $g\colon Z\to \Dst(W)$, the tensor
product / parallel composition $f\otimes g\colon X\times Z\to
\Dst(Y\times W)$ is given by:
\[ \begin{array}{rcl}
(f\otimes g)(x,z)(y,w)
& = &
f(x,y)\cdot g(z,w).
\end{array}
\]

\noindent For each set $X$ there are a copier $\copier_X\colon X\to
\Dst(X\times X)$ and a discarder $\ground_X\colon X\to \Dst(1)$
given by $\copier_X(x)=1\ket{x,x}$
and $\ground_X(x)=1\ket{*}$, respectively.
They come from the cartesian
(finite product) structure of the base category $\Set$,
through the obvious functor $\Set\to\Kl(\Dst)$.
Therefore $\Kl(\Dst)$ is a
CD-category.  It is moreover affine, since the monad is affine in the
sense that $\Dst(1)\cong 1$.
\end{example}

\begin{example}
\label{ex:KlG}
  Let $X = (X, \Sigma_{X})$ be a measurable space, where
  $\Sigma_{X}$ is a $\sigma$-algebra on $X$. A \emph{probability
  measure}, also called a \emph{state}, on $X$ is a function $\omega
  \colon \Sigma_{X} \rightarrow [0,1]$ which is countably additive and
  satisfies $\omega(X) = 1$. We write $\Giry(X)$ for the collection of
  all such probability measures on $X$. This set $\Giry(X)$ is itself
  a measurable space with the smallest $\sigma$-algebra
  such that
  for each $A\in\Sigma_X$ the `evaluation' map
  $\mathrm{ev}_A\colon \Giry(X)\to[0,1]$,
  $\mathrm{ev}_A(\omega)=\omega(A)$,
  is measurable.
  Notice that $\Giry(X) \cong \Dst(X)$ when $X$ is
  a finite set (as discrete space). In particular, $\Giry(2) \cong
  \Dst(2) \cong [0,1]$. This $\Giry$ is a monad on the category
  $\Meas$ of measurable spaces, with measurable functions between
  them; it is called the Giry monad, after~\cite{Giry1982}.

A Kleisli map, that is, a measurable function $f\colon X \rightarrow
\Giry(Y)$ is a channel (or a \emph{probability kernel},
see Example~\ref{ex:CD-cat}). These channels can be
composed, via Kleisli composition $\klafter$, using integration%
\footnote{%
We denote
the integral of a function $f$
with respect to a measure $\mu$
by $\int_X f(x)\,\mu(\dd x)$.}:
\[
(g \klafter f)(x)(C)
=
\int_Y g(y)(C) \, f(x)(\dd y)  \quad 
   \text{where $g\colon Y \rightarrow \Giry(Z)$ and $x\in X, C\in\Sigma_{Z}$.}
\]

It is well-known that the monad $\Giry$ is commutative and affine, see
also~\cite{Jacobs17}.  Thus, in a similar manner to the previous
example, the Kleisli category $\Kl(\Giry)$ is an affine CD-category.
The parallel composition $f\otimes g\colon X\times Z\to \Giry(Y\times W)$
for $f\colon X\to \Giry(Y)$ and $g\colon Z\to \Giry(W)$ is given as:
\[ \begin{array}{rcl}
(f\otimes g)(x,z)(B\times D)
& = &
f(x)(B)\cdot g(z)(D),
\end{array}
\]
for $x\in X$, $z\in Z$, $B\in\Sigma_Y$, and $D\in\Sigma_W$.
Since $f(x)$ and $g(z)$ are ($\sigma$-)finite measures,
this indeed
determines a unique measure
$(f\otimes g)(x,z)\in \Giry(Y\times W)$,
namely the unique
product measure of $f(x)$ and $g(z)$.
Explicitly, for $E\in\Sigma_{Y\times W}$,
\[
(f\otimes g)(x,z)(E)
=
\int_{W}
\paren[\bigg]{
\int_{Y}
\indic{E}(y,w)\,
f(x)(\dd y)}
g(z)(\dd w)
\]
where $\indic{E}$ is the indicator function.
\end{example}

\section{Marginalisation, integration and disintegration}
\label{sec:disintegration}
\tikzextname{s_disint}

Let $\catC$ be an affine CD-category.  We think of states
$\omega\colon I\to X$ in $\catC$ as abstract (probability)
distributions on type $X$.  States of the form $\omega\colon I\to
X\otimes Y$, often called (bipartite) joint states, are seen as joint
distributions on $X$ and $Y$.  Later on we shall also consider
$n$-partite joint states, but for the time being we restrict ourselves
to bipartite ones.  For a joint distribution $\Pr(x,y)$ in discrete
probability, we can calculate the marginal distribution on $X$ by
summing (or marginalising) $Y$ out, as $\Pr(x)=\sum_y \Pr(x,y)$.  The
marginal distribution on $Y$ is also calculated by $\Pr(y)=\sum_x
\Pr(x,y)$.  In our abstract setting, given a joint state $\omega\colon
I\to X\otimes Y$, we can obtain marginal states simply by discarding
wires, as in:
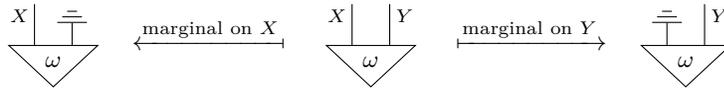
\[
\vcenter{\hbox{%
\begin{tikzpicture}[font=\small]
\node[state] (omega) at (0,0) {\;$\omega$\;};
\coordinate (X) at (-0.25,0.55) {};
\node [discarder] (ground) at (0.25,0.3) {};
\draw (omega) ++(-0.25, 0) to (X);
\draw (omega) ++(0.25, 0) to (ground);
\node[scriptstyle] at (-0.45,0.4) {$X$};
\end{tikzpicture}}}
\quad\xmapsfrom{\text{marginal on $X$}}\quad
\vcenter{\hbox{%
\begin{tikzpicture}[font=\small]
\node[state] (omega) at (0,0) {\;$\omega$\;};
\coordinate (X) at (-0.25,0.55) {};
\coordinate (Y) at (0.25,0.55) {};
\draw (omega) ++(-0.25, 0) to (X);
\draw (omega) ++(0.25, 0) to (Y);
\path[scriptstyle]
node at (-0.45,0.4) {$X$}
node at (0.45,0.4) {$Y$};
\end{tikzpicture}}}
\quad\xmapsto{\text{marginal on $Y$}}\quad
\vcenter{\hbox{%
\begin{tikzpicture}[font=\small]
\node[state] (omega) at (0,0) {\;$\omega$\;};
\node [discarder] (X) at (-0.25,0.3) {};
\coordinate (Y) at (0.25,0.55) {};
\draw (omega) ++(-0.25, 0) to (X);
\draw (omega) ++(0.25, 0) to (Y);
%
\node[scriptstyle] at (0.45,0.4) {$Y$};
\end{tikzpicture}}}
\]
In other words, the marginal states are the state $\omega$ composed with the projection maps
$\pi_1\colon X\otimes Y\to X$ and $\pi_2\colon X\otimes Y\to Y$, as below.
\[
\pi_1 \coloneqq\;\;
\begin{ctikzpicture}[font=\small]
\node[discarder] (d) at (0,0.4) {};
\draw (0,0) to (d);
\draw (-0.35,0) to +(0,0.9);
\path[scriptstyle]
node at (-0.5,0.15) {$X$}
node at (0.15,0.15) {$Y$};
\end{ctikzpicture}
\qquad\qquad\qquad\qquad
\pi_2 \coloneqq\;\;
\begin{ctikzpicture}[font=\small]
\node[discarder] (d) at (-0.35,0.35) {};
\draw (-0.35,0) to (d);
\draw (0,0) to +(0,0.9);
\path[scriptstyle]
node at (-0.5,0.1) {$X$}
node at (0.15,0.1) {$Y$};
\end{ctikzpicture}
\]

\begin{example}
For a joint state $\omega\in\Dst(X\times Y)$ in $\Kl(\Dst)$,
the first marginal $\omega_1=\pi_1\klcirc\omega$ is given by
$\omega_1(x)=\sum_{y\in Y} \omega(x,y)$, as expected.
Similarly the second marginal is
given by $\omega_2(y)=\sum_{x\in X} \omega(x,y)$.
\end{example}

\begin{example}
For a joint state $\omega\in\Giry(X\times Y)$ in $\Kl(\Giry)$,
the first marginal is given by $\omega_1(A)=\omega(A\times Y)$
for $A\in\Sigma_X$,
and the second marginal by
$\omega_2(B)=\omega(X\times B)$ for $B\in\Sigma_Y$.
\end{example}

A channel $c\colon X\to Y$ is seen as an abstract \emph{conditional}
distribution $\Pr(y|x)$.
In (discrete) probability theory, we can
calculate a joint distribution $\Pr(x,y)$ from a distribution $\Pr(x)$
and a conditional distribution $\Pr(y|x)$ by the formula
$\Pr(x,y)=\Pr(y|x)\cdot\Pr(x)$,
which is often called the \emph{product rule}.
Similarly we have $\Pr(x,y)=\Pr(x|y)\cdot\Pr(y)$.
In our setting, starting from a
state $\sigma\colon I\to X$ and a channel $c\colon X\to Y$, or a state
$\tau\colon I\to Y$ and a channel $d\colon Y\to X$, we can
`integrate' them into a joint state on $X\otimes Y$ as follows, respectively:
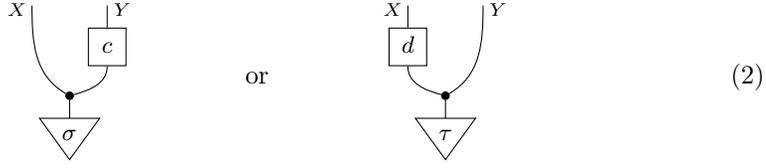
\begin{equation}
\label{eq:joint-state}
\vcenter{\hbox{%
\begin{tikzpicture}[font=\small]
\node[state] (omega) at (0,0) {$\sigma$};
\node[copier] (copier) at (0,0.3) {};
\node[arrow box] (c) at (0.5,0.95) {$c$};
\coordinate (X) at (-0.5,1.5);
\coordinate (Y) at (0.5,1.5);
\draw (omega) to (copier);
\draw (copier) to[out=150,in=-90] (X);
\draw (copier) to[out=15,in=-90] (c);
\draw (c) to (Y);
\path[scriptstyle]
node at (-0.7,1.45) {$X$}
node at (0.7,1.45) {$Y$};
\end{tikzpicture}}}
\qquad\qquad\text{or}\qquad\qquad
\vcenter{\hbox{%
\begin{tikzpicture}[font=\small]
\node[state] (omega) at (0,0) {$\tau$};
\node[copier] (copier) at (0,0.3) {};
\node[arrow box] (c) at (-0.5,0.95) {$d$};
\coordinate (X) at (-0.5,1.5);
\coordinate (Y) at (0.5,1.5);
\draw (omega) to (copier);
\draw (copier) to[out=165,in=-90] (c);
\draw (copier) to[out=30,in=-90] (Y);
\draw (c) to (X);
\path[scriptstyle]
node at (-0.7,1.45) {$X$}
node at (0.7,1.45) {$Y$};
\end{tikzpicture}}}
\end{equation}

\begin{example}
Let $\sigma\in\Dst(X)$ and $c\colon X\to \Dst(Y)$
be a state and a channel in $\Kl(\Dst)$.
An easy calculation verifies that
$\omega=(\id\otimes c)\klcirc\copier\klcirc\sigma$, the joint state on $X\times Y$
defined as in~\eqref{eq:joint-state},
satisfies $\omega(x,y)=c(x)(y)\cdot\sigma(x)$, as we expect
from the product rule.
\end{example}

\begin{example}
For a state $\sigma\in\Giry(X)$ and a channel $c\colon X\to \Giry(Y)$
in $\Kl(\Giry)$, the joint state $\omega=(\id\otimes
c)\klcirc\copier\klcirc\sigma$ is given by $\omega(A\times B)=\int_A
c(x)(B)\, \sigma(\dd x)$ for $A\in\Sigma_X$ and $B\in\Sigma_Y$.  This
`integration' construction of a joint probability measure is standard,
see \textit{e.g.}~\cite{Pollard2002,Panangaden09}.
\end{example}

\emph{Disintegration} is an inverse operation of
the `integration' of a state and a channel into a joint state, as in~\eqref{eq:joint-state}.
More specifically, it starts from a joint state $\omega\colon I\to X\otimes Y$
and extracts either a state $\omega_1\colon I\to X$ and a channel $c_1\colon X\to Y$,
or a state $\omega_2\colon I\to Y$ and a channel $c_2\colon Y\to X$ as below,
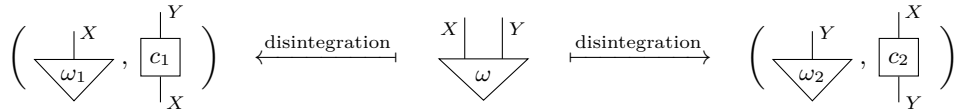
\[
\biggl(\;
\vcenter{\hbox{%
\begin{tikzpicture}[font=\small]
\node[state] (omega) at (0,0) {$\omega_1$};
\node (X) at (0,0.5) {};
\draw (omega) to (X);
\node[scriptstyle] at (0.2,0.35) {$X$};
\end{tikzpicture}}}
\;,\;
\vcenter{\hbox{%
\begin{tikzpicture}[font=\small]
\node[arrow box] (c) at (0,0) {$c_1$};
\draw (c) to (0, 0.6);
\draw (c) to (0,-0.6);
\path[scriptstyle]
node at (0.2,0.6) {$Y$}
node at (0.2,-0.6) {$X$};
\end{tikzpicture}}}
\;\biggr)
\quad\xmapsfrom{\text{disintegration}}\quad
\vcenter{\hbox{%
\begin{tikzpicture}[font=\small]
\node[state] (omega) at (0,0) {\;$\omega$\;};
\coordinate (X) at (-0.25,0.55) {};
\coordinate (Y) at (0.25,0.55) {};
\draw (omega) ++(-0.25, 0) to (X);
\draw (omega) ++(0.25, 0) to (Y);
\path[scriptstyle]
node at (-0.45,0.4) {$X$}
node at (0.45,0.4) {$Y$};
\end{tikzpicture}}}
\quad\xmapsto{\text{disintegration}}\quad
\biggl(\;
\vcenter{\hbox{%
\begin{tikzpicture}[font=\small]
\node[state] (omega) at (0,0) {$\omega_2$};
\node (X) at (0,0.5) {};
\draw (omega) to (X);
\node[scriptstyle] at (0.2,0.35) {$Y$};
\end{tikzpicture}}}
\;,\;
\vcenter{\hbox{%
\begin{tikzpicture}[font=\small]
\node[arrow box] (c) at (0,0) {$c_2$};
\draw  (c) to (0, 0.6);
\draw (c) to (0,-0.6);
\path[scriptstyle]
node at (0.2,0.6) {$X$}
node at (0.2,-0.6) {$Y$};
\end{tikzpicture}}}
\;\biggr)
\]
such that the equation on the left or right below holds, respectively.
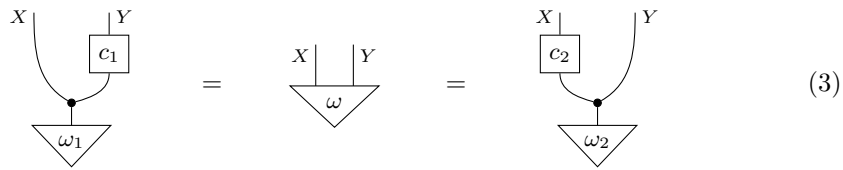
\begin{equation}
\label{eqn:disintegrations}
\vcenter{\hbox{%
\begin{tikzpicture}[font=\small]
\node[state] (omega) at (0,0) {$\omega_1$};
\node[copier] (copier) at (0,0.3) {};
\node[arrow box] (c) at (0.5,0.95) {$c_1$};
\coordinate (X) at (-0.5,1.5);
\coordinate (Y) at (0.5,1.5);
\draw (omega) to (copier);
\draw (copier) to[out=150,in=-90] (X);
\draw (copier) to[out=15,in=-90] (c);
\draw (c) to (Y);
\path[scriptstyle]
node at (-0.7,1.45) {$X$}
node at (0.7,1.45) {$Y$};
\end{tikzpicture}}}
\qquad=\qquad
\vcenter{\hbox{%
\begin{tikzpicture}[font=\small]
\node[state] (omega) at (0,0) {\;$\omega$\;};
\coordinate (X) at (-0.25,0.55) {};
\coordinate (Y) at (0.25,0.55) {};
\draw (omega) ++(-0.25, 0) to (X);
\draw (omega) ++(0.25, 0) to (Y);
\path[scriptstyle]
node at (-0.45,0.4) {$X$}
node at (0.45,0.4) {$Y$};
\end{tikzpicture}}}
\qquad=\qquad
\vcenter{\hbox{%
\begin{tikzpicture}[font=\small]
\node[state] (omega) at (0,0) {$\omega_2$};
\node[copier] (copier) at (0,0.3) {};
\node[arrow box] (c) at (-0.5,0.95) {$c_2$};
\coordinate (X) at (-0.5,1.5);
\coordinate (Y) at (0.5,1.5);
\draw (omega) to (copier);
\draw (copier) to[out=165,in=-90] (c);
\draw (copier) to[out=30,in=-90] (Y);
\draw (c) to (X);
\path[scriptstyle]
node at (-0.7,1.45) {$X$}
node at (0.7,1.45) {$Y$};
\end{tikzpicture}}}
\end{equation}

We immediately see from the equation that
$\omega_{1}$ and $\omega_{2}$ must be marginals of $\omega$:
\[
\vcenter{\hbox{%
\begin{tikzpicture}[font=\small]
\node[state] (omega) at (0,0) {\;$\omega$\;};
\coordinate (X) at (-0.25,0.55);
\node[discarder] (Y) at (0.25,0.25) {};
\draw (omega) ++(-0.25, 0) to (X);
\draw (omega) ++(0.25, 0) to (Y);
\end{tikzpicture}}}
\quad=\quad
\vcenter{\hbox{%
\begin{tikzpicture}[font=\small]
\node[state] (omega) at (0,0) {$\omega_1$};
\node[copier] (copier) at (0,0.3) {};
\node[arrow box] (c) at (0.5,0.95) {$c_1$};
\coordinate (X) at (-0.5,1.55);
\node[discarder] (Y) at (0.5,1.4) {};
\draw (omega) to (copier);
\draw (copier) to[out=150,in=-90] (X);
\draw (copier) to[out=15,in=-90] (c);
\draw (c) to (Y);
\end{tikzpicture}}}
\quad=\quad
\vcenter{\hbox{%
\begin{tikzpicture}[font=\small]
\node[state] (omega) at (0,0) {$\omega_1$};
\node[copier] (copier) at (0,0.3) {};
\coordinate (X) at (-0.5,1.5);
\coordinate (c) at (0.5,0.91);
\node[discarder] (Y) at (c) {};
\draw (omega) to (copier);
\draw (copier) to[out=150,in=-90] (X);
\draw (copier) to[out=15,in=-90] (c);
\end{tikzpicture}}}
\quad=\quad
\vcenter{\hbox{%
\begin{tikzpicture}[font=\small]
\node[state] (omega) at (0,0) {$\omega_1$};
\draw (omega) to (0,0.55);
\end{tikzpicture}}}
\]
and similarly
\[
\vcenter{\hbox{%
\begin{tikzpicture}[font=\small]
\node[state] (omega) at (0,0) {\;$\omega$\;};
\node[discarder] (X) at (-0.25,0.25) {};
\coordinate (Y) at (0.25,0.55) {};
\draw (omega) ++(-0.25, 0) to (X);
\draw (omega) ++(0.25, 0) to (Y);
\end{tikzpicture}}}
\quad=\quad
\vcenter{\hbox{%
\begin{tikzpicture}[font=\small]
\node[state] (omega) at (0,0) {$\omega_2$};
\draw (omega) to (0,0.55);
\end{tikzpicture}}}
\quad.
\]
Therefore
a disintegration of $\omega$
may be referred to by
a channel $c_i$ only,
rather than a pair $(\omega_i,c_i)$.
This leads to the following definition:

\begin{definition}
\label{def:disintegration}
Let $\omega\colon I\to X\otimes Y$ be a joint state.
A channel $c_1\colon X\to Y$
(or $c_2\colon Y\to X$) is called a \emph{disintegration} of $\omega$
if it satisfies the equation~\eqref{eqn:disintegrations}
with $\omega_i$ the marginals of $\omega$.
\end{definition}

Let us look at concrete instances of disintegrations,
in the Kleisli categories $\Kl(\Dst)$
and $\Kl(\Giry)$
from Examples~\ref{ex:KlD}
and~\ref{ex:KlG}.

\begin{example}
\label{ex:disintegration}
Let $\omega\in\Dst(X\times Y)$ be a joint state in $\Kl(\Dst)$.
We write $\omega_{1}
\in\Dst(X)$ for the first marginal, given by $\omega_{1}(x) =
\sum_{y}\omega(x,y)$.
Then a channel $c\colon X\to \Dst(Y)$ is a disintegration of $\omega$
if and only if $\omega(x,y)=c(x)(y)\cdot \omega_1(x)$
for all $x\in X$ and $y\in Y$.
It turns out that there is always such a channel $c$.
We define a channel $c$ by:
\begin{equation}
\label{eqn:discrete:disintegration}
\begin{array}{rcl}
c(x)(y)
& \coloneqq &
\frac{\omega(x,y)}{\omega_{1}(x)}
\qquad\qquad
\text{if}\quad\omega_1(x)\neq 0
\enspace,
\end{array}
\end{equation}

\noindent
and $c(x)\coloneqq\tau$ if $\omega_1(x)=0$, for an arbitrary state $\tau\in\Dst(Y)$.
(Here $\Dst(Y)$ is nonempty, since so is $\Dst(X\times Y)$.)
This indeed defines a channel $c$ satisfying the required equation.
Roughly speaking, disintegration in discrete probability
is nothing but the `definition'
of conditional probability: $\Pr(y|x)=\Pr(x,y)/\Pr(x)$.
There is still some subtlety --- disintegrations need not be unique,
when there are $x\in X$ with $\omega_1(x)=0$.
They are, nevertheless, `almost surely'
unique; see Section~\ref{sec:equality}.
\end{example}

\begin{example}
Disintegrations in measure-theoretic probability, in $\Kl(\Giry)$, are far more difficult.
Let $\omega\in\Giry(X\times Y)$ be a joint state,
with $\omega_1\in\Giry(X)$ the first marginal.
A channel $c\colon X\to \Giry(Y)$ is a disintegration of $\omega$ if and only if
\[
\omega(A\times B) = \int_A c(x)(B)\,\omega_1(\dd x)
\]
for all $A\in \Sigma_X$ and $B\in \Sigma_Y$.  This is the ordinary
notion of \emph{disintegration} (of probability measures), also known
as \emph{regular conditional probability}; see
\textit{e.g.}~\cite{Faden1985,Pollard2002,Panangaden09}.  We see that
there is no obvious way to obtain a channel $c$ here, unlike the
discrete case.  In fact, a disintegration may not
exist~\cite{Stoyanov2014}.  There are, however, a number of results
that guarantee the existence of a disintegration in certain
situations.  We will come back to this issue later in the section.
\end{example}

\emph{Bayesian inversion} is a special form of disintegration,
occurring frequently.
We start from a state $\sigma\colon I\to X$ and
a channel $c\colon X\to Y$. We then integrate
them into a joint state on $X$ and $Y$,
and disintegrate it in the other direction, as below.
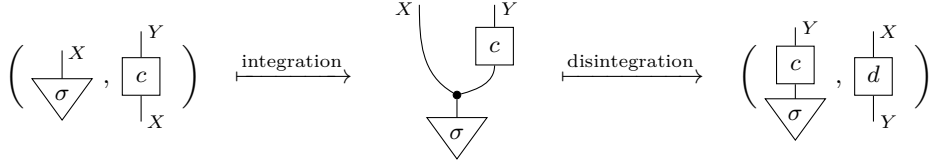
\[
\biggl(\;
\vcenter{\hbox{%
\begin{tikzpicture}[font=\small]
\node[state] (omega) at (0,0) {$\sigma$};
\node (X) at (0,0.5) {};
\draw (omega) to (X);
\node[scriptstyle] at (0.2,0.35) {$X$};
\end{tikzpicture}}}
\;,\;
\vcenter{\hbox{%
\begin{tikzpicture}[font=\small]
\node[arrow box] (c) at (0,0) {$c$};
\draw (c) to (0, 0.6);
\draw (c) to (0,-0.6);
\path[scriptstyle]
node at (0.2,0.6) {$Y$}
node at (0.2,-0.6) {$X$};
\end{tikzpicture}}}
\;\biggr)
\quad\xmapsto{\text{integration}}\quad
\vcenter{\hbox{%
\begin{tikzpicture}[font=\small]
\node[state] (omega) at (0,0) {$\sigma$};
\node[copier] (copier) at (0,0.3) {};
\node[arrow box] (c) at (0.5,0.95) {$c$};
\coordinate (X) at (-0.5,1.5);
\coordinate (Y) at (0.5,1.5);
\draw (omega) to (copier);
\draw (copier) to[out=150,in=-90] (X);
\draw (copier) to[out=15,in=-90] (c);
\draw (c) to (Y);
\path[scriptstyle]
node at (-0.7,1.45) {$X$}
node at (0.7,1.45) {$Y$};
\end{tikzpicture}}}
\quad\xmapsto{\text{disintegration}}\quad
\biggl(\;
\vcenter{\hbox{%
\begin{tikzpicture}[font=\small]
\node[state] (s) at (0,0) {$\sigma$};
\node[arrow box] (c) at (0,0.45) {$c$};
\draw (s) to (c);
\draw (c) to (0,0.95);
\node[scriptstyle] at (0.2,0.95) {$Y$};
\end{tikzpicture}}}
\;,\;
\vcenter{\hbox{%
\begin{tikzpicture}[font=\small]
\node[arrow box] (c) at (0,0) {$d$};
\draw  (c) to (0, 0.6);
\draw (c) to (0,-0.6);
\path[scriptstyle]
node at (0.2,0.6) {$X$}
node at (0.2,-0.6) {$Y$};
\end{tikzpicture}}}
\;\biggr)
\]
We call the disintegration $d\colon Y\to X$ a \emph{Bayesian inversion}
for $\sigma\colon I\to X$ along $c\colon X\to Y$.
By unfolding the definitions,
a channel $d\colon Y\to X$ is
a Bayesian inversion if and only if
\begin{equation}
\label{eq:bayesian-inversion}
\vcenter{\hbox{%
\begin{tikzpicture}[font=\small]
\node[state] (omega) at (0,0) {$\sigma$};
\node[copier] (copier) at (0,0.3) {};
\node[arrow box] (c) at (0.5,0.95) {$c$};
\coordinate (X) at (-0.5,1.5);
\coordinate (Y) at (0.5,1.5);
\draw (omega) to (copier);
\draw (copier) to[out=150,in=-90] (X);
\draw (copier) to[out=15,in=-90] (c);
\draw (c) to (Y);
%
\end{tikzpicture}}}
\quad=\quad
\vcenter{\hbox{%
\begin{tikzpicture}[font=\small]
\node[state] (omega) at (0,-0.55) {$\sigma$};
\node[copier] (copier) at (0,0.3) {};
\node[arrow box] (d) at (-0.5,0.95) {$d$};
\coordinate (X) at (-0.5,1.5);
\coordinate (Y) at (0.5,1.5);
\node[arrow box] (c) at (0,-0.15) {$c$};
\draw (omega) to (c);
\draw (c) to (copier);
\draw (copier) to[out=165,in=-90] (d);
\draw (copier) to[out=30,in=-90] (Y);
\draw (d) to (X);
\end{tikzpicture}}}
\quad.
\end{equation}
The composite $c\circ\sigma$ is also written as $c_*(\sigma)$.
The operation $\sigma\mapsto c_*(\sigma)$,
called \emph{state transformation},
is used to explain
forward inference in \cite{JacobsZ16}.

\begin{example}
Let $\sigma\in\Dst(X)$
and $c\colon X\to \Dst(Y)$ be a state and a channel in $\Kl(\Dst)$.
Then
a channel $d\colon Y\to \Dst(X)$ is
a Bayesian inversion for $\sigma$ along $c$
if and only if
$c(x)(y)\cdot \sigma(x)=
d(y)(x)\cdot c_*(\sigma)(y)$,
where $c_*(\sigma)(y)=\sum_{x'} c(x')(y)\cdot \sigma(x')$.
In a similar manner to Example~\ref{ex:disintegration},
we can obtain such a $d$ by:
\[
d(y)(x)
\coloneqq
\frac{c(x)(y)\cdot \sigma(x)}{c_*(\sigma)(y)}
=
\frac{c(x)(y)\cdot \sigma(x)}
{\sum_{x'} c(x')(y)\cdot \sigma(x')}
\]
for $y\in Y$ with $c_*(\sigma)(y)\neq 0$.
For $y\in Y$ with $c_*(\sigma)(y)=0$,
we may define $d(y)$ to be an arbitrary state in $\Dst(X)$.
We can recognise the above formula as the Bayes formula:
\[
\Pr(x|y)
=\frac{\Pr(y|x)\cdot \Pr(x)}{\Pr(y)}
=\frac{\Pr(y|x)\cdot \Pr(x)}{\sum_{x'}\Pr(y|x')\cdot \Pr(x')}
\enspace.
\]
\end{example}

\begin{example}
\label{ex:inversion}
Let $\sigma\in\Giry(X)$ and $c\colon X\to \Giry(Y)$
be a state and a channel in $\Kl(\Giry)$.
A channel $d\colon Y\to \Giry(X)$ is a Bayesian inversion if and only if
\[
\int_A c(x)(B)\,\sigma(\dd x)
=
\int_B d(y)(A)\,c_*(\sigma)(\dd y)
\]
for all $A\in\Sigma_X$ and $B\in\Sigma_Y$.
Here $c_*(\sigma)\in\Giry(Y)$
is the measure given by $c_*(\sigma)(B)=\int_X c(x)(B) \,\sigma(\dd x)$.
As we see below, Bayesian inversions are in some sense equivalent to disintegrations,
and thus, they are as difficult as disintegrations.
In particular, a Bayesian inversion need not exist.

In practice, however, the state $\sigma$ and channel $c$ are often given
via density functions.
This setting, so-called (absolutely) continuous probability,
makes it easy to compute a Bayesian inversion.
Suppose that $X$ and $Y$ are subspaces of $\RR$, and that
$\sigma$ and $c$ admit density functions as
\[
\sigma(A) = \int_A f(x) \,\dd x
\qquad\qquad
c(x)(B) = \int_B \ell(x,y) \,\dd y
\]
for measurable functions $f\colon X\to \pRR$ and $\ell\colon X\times
Y\to \pRR$.  The conditional probability density $\ell(x,y)$ of $y$
given $x$ is often called the \emph{likelihood} of $x$ given $y$.  By
the familiar Bayes formula for densities --- see
\textit{e.g.}~\cite{BernardoS00} --- the conditional density of $x$
given $y$ is:
\begin{equation*}
\label{eq:bayes-density}
k(y,x)
\coloneqq
\frac{\ell(x,y)\cdot f(x)}
{\int_X \ell(x',y)\cdot f(x')\,\dd x'}
\enspace.
\end{equation*}
This $k$ then gives a channel $d\colon Y\to \Giry(X)$ by
\[
d(y)(A) = \int_A k(y,x)\,\dd x
\]
for each $y\in Y$ such that $\int_X \ell(x',y)\cdot f(x')\,\dd x'\neq 0$.
For the other $y$'s we define $d(y)$ to be some fixed state in $\Giry(X)$.
An elementary calculation verifies that $d$ is indeed a Bayesian inversion
for $\sigma$ along $c$.
Later, in Section~\ref{sec:likelihood},
we generalise this calculation into our abstract setting.
\end{example}

Although Bayesian inversions are
a special case of disintegrations,
we can conversely obtain disintegrations from
Bayesian inversions, as in the proposition below.
Therefore, in some sense the two notions are equivalent.

\begin{proposition}
\label{prop:disintegration-from-inversion}
Let $\omega$ be a state on $X\otimes Y$.
Let $d\colon X\to
X\otimes Y$ be a Bayesian inversion for $\omega$ along the
first projection $\pi_1\colon X\otimes Y\to X$ on the left below.
\[
\pi_1 \,=\;\;
\begin{ctikzpicture}[font=\small]
\node[discarder] (d) at (0,0.4) {};
\draw (0,0) to (d);
\draw (-0.35,0) to +(0,0.9);
\path[scriptstyle]
node at (-0.5,0.15) {$X$}
node at (0.15,0.15) {$Y$};
\end{ctikzpicture}
\qquad\qquad\qquad\qquad\qquad
\pi_2\circ d
\,=\;\;
\begin{ctikzpicture}[font=\small]
\node[discarder] (dis) at (-0.2,0.4) {};
\node[arrow box] (d) at (0,0) {\;\;$d$\;\;};
\draw  (d.north) ++ (-0.2,0) to (dis);
\draw  (d.north) ++ (0.2,0) to (0.2,0.7);
\draw  (d.south) to (0,-0.5);
\path[scriptstyle]
node at (0.15,-0.45) {$X$}
node at (0.35,0.65) {$Y$};
\end{ctikzpicture}
\]
Then the composite
$\pi_2\circ d\colon X\to Y$
shown on the right above is a disintegration of $\omega$.
\end{proposition}
\begin{myproof}
We prove that the first equation in \eqref{eqn:disintegrations} holds for $c_1=\pi_2\circ d$, as follows.
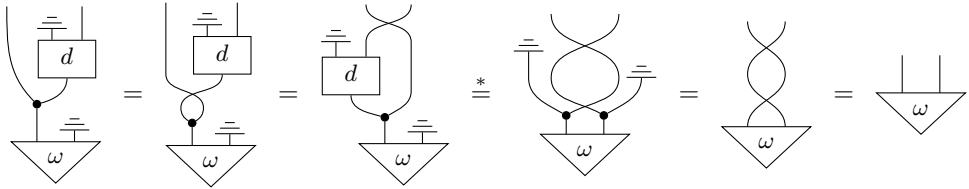
\[
\vcenter{\hbox{%
\begin{tikzpicture}[font=\small]
\node[discarder] (di1) at (0.25,0.15) {};
\node[discarder] (di2) at (-0.05,1.5) {};
\node[state] (s) at (0,0) {\;$\omega$\;};
\node[copier] (c) at (-0.25,0.5) {};
\node[arrow box] (d) at (0.15,1.1) {\;\;$d$\;\;};
\draw (s) ++(-0.25,0) to (c);
\draw (s) ++(0.25,0) to (di1);
\draw (c) to[out=15,in=-90] (d.south);
\draw (c) to[out=135,in=-90] (-0.65,1.8);
\draw (d.north) ++(-0.2,0) to (di2);
\draw (d.north) ++(0.2,0) to (0.35,1.8);
\end{tikzpicture}}}
\;\;=\;\;
\vcenter{\hbox{%
\begin{tikzpicture}[font=\small]
\node[discarder] (di1) at (0.25,0.15) {};
\node[discarder] (di2) at (-0.05,1.6) {};
\node[state] (s) at (0,0) {\;$\omega$\;};
\node[copier] (c) at (-0.25,0.3) {};
\node[arrow box] (d) at (0.15,1.2) {\;\;$d$\;\;};
\draw (s) ++(-0.25,0) to (c);
\draw (s) ++(0.25,0) to (di1);
\draw (c) to[out=150,in=-90] (-0.4,0.5)
to[out=90,in=-90] (d.south);
\draw (c) to[out=30,in=-90] (-0.1,0.5)
to[out=90,in=-90] (-0.6,0.95)
to (-0.6,1.9);
\draw (d.north) ++(-0.2,0) to (di2);
\draw (d.north) ++(0.2,0) to (0.35,1.9);
\end{tikzpicture}}}
\;\;=\;\;
\vcenter{\hbox{%
\begin{tikzpicture}[font=\small]
\node[discarder] (di1) at (0.25,0.15) {};
\node[discarder] (di2) at (-0.9,1.35) {};
\node[state] (s) at (0,0) {\;$\omega$\;};
\node[copier] (c) at (-0.25,0.35) {};
\node[arrow box] (d) at (-0.7,0.9) {\;\;$d$\;\;};
\draw (s) ++(-0.25,0) to (c);
\draw (s) ++(0.25,0) to (di1);
\draw (c) to[out=165,in=-90] (d.south);
\draw (c) to[out=30,in=-90] (0.1,0.7)
to (0.1,1.35)
to[out=90,in=-90] (-0.5,1.85);
\draw (d.north) ++(-0.2,0) to (di2);
\draw (d.north) ++(0.2,0) to (-0.5,1.35)
to[out=90,in=-90] (0.1,1.85);
\end{tikzpicture}}}
\;\;\overset{*}{=}\;\;
\vcenter{\hbox{%
\begin{tikzpicture}[font=\small]
\coordinate (di1) at (-0.75,1.1);
\coordinate (di2) at (0.75,0.75);
\node[discarder] at (di2) {};
\node[discarder] at (di1) {};
\node[state] (s) at (0,0) {\;$\omega$\;};
\node[copier] (c1) at (-0.25,0.25) {};
\node[copier] (c2) at (0.25,0.25) {};
\draw (s) ++(-0.25,0) to (c1);
\draw (s) ++(0.25,0) to (c2);
\draw (c1) to[out=150,in=-90] (-0.75,0.75)
to (di1);
\draw (c2) to[out=30,in=-90] (di2);
\draw (c1) to[out=30,in=-90] (0.45,0.75)
to[out=90,in=-90] (-0.45,1.6);
\draw (c2) to[out=150,in=-90] (-0.45,0.75)
to[out=90,in=-90] (0.45,1.6);
\end{tikzpicture}}}
\;\;=\;\;
\vcenter{\hbox{%
\begin{tikzpicture}[font=\small]
\node[state] (s) at (0,0) {\;$\omega$\;};
\draw (s) ++(-0.25,0)
to[out=90,in=-90] (0.25,0.7)
to[out=90,in=-90] (-0.25,1.4);
\draw (s) ++(0.25,0)
to[out=90,in=-90] (-0.25,0.7)
to[out=90,in=-90] (0.25,1.4);
\end{tikzpicture}}}
\;\;=\;\;
\vcenter{\hbox{%
\begin{tikzpicture}[font=\small]
\node[state] (s) at (0,0) {\;$\omega$\;};
\draw (s) ++(-0.25,0) to (-0.25,0.5);
\draw (s) ++(0.25,0) to (0.25,0.5);
\end{tikzpicture}}}
\]
For the marked equality $\overset{*}{=}$
we used the equation \eqref{eq:bayesian-inversion}
for the Bayesian inversion $d$.
\end{myproof}

We say that
an affine CD-category $\catC$ \emph{admits disintegration}
if for every bipartite state $\omega\colon I\to X\otimes Y$
there exist a disintegration $c_1\colon X\to Y$ of $\omega$.
Note that in such categories there also exists
a disintegration $c_2\colon Y\to X$ of $\omega$ in the other direction,
since it can be obtained as a disintegration of the following state:
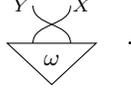
\[
\begin{ctikzpicture}[font=\small]
\node[state] (s) at (0,0) {\;$\omega$\;};
\draw (s) ++(-0.25,0)
to[out=90,in=-90] (0.25,0.5);
\draw (s) ++(0.25,0)
to[out=90,in=-90] (-0.25,0.5);
\path[scriptstyle]
node at (-0.4,0.5) {$Y$}
node at (0.4,0.5) {$X$};
\end{ctikzpicture}
\quad.
\]
By Proposition~\ref{prop:disintegration-from-inversion},
admitting disintegration is equivalent to admitting Bayesian inversion.

In Example~\ref{ex:disintegration},
we have seen that $\Kl(\Dst)$ admits disintegration,
but that in measure-theoretic probability, in $\Kl(\Giry)$,
disintegrations may not exist.
There are however a number of results that guarantee
the existence of disintegrations in specific situations,
see \textit{e.g.}~\cite{Pachl1978,Faden1985}.
We here invoke one of these results
and show that
there is a subcategory of $\Kl(\Giry)$ that admits disintegration.
A measurable space is called a \emph{standard
  Borel space} if it is measurably isomorphic to a Polish space with its
Borel $\sigma$-algebra, or equivalently, if it is measurably isomorphic to
a Borel subspace of $\RR$.
Then the following theorem is standard,
see \textit{e.g.}~\cite[\S5.2]{Pollard2002}
or~\cite[\S5]{Faden1985}.

\begin{theorem}
\label{thm:disintegration-stborel}
Let $X$ be any measurable space and $Y$ be a standard
Borel space.
Then for any state (\textit{i.e.}\ a probability measure)
$\omega\in\Giry(X\times Y)$ in $\Kl(\Giry)$,
there exists a disintegration $c_1\colon X\to \Giry(Y)$ of $\omega$.
\qed
\end{theorem}

Let $\pKrnsb$ be the full subcategory of $\Kl(\Giry)$
consisting of standard Borel spaces as objects.
Clearly $\pKrnsb$ contains the singleton
measurable space $1$,
which is the tensor unit of $\Kl(\Giry)$.
We claim that the product
(in $\Meas$, hence the tensor product in $\Kl(\Giry)$)
of two standard Borel spaces
is again standard Borel.
Let $(X,\Sigma_X)$ and $(Y,\Sigma_Y)$
be standard Borel spaces.
We fix Polish ($=$ separable completely metrisable)
topologies on $X$ and $Y$ that induce
the $\sigma$-algebras $\Sigma_X$
and $\Sigma_Y$, respectively.
Then the product $X\times Y$, as topological spaces,
is Polish \cite[Lemma~1.17]{Doberkat2007}.
The Borel $\sigma$-algebra of the Polish space $X\times Y$
coincides with
the $\sigma$-algebra on the product $X\times Y$
in $\Meas$ \cite[Proposition~2.8]{Panangaden09}.
Hence $X\times Y$ is standard Borel.
Therefore
the subcategory $\pKrnsb\hookrightarrow \Kl(\Giry)$
is closed under the monoidal structure of $\Kl(\Giry)$,
so that $\pKrnsb$ is an affine CD-category.
Then the previous theorem immediately shows:

\begin{corollary}
The category $\pKrnsb$ admits disintegration.
\qed
\end{corollary}

We note that $\pKrnsb$ can also be seen as the Kleisli category
of the Giry monad restricted on the category of standard Borel spaces,
because $\Giry(X)$ is standard Borel whenever $X$ is standard Borel.
To see this, let $(X,\Sigma_X)$ be
a standard Borel space.
If we fix a Polish topology on $X$ that
induces $\Sigma_X$,
then $\Giry(X)$ also forms a Polish space
with the \emph{topology of weak convergence},
see \cite{Giry1982} or \cite[\S1.5.2]{Doberkat2007}.
The $\sigma$-algebra on $\Giry(X)$,
defined in Example~\ref{ex:KlG},
coincides with
the Borel $\sigma$-algebra
with respect to the Polish topology
\cite[Proposition~1.80]{Doberkat2007}.
Therefore $\Giry(X)$ is standard Borel.

Since there are various `existence' theorems like
Theorem~\ref{thm:disintegration-stborel}, there may be other
subcategories of $\Kl(\Giry)$ that admit disintegration.  A likely
candidate is the category of perfect probabilistic mappings
in~\cite{CulbertsonS2014}.  We do not go into this question here,
since $\pKrnsb$ suffices for the present paper.

\section{Example: naive Bayesian classifiers via inversion}
\label{sec:classifier}\tikzextname{s_classifier}

Bayesian classification is a well-known technique in machine learning
that produces a distribution over data classifications, given certain
sample data. The distribution describes the probability, for each data
(classification) category, that the sample data is in that
category. Here we consider an example of `naive' Bayesian
classification, where the features are assumed to be independent.  We
consider a standard classification example from the literature which
forms an ideal setting to illustrate the use of both disintegration
and Bayesian inversion. Disintegration is used to extract channels
from a given table, and subsequently Bayesian inversion is applied to
(the tuple of) these channels to obtain the actual classification.
The use of channels and disintegration/inversion in this
classification setting is new, as far as we know.

For the description of the relevant operations in this example we use
notation for marginalisation and disintegration that we borrowed from
the \EfProb library~\cite{ChoJ17}. There are many ways to marginalise
an $n$-partite state, namely one for each subset of the wires $\{1, 2,
\ldots, n\}$. Such a subset can be described as a \emph{mask},
consisting of a list of $n$ zero's or one's, where a zero at position
$i$ means that the $i$-th wire/component is marginalised out, and a
one at position $i$ means that it remains. Such a mask $M = [b_{1},
  \ldots, b_{n}]$ with $b_{i}\in\{0,1\}$ is used as a post-fix
selection operation in $\marg{\omega}{M}$ on an $n$-partite state
$\omega$. An example explains it all:
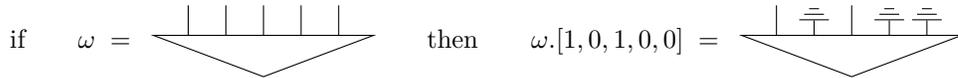
\[ \mbox{if}\qquad \omega\; = \;\;
\vcenter{\hbox{%
\begin{tikzpicture}[font=\small]
\node[state] (omega) at (0,0) {\hspace*{3em}};
\coordinate (X1) at (-1.0,0.4);
\coordinate (X2) at (-0.5,0.4) {};
\coordinate (X3) at (0.0,0.4) {};
\coordinate (X4) at (0.5,0.4) {};
\coordinate (X5) at (1.0,0.4) {};
\draw (omega) ++(-1.0, 0) to (X1);
\draw (omega) ++(-0.5, 0) to (X2);
\draw (omega) ++(0.0, 0) to (X3);
\draw (omega) ++(0.5, 0) to (X4);
\draw (omega) ++(1.0, 0) to (X5);
\end{tikzpicture}}}
\qquad\mbox{then}\qquad
\marg{\omega}{[1,0,1,0,0]} \; = \;\;
\vcenter{\hbox{%
\begin{tikzpicture}[font=\small]
\node[state] (omega) at (0,0) {\hspace*{3em}};
\coordinate (X1) at (-1.0,0.4);
\node [discarder] (ground1) at (-0.5,0.2) {};
\coordinate (X2) at (0.0,0.4) {};
\node [discarder] (ground2) at (0.5,0.2) {};
\node [discarder] (ground3) at (1.0,0.2) {};
\draw (omega) ++(-1.0, 0) to (X1);
\draw (omega) ++(-0.5, 0) to (ground1);
\draw (omega) ++(0.0, 0) to (X2);
\draw (omega) ++(0.5, 0) to (ground2);
\draw (omega) ++(1.0, 0) to (ground3);
\end{tikzpicture}}} \]

\noindent In a similar way one can disintegrate an $n$-partite state
in $2^n$ may ways, where a mask of length $n$ is now used to describe
which wires are used as input to the extracted channel and which ones
as output. We write $\disint{\omega}{M}$ for such a disintegration,
where $M$ is a mask, as above. A systematic description will be given
in Section~\ref{sec:conditionalindependence} below.

In practice it is often useful to be able to marginalise first, and
disintegrate next. The general description in $n$-ary form is a bit
complicated, so we use an example for $n=5$. We shall label the wires
with $x_i$, as on the left below. We seek the conditional probability
written conventionally as $c = \omega[x_{1}, x_{4} \mid x_{2}, x_{5}]$
on the right below.
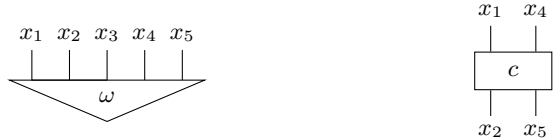
\[ \vcenter{\hbox{%
\begin{tikzpicture}[font=\small]
\node[state] (omega) at (0,0) {\hspace*{1em}$\omega$\hspace*{1em}};
\coordinate (X1) at (-1.0,0.4);
\coordinate (X2) at (-0.5,0.4) {};
\coordinate (X3) at (0.0,0.4) {};
\coordinate (X4) at (0.5,0.4) {};
\coordinate (X5) at (1.0,0.4) {};
\node at (-1.0,0.6) {$x_{1}$};
\node at (-0.5,0.6) {$x_{2}$};
\node at (0.0,0.6) {$x_{3}$};
\node at (0.5,0.6) {$x_{4}$};
\node at (1.0,0.6) {$x_{5}$};
\draw (omega) -- (-1.0, 0) to (X1);
\draw (omega) ++(-0.5, 0) to (X2);
\draw (omega) ++(0.0, 0) to (X3);
\draw (omega) ++(0.5, 0) to (X4);
\draw (omega) ++(1.0, 0) to (X5);
\end{tikzpicture}}}
\hspace*{10em}
\vcenter{\hbox{%
\begin{tikzpicture}[font=\small]
\node[arrow box] (c) at (0,0) {\hspace*{1em}$c$\hspace*{1em}};
\coordinate (X1) at (-0.3,0.6);
\coordinate (X4) at (0.3,0.6) {};
\coordinate (X2) at (-0.3,-0.6) {};
\coordinate (X5) at (0.3,-0.6) {};
\draw (c.north)++(-0.3,0) to (X1);
\draw (c.north)++(0.3,0) to (X4);
\draw (c.south)++(-0.3,0) to (X2);
\draw (c.south)++(0.3,0) to (X5);
\node at (-0.3,0.8) {$x_{1}$};
\node at (0.3,0.8) {$x_{4}$};
\node at (-0.3,-0.8) {$x_{2}$};
\node at (0.3,-0.8) {$x_{5}$};
\end{tikzpicture}}}
\]

\noindent This channel $c$ must satisfy:
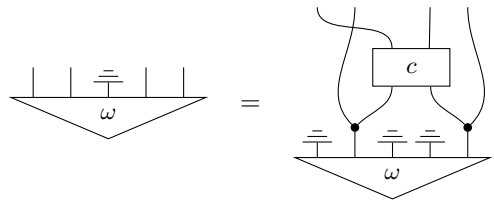
\[ \vcenter{\hbox{%
\begin{tikzpicture}[font=\small]
\node[state] (omega) at (0,0) {\hspace*{1em}$\omega$\hspace*{1em}};
\coordinate (X1) at (-1.0,0.4);
\coordinate (X2) at (-0.5,0.4) {};
\node [discarder] (ground) at (0.0,0.2) {};
\coordinate (X4) at (0.5,0.4) {};
\coordinate (X5) at (1.0,0.4) {};
\draw (omega) ++(-1.0, 0) to (X1);
\draw (omega) ++(-0.5, 0) to (X2);
\draw (omega) ++(0.0, 0) to (ground);
\draw (omega) ++(0.5, 0) to (X4);
\draw (omega) ++(1.0, 0) to (X5);
\end{tikzpicture}}}
\quad = \quad
\vcenter{\hbox{%
\begin{tikzpicture}[font=\small]
\coordinate (X1) at (-1.0,2.0);
\coordinate (X2) at (-0.5,2.0) {};
\coordinate (X4) at (0.5,2.0) {};
\coordinate (X5) at (1.0,2.0) {};
\node[state] (omega) at (0,0) {\hspace*{1em}$\omega$\hspace*{1em}};
\node [discarder] (ground1) at (-1.0,0.2) {};
\node[copier] (c2) at (-0.5,0.4) {};
\node [discarder] (ground3) at (0.0,0.2) {};
\node [discarder] (ground4) at (0.5,0.2) {};
\node[copier] (c5) at (1.0,0.4) {};
\node[arrow box] (c) at (0.25,1.2) {\hspace*{1em}$c$\hspace*{1em}};
\draw (omega) ++(-1.0, 0) to (ground1);
\draw (omega) ++(-0.5, 0) to (c2);
\draw (omega) ++(0.0, 0) to (ground3);
\draw (omega) ++(0.5, 0) to (ground4);
\draw (omega) ++(1.0, 0) to (c5);
\draw (c2) to[out=45,in=-90] (c.-135);
\draw (c2) to[out=135,in=-90] (X2);
\draw (c5) to[out=135,in=-90] (c.-45);
\draw (c5) to[out=45,in=-90] (X5);
\draw (c.135) to[out=90,in=-90] (X1);
\draw (c.47) to (X4);
\end{tikzpicture}}}
\]

\noindent This picture shows how to obtain the channel $c$ from
$\omega$: we first marginalise to restrict to the relevant wires
$x_{1}, x_{2}, x_{4}, x_{5}$. This is written as
$\marg{\omega}{[1,1,0,1,1]}$. Subsequently we disintegrate with
$x_{1},x_{4}$ as output and $x_{2},x_{5}$ as input. Hence:
\[ \begin{array}{rcl}
c
& \coloneqq &
\disint{\marg{\omega}{[1,1,0,1,1]}}{[0,1,0,1]}
\\
& \eqqcolon &
\extract{\omega}{[1,0,0,1,0]}{[0,1,0,0,1]}
   \quad \mbox{as we shall write in the sequel.}
\end{array} \]

\noindent We see that the latter post-fix $\big[[1,0,0,1,0] \,\big|\,
  [0,1,0,0,1]\big]$ is a `variable free' version of the traditional
notation $[x_{1}, x_{4} \mid x_{2}, x_{5}]$, selecting the relevant
positions.

We have now prepared the ground and can turn to the classification
example that we announced.  It involves the classification of
`playing' (yes or no) for certain weather data, used
in~\cite{WittenFH11}.  We shall first go through the discrete example
in some detail. The relevant data are in the table in
Figure~\ref{fig:discreteplay}.  The question is: given this table,
what can be said about the probability of playing if the outlook is
\emph{Sunny}, the temperature is \emph{Cold}, the humidity is
\emph{High} and it is Windy?

\begin{figure}
\begin{center}
{\setlength\tabcolsep{2em}\renewcommand{\arraystretch}{0.9}
\begin{tabular}{c|c|c|c|c}
\textbf{Outlook} & \textbf{Temperature} & \textbf{Humidity} &
   \textbf{Windy} & \textbf{Play}
\\
\hline\hline
Sunny & hot & high & false & no
\\
Sunny & hot & high & true & no
\\
Overcast & hot & high & false & yes
\\
Rainy & mild & high & false & yes
\\
Rainy & cool & normal & false & yes
\\
Rainy & cool & normal & true & no
\\
Overcast & cool & normal & true & yes
\\
Sunny & mild & high & false & no
\\
Sunny & cool & normal & false & yes
\\
Rainy & mild & normal & false & yes
\\
Sunny & mild & normal & true & yes
\\
Overcast & mild & high & true & yes
\\
Overcast  & hot & normal & false & yes
\\
Rainy  & mild & high & true & no
\end{tabular}}
\end{center}
\caption{Weather and play data, copied from~\cite{WittenFH11}.}
\label{fig:discreteplay}
\end{figure}

Our plan is to first organise these table data into four channels
$d_{O}, d_{T}, d_{H}, d_{W}$ in a network of the form:
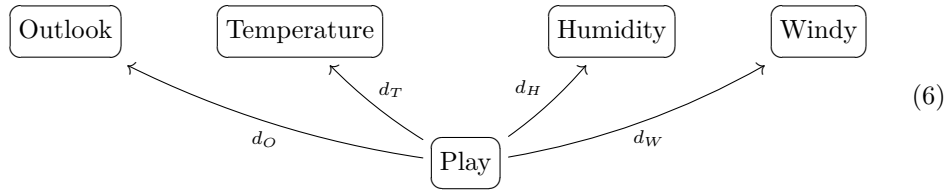
\begin{equation}
\label{diag:discreteplay}
\vcenter{\xymatrix@C-1pc{
\ovalbox{\strut Outlook} & & \ovalbox{\strut Temperature} & & 
   \ovalbox{\strut Humidity} & & \ovalbox{\strut Windy}
\\
& & & \ovalbox{\strut Play}\ar@/^2ex/[ulll]^-{d_{O}}\ar@/_1ex/[ur]^-{d_{H}}\ar@/_2ex/[urrr]_-{d_{W}}\ar@{-<}@/^1ex/[ul]_-{d_{T}} & & &
}}
\end{equation}

\noindent We start by extracting the underlying sets for 
for the categories in the table in Figure~\ref{fig:discreteplay}. We
choose abbreviations for the entries in each of the categories. 
\[ \begin{array}{rclcrclcrclcrclcrcl}
O
& = &
\{s, o, r\}
& \quad &
W
& = &
\{t, f\}
& \quad &
T
& = &
\{h, m, c\}
& \quad &
P
& = &
\{y, n\}
& \quad &
H
& = &
\{h, n\}.
\end{array} \]

\noindent These sets are combined into a single product domain:
\[ \begin{array}{rcl}
D
& = &
O \times T \times H \times W \times P.
\end{array} \]

\noindent It combines the five columns in
Figure~\ref{fig:discreteplay}.  The table itself in the figure is
represented as a uniform distribution $\tau \in\Dst(D)$. This
distribution has 14 entries --- like in the table --- and looks as
follows.
\[ \begin{array}{rcl}
\tau
& = &
\frac{1}{14}\ket{s,h,h,f,n} + \frac{1}{14}\ket{s,h,h,t,n} + \cdots +
   \frac{1}{14}\ket{r,m,h,t,n}.
\end{array} \]

We extract the four channels in Diagram~\eqref{diag:discreteplay} via
appropriate disintegrations, from the Play column to the Outlook /
Temperature / Humidity / Windy columns.
\[ \begin{array}{rclcrcl}
d_{O}
& = &
\extract{\tau}{[1,0,0,0,0]}{[0,0,0,0,1]}
& \hspace*{3em} &
d_{T} 
& = &
\extract{\tau}{[0,1,0,0,0]}{[0,0,0,0,1]}
\\
d_{H}
& = &
\extract{\tau}{[0,0,1,0,0]}{[0,0,0,0,1]}
& &
d_{W}
& = &
\extract{\tau}{[0,0,0,1,0]}{[0,0,0,0,1]}.
\end{array} \]

\noindent Thus, as described in the beginning of this section, the
`outlook' channel $d_{O} \colon P \rightarrow \Dst(O)$ is extracted by
first marginalising the table $\tau$ to the relevant (first and last)
wires, and then disintegrating. Explicitly, $d_{O}$ is
$\disint{\marg{\tau}{[1,0,0,0,1]}}{[0,1]}$ and satisfies:
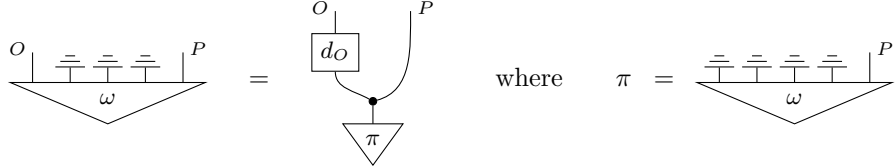
\[ \vcenter{\hbox{%
\begin{tikzpicture}[font=\small]
\node[state] (omega) at (0,0) {\hspace*{1em}$\omega$\hspace*{1em}};
\coordinate (X1) at (-1.0,0.4) {};
\node [discarder] (ground2) at (-0.5,0.2) {};
\node [discarder] (ground3) at (0.0,0.2) {};
\node [discarder] (ground4) at (0.5,0.2) {};
\coordinate (X5) at (1.0,0.4) {};
\draw (omega) ++(-1.0, 0) to (X1);
\draw (omega) ++(-0.5, 0) to (ground2);
\draw (omega) ++(0.0, 0) to (ground3);
\draw (omega) ++(0.5, 0) to (ground4);
\draw (omega) ++(1.0, 0) to (X5);
\path[scriptstyle]
node at (-1.2,0.45) {$O$}
node at (1.2,0.45) {$P$};
\end{tikzpicture}}}
\quad=\quad
\vcenter{\hbox{%
\begin{tikzpicture}[font=\small]
\node[state] (omega) at (0,0) {$\pi$};
\node[copier] (copier) at (0,0.3) {};
\node[arrow box] (c) at (-0.5,0.95) {$d_O$};
\coordinate (X) at (-0.5,1.5);
\coordinate (Y) at (0.5,1.5);
\draw (omega) to (copier);
\draw (copier) to[out=150,in=-90] (c);
\draw (copier) to[out=15,in=-90] (Y);
\draw (c) to (X);
\path[scriptstyle]
node at (-0.7,1.55) {$O$}
node at (0.7,1.55) {$P$};
\end{tikzpicture}}}
\qquad\mbox{where}\qquad
\pi
\;\;=\;\;
\vcenter{\hbox{%
\begin{tikzpicture}[font=\small]
\node[state] (omega) at (0,0) {\hspace*{1em}$\omega$\hspace*{1em}};
\node [discarder] (ground1) at (-1.0,0.2) {};
\node [discarder] (ground2) at (-0.5,0.2) {};
\node [discarder] (ground3) at (0.0,0.2) {};
\node [discarder] (ground4) at (0.5,0.2) {};
\coordinate (X5) at (1.0,0.4) {};
\draw (omega) ++(-1.0, 0) to (ground1);
\draw (omega) ++(-0.5, 0) to (ground2);
\draw (omega) ++(0.0, 0) to (ground3);
\draw (omega) ++(0.5, 0) to (ground4);
\draw (omega) ++(1.0, 0) to (X5);
\path[scriptstyle]
node at (1.2,0.45) {$P$};
\end{tikzpicture}}}
\] 

\noindent In a next step we combine these four channels into a single channel
$d\colon P\rightarrow O\times T\times H\times W$ via tupling:
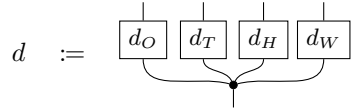
\[ d
\quad\coloneqq\quad
\vcenter{\hbox{%
\begin{tikzpicture}[font=\small]
\node[copier] (copier) at (0,0.0) {};
\node[arrow box] (o) at (-1.2,0.6) {$d_O$};
\node[arrow box] (t) at (-0.4,0.6) {$d_T$};
\node[arrow box] (h) at (0.4,0.6) {$d_H$};
\node[arrow box] (w) at (1.2,0.6) {$d_W$};
\draw (copier) to (0.0,-0.3);
\draw (copier) to[out=165,in=-90] (o);
\draw (copier) to[out=115,in=-90] (t);
\draw (copier) to[out=65,in=-90] (h);
\draw (copier) to[out=15,in=-90] (w);
\draw (o) to (-1.2,1.1);
\draw (t) to (-0.4,1.1);
\draw (h) to (0.4,1.1);
\draw (w) to (1.2,1.1);
\end{tikzpicture}}} \]

\noindent The answer that we are looking for will be obtained by
Bayesian inversion of this channel $d$ wrt.\ the above fifth marginal
Play state $\pi = \marg{\tau}{[0,0,0,0,1]} = \frac{9}{14}\ket{y} +
\frac{5}{14}\ket{n} \in \Dst(P)$.  We write this inversion as a channel $e
\colon O\times T\times H\times W \rightarrow P$. It satisfies, by
construction, according to the pattern
in~\eqref{eq:bayesian-inversion}:
\[ \vcenter{\hbox{%
\begin{tikzpicture}[font=\small]
\node[state] (omega) at (1.2,-1.0) {$\pi$};
\node[copier] (copier1) at (1.2,-0.8) {};
\node[copier] (copier4) at (0.0,-0.2) {};
\node[arrow box] (o) at (-1.2,0.6) {$d_O$};
\node[arrow box] (t) at (-0.4,0.6) {$d_T$};
\node[arrow box] (h) at (0.4,0.6) {$d_H$};
\node[arrow box] (w) at (1.2,0.6) {$d_W$};
\draw (omega) to (copier1);
\draw (copier1) to[out=150,in=-90] (copier4);
\draw (copier1) to[out=15,in=-90] (2.0,1.1);
\draw (copier4) to[out=165,in=-90] (o);
\draw (copier4) to[out=115,in=-90] (t);
\draw (copier4) to[out=65,in=-90] (h);
\draw (copier4) to[out=15,in=-90] (w);
\draw (o) to (-1.2,1.1);
\draw (t) to (-0.4,1.1);
\draw (h) to (0.4,1.1);
\draw (w) to (1.2,1.1);
\end{tikzpicture}}}
\quad=\quad
\vcenter{\hbox{%
\begin{tikzpicture}[font=\small]
\node[state] (omega) at (0,-0.2) {$\pi$};
\node[copier] (copier) at (0,0.0) {};
\node[arrow box] (o) at (-1.2,0.6) {$d_O$};
\node[arrow box] (t) at (-0.4,0.6) {$d_T$};
\node[arrow box] (h) at (0.4,0.6) {$d_H$};
\node[arrow box] (w) at (1.2,0.6) {$d_W$};
\node[copier] (copier1) at ([yshiftu=0.5]o) {};
\node[copier] (copier2) at ([yshiftu=0.5]t) {};
\node[copier] (copier3) at ([yshiftu=0.5]h) {};
\node[copier] (copier4) at ([yshiftu=0.5]w) {};
\node[arrow box] (e) at (2.0,2.3) {\hspace*{2.2em}$e$\hspace*{2.2em}};
\draw (omega) to (copier);
\draw (copier) to[out=165,in=-90] (o);
\draw (copier) to[out=115,in=-90] (t);
\draw (copier) to[out=65,in=-90] (h);
\draw (copier) to[out=15,in=-90] (w);
\draw (o) to (copier1);
\draw (t) to (copier2);
\draw (h) to (copier3);
\draw (w) to (copier4);
\draw (copier1) to[out=165,in=-90] (-1.5,2.8);
\draw (copier1) to[out=65,in=-90] (e.-160);
\draw (copier2) to[out=165,in=-90] (-0.8,2.8);
\draw (copier2) to[out=45,in=-90] (e.-135);
\draw (copier3) to[out=165,in=-90] (-0.1,2.8);
\draw (copier3) to[out=25,in=-90] (e.-45);
\draw (copier4) to[out=165,in=-90] (0.6,2.8);
\draw (copier4) to[out=5,in=-90] (e.-20);
\draw (e) to ([yshiftu=0.5]e);
\end{tikzpicture}}}
\] 

\noindent We can finally answer the original classification
question. The assumptions --- \emph{Sunny} outlook,
\emph{Cold} temperature, \emph{High} humidity, \emph{true} windiness
--- are used as input to the inversion channel $e$. This yields the
probability of play that we are looking for:
\[ \begin{array}{rcl}
e(s,c,H,t)
& = &
0.205\ket{y} + 0.795\ket{n}.
\end{array} \]

\noindent The resulting classification probability\footnote{This
  outcome has been calculated with the tool \EfProb~\cite{ChoJ17}
  that can do both disintegration and inversion.} of $0.205$ coincides
with the probability of $20.5\%$ that is computed in~\cite{WittenFH11}
--- without the above channel-based approach.


One could complain that our approach is `too' abstract, since it
remains implicit what these extracted channels do. We elaborate the
outlook channel $d_{O} \colon P \rightarrow O$. For the two elements
in $P = \{y,n\}$ we have:
\[ \begin{array}{rclcrcl}
d_{O}(y)
& = &
\frac{2}{9}\ket{s} + \frac{4}{9}\ket{o} + \frac{3}{9}\ket{r}
& \qquad\qquad &
d_{O}(n)
& = &
\frac{3}{5}\ket{s} + 0\ket{o} + \frac{2}{5}\ket{r}.
\end{array} \]

\noindent These outcomes arise from the general
formula~\eqref{eqn:discrete:disintegration}. But we can also
understand them at a more concrete level: for the first distribution
$d_{O}(y)$ we need to concentrate on the 9 lines in
Figure~\ref{fig:discreteplay} for which Play is \emph{yes}; in these
lines, in the first Outlook column, 2 out of 9 entries are
\emph{Sunny}, 4 out of 9 are \emph{Overcast}, and 3 out of 9 are
\emph{Rainy}. This corresponds to the above first distribution
$d_{O}(y)$. Similarly, the second distribution $d_{O}(n)$ captures the
Outlook for the 5 lines where Play is \emph{no}: 3 out of 5 are
\emph{Sunny} and 2 out of 5 are \emph{Rainy}.

\section{Almost equality of channels}\label{sec:equality}
\tikzextname{s_eq}

This section explains how the standard notion of
`equal up to negligible sets' or
`equal almost everywhere' (with respect to a measure)
can be expressed abstractly using string diagrams.
Via this equality relation Bayesian inversion can be characterised
very neatly, following~\cite{ClercDDG2017}.
We consider an affine CD-category,
continuing in the setting of Section~\ref{sec:disintegration}.

\begin{definition}
\label{def:almost-equal}
Let $c,d\colon X\to Y$ be two parallel channels, and $\sigma\colon
I\to X$ be a state on their domain.  We say that $c$
is \emph{$\sigma$-almost equal} to $d$, written as $c\equpto[\sigma]d$ if
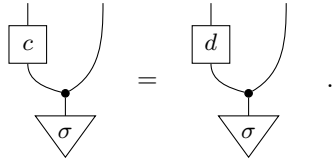
\[
\begin{ctikzpicture}[font=\small]
\node[state] (omega) at (0,0) {$\sigma$};
\node[copier] (copier) at (0,0.3) {};
\node[arrow box] (c) at (-0.5,0.95) {$c$};
\coordinate (X) at (0.5,1.5);
\coordinate (Y) at (-0.5,1.5);
\draw (omega) to (copier);
\draw (copier) to[out=30,in=-90] (X);
\draw (copier) to[out=165,in=-90] (c);
\draw (c) to (Y);
\end{ctikzpicture}
\quad=\quad
\begin{ctikzpicture}[font=\small]
\node[state] (omega) at (0,0) {$\sigma$};
\node[copier] (copier) at (0,0.3) {};
\node[arrow box] (c) at (-0.5,0.95) {$d$};
\coordinate (X) at (0.5,1.5);
\coordinate (Y) at (-0.5,1.5);
\draw (omega) to (copier);
\draw (copier) to[out=30,in=-90] (X);
\draw (copier) to[out=165,in=-90] (c);
\draw (c) to (Y);
\end{ctikzpicture}
\quad\enspace.
\]
It is obvious that $\equpto[\sigma]$ is an equivalence relation
on channels of type $X\to Y$.
When $S$ is a set of arrows of type $X\to Y$,
we write $S/\sigma$ for the quotient $S/\!\equpto[\sigma]$.
\end{definition}

\noindent
To put it more intuitively,
we have $c\equpto[\sigma]d$
iff $c$ and $d$ can be identified
whenever the input wires are connected to $\sigma$,
possibly through copiers.
For instance, using the associativity and commutativity of copiers,
by $c\equpto[\sigma]d$ we may reason as:
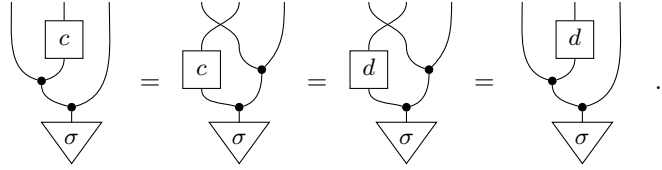
\[
\begin{ctikzpicture}[font=\small]
\node[state] (s) at (-0.5,-0.2) {$\sigma$};
\node[copier] (c2) at (-0.5,0) {};
\node[copier] (c1) at (-0.9,0.35) {};
\node[arrow box] (a) at (-0.6,0.9) {$c$};
\draw (c2) to[out=165,in=-90] (c1);
\draw (c2) to[out=15,in=-90] (0,1.4);
\draw (c1) to[out=15,in=-90] (a);
\draw (c1) to[out=165,in=-90] (-1.3,1.4);
\draw (c2) to (s);
\draw (a) to (-0.6,1.4);
\end{ctikzpicture}
\;\;=\;\;
\begin{ctikzpicture}[font=\small]
\node[state] (s) at (-0.5,-0.1) {$\sigma$};
\node[copier] (c2) at (-0.5,0.1) {};
\node[copier] (c1) at (-0.2,0.6) {};
\node[arrow box] (a) at (-1,0.6) {$c$};
\draw (c2) to[out=15,in=-90] (c1);
\draw (c2) to[out=165,in=-90] (a)
(a) to (-1,0.9)
to[out=90,in=-90] (-0.5,1.5);
\draw (c1) to[out=45,in=-90] (0.1,1.5);
\draw (c1) to[out=165,in=-90] (-0.5,0.9)
to[out=90,in=-90] (-1,1.5);
\draw (c2) to (s);
\end{ctikzpicture}
\;\;=\;\;
\begin{ctikzpicture}[font=\small]
\node[state] (s) at (-0.5,-0.1) {$\sigma$};
\node[copier] (c2) at (-0.5,0.1) {};
\node[copier] (c1) at (-0.2,0.6) {};
\node[arrow box] (a) at (-1,0.6) {$d$};
\draw (c2) to[out=15,in=-90] (c1);
\draw (c2) to[out=165,in=-90] (a)
(a) to (-1,0.9)
to[out=90,in=-90] (-0.5,1.5);
\draw (c1) to[out=45,in=-90] (0.1,1.5);
\draw (c1) to[out=165,in=-90] (-0.5,0.9)
to[out=90,in=-90] (-1,1.5);
\draw (c2) to (s);
\end{ctikzpicture}
\;\;=\;\;
\begin{ctikzpicture}[font=\small]
\node[state] (s) at (-0.5,-0.2) {$\sigma$};
\node[copier] (c2) at (-0.5,0) {};
\node[copier] (c1) at (-0.9,0.35) {};
\node[arrow box] (a) at (-0.6,0.9) {$d$};
\draw (c2) to[out=165,in=-90] (c1);
\draw (c2) to[out=15,in=-90] (0,1.4);
\draw (c1) to[out=15,in=-90] (a);
\draw (c1) to[out=165,in=-90] (-1.3,1.4);
\draw (c2) to (s);
\draw (a) to (-0.6,1.4);
\end{ctikzpicture}
\quad.
\]
In particular, $c\equpto[\sigma]d$ if and only if
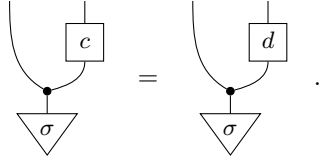
\[
\begin{ctikzpicture}[font=\small]
\node[state] (omega) at (0,0) {$\sigma$};
\node[copier] (copier) at (0,0.3) {};
\node[arrow box] (c) at (0.5,0.95) {$c$};
\coordinate (X) at (-0.5,1.5);
\coordinate (Y) at (0.5,1.5);
\draw (omega) to (copier);
\draw (copier) to[out=150,in=-90] (X);
\draw (copier) to[out=15,in=-90] (c);
\draw (c) to (Y);
\end{ctikzpicture}
\quad=\quad
\begin{ctikzpicture}[font=\small]
\node[state] (omega) at (0,0) {$\sigma$};
\node[copier] (copier) at (0,0.3) {};
\node[arrow box] (c) at (0.5,0.95) {$d$};
\coordinate (X) at (-0.5,1.5);
\coordinate (Y) at (0.5,1.5);
\draw (omega) to (copier);
\draw (copier) to[out=150,in=-90] (X);
\draw (copier) to[out=15,in=-90] (c);
\draw (c) to (Y);
\end{ctikzpicture}
\quad.
\]

Now the following is an obvious consequence from the definition.

\begin{proposition}
If both $c,d\colon X\to Y$ are disintegrations
of a joint state $\omega\colon I\to X\otimes Y$,
then $c\equpto[\omega_1] d$,
where $\omega_1\colon I\to X$ is the first marginal of $\omega$.
\qed
\end{proposition}

For channels $f,g\colon X\to\Dst(Y)$
and a state $\sigma\in\Dst(X)$ in $\Kl(\Dst)$,
it is easy to see that
$f\equpto[\sigma] g$  if and only if
$f(x)(y)\cdot \sigma(x)=g(x)(y)\cdot \sigma(x)$ for all $x\in X$ and $y\in Y$
if and only if
$f(x)=g(x)$ for any $x\in X$ with $\sigma(x)\ne 0$.
Almost equality in $\Kl(\Giry)$
is less trivial but characterised in an expected way.

\begin{proposition}
\label{prop:equality-up-to-giry}
Let $f,g\colon X\to \Giry (Y)$ be channels
and $\mu\in\Giry(X)$ a state in $\Kl(\Giry)$.
Then $f\equpto[\mu] g$ if and only if
for any $B\in \Sigma_Y$,
$f(-)(B)=g(-)(B)$ $\mu$-almost everywhere.
\end{proposition}
\begin{myproof}
By expanding the definition,
$f\equpto[\mu] g$ if and only if
\[
\int_A f(x)(B)\,\mu(\dd x)=\int_A g(x)(B)\,\mu(\dd x)
\]
for all $A\in \Sigma_X$ and $B\in \Sigma_Y$.
This is equivalent to
$f(-)(B)=g(-)(B)$ $\mu$-almost everywhere
for all $B\in\Sigma_Y$, see~\cite[131H]{FremlinAll}.
\end{myproof}

Almost-everywhere equality of probability kernels
$f,g\colon X\to \Giry (Y)$ is often formulated
by the stronger condition that $f=g$ $\mu$-almost everywhere.
The next proposition shows that the stronger variant
is equivalent under a reasonable assumption
(any standard Borel space is countably generated, for example).

\begin{proposition}
In the setting of the previous proposition,
additionally assume that
the measurable space $Y$ is countably generated.
Then $f\equpto[\mu] g$ if and only if
$f=g$ $\mu$-almost everywhere.
\end{proposition}
\begin{proof}
Let the $\sigma$-algebra $\Sigma_Y$ on $Y$ be generated by
a countable family $(B_n)_n$. We may assume that $(B_n)_n$
is a $\pi$-system, \textit{i.e.}\ a family closed under binary intersections.
Let $A_n=\set{x\in X|f(x)(B_n)=g(x)(B_n)}$, and $A=\bigcap_n A_n$.
Each $A_n$ is $\mu$-conegligible, and thus $A$ is $\mu$-conegligible.
For each $x\in A$, we have $f(x)(B_n)=g(x)(B_n)$ for all $n$.
By application of the Dynkin $\pi$-$\lambda$ theorem,
it follows that $f(x)=g(x)$.
Therefore $f=g$ $\mu$-almost everywhere.
\end{proof}

We can now present a fundamental result
from~\cite[\S3.3]{ClercDDG2017} in our abstract setting.
Let $\catC$ be an affine CD-category
and $\commacat{I}{\catC}$ be the comma (coslice) category.
Objects in
$\commacat{I}{\catC}$ are states in $\catC$, formally pairs
$(X,\sigma)$ of objects $X\in\catC$ and states $\sigma\colon I\to X$.
Arrows from $(X,\sigma)$ to $(Y,\tau)$ are \emph{state-preserving}
channels, namely
$c\colon X\to Y$ in $\catC$ satisfying $c\circ\sigma=\tau$.
A joint state $(X\otimes Y,\omega)\in\commacat{I}{\catC}$
is called a \emph{coupling} of two states
$(X,\sigma),(Y,\tau)\in\commacat{I}{\catC}$ if
\[\tikzextname{eq_coupling}
\begin{ctikzpicture}[font=\small]
\node[state] (omega) at (0,0) {\;$\omega$\;};
\coordinate (X) at (-0.25,0.5) {} {};
\node [discarder] (ground) at (0.25,0.2) {};
\draw (omega) ++(-0.25, 0) to (X);
\draw (omega) ++(0.25, 0) to (ground);
\node[scriptstyle] at (-0.45,0.4) {$X$};
\end{ctikzpicture}
\;\;=\;\;
\begin{ctikzpicture}[font=\small]
\node[state] (omega) at (0,0) {$\sigma$};
\node (X) at (0,0.5) {};
\draw (omega) to (X);
\node[scriptstyle] at (0.2,0.3) {$X$};
\end{ctikzpicture}
\qquad\text{and}\qquad
\begin{ctikzpicture}[font=\small]
\node[state] (omega) at (0,0) {\;$\omega$\;};
\coordinate (Y) at (0.25,0.5);
\node [discarder] (ground) at (-0.25,0.2) {};
\draw (omega) ++(-0.25, 0) to (ground);
\draw (omega) ++(0.25, 0) to (Y);
\node[scriptstyle] at (0.45,0.4) {$Y$};
\end{ctikzpicture}
\;\;=\;\;
\begin{ctikzpicture}[font=\small]
\node[state] (omega) at (0,0) {$\tau$};
\node (X) at (0,0.5) {};
\draw (omega) to (X);
\node[scriptstyle] at (0.2,0.3) {$Y$};
\end{ctikzpicture}
\quad.
\]
We write $\Coupl((X,\sigma),(Y,\tau))$
for the set of couplings of $(X,\sigma)$ and $(Y,\tau)$.

\begin{theorem}
Let $\catC$ be an affine CD-category that admits disintegration.
For each pair of states $(X,\sigma),(Y,\tau)\in\commacat{I}{\catC}$,
there is the following bijection:
\[
\commacat{I}{\catC}\paren[\big]{(X,\sigma),(Y,\tau)}/\sigma
\;\;
\cong
\;\;
\Coupl((X,\sigma),(Y,\tau))
\]
\end{theorem}
\begin{proof}
For each $c\in \commacat{I}{\catC}\paren[\big]{(X,\sigma),(Y,\tau)}$,
we define a joint state $I\to X\otimes Y$ to be
the `integration' of $\sigma$ and $c$ as below,
for which we use the following \emph{ad hoc} notation:
\[
\sigma\ogreaterthan c \coloneqq\;\;
\vcenter{\hbox{%
\begin{tikzpicture}[font=\small]
\node[state] (omega) at (0,0) {$\sigma$};
\node[copier] (copier) at (0,0.3) {};
\node[arrow box] (c) at (0.5,0.95) {$c$};
\coordinate (X) at (-0.5,1.5);
\coordinate (Y) at (0.5,1.5);
\draw (omega) to (copier);
\draw (copier) to[out=150,in=-90] (X);
\draw (copier) to[out=15,in=-90] (c);
\draw (c) to (Y);
\path[scriptstyle]
node at (-0.65,1.45) {$X$}
node at (0.65,1.45) {$Y$};
\end{tikzpicture}}}
\]
It is easy to check that $\sigma\ogreaterthan c$ is a coupling of $\sigma$ and
$\tau$.  For two channels $c,d\colon X\to Y$, we have
$\sigma\ogreaterthan c=\sigma\ogreaterthan d$ if and only if $c\equpto[\sigma] d$, by the
definition of $\equpto[\sigma]$.  This means the mapping
\[
c \longmapsto\sigma\ogreaterthan c
\,,\quad
\commacat{I}{\catC}\paren[\big]{(X,\sigma),(Y,\tau)}/\sigma
\longto
\Coupl((X,\sigma),(Y,\tau))
\]

\noindent is well-defined and injective.  To prove the surjectivity
let $(X\otimes Y,\omega)\in \Coupl((X,\sigma),(Y,\tau))$.  Let
$c\colon X\to Y$ be a disintegration of $\omega$.  Then $c$ is
state-preserving since
\[
\vcenter{\hbox{%
\begin{tikzpicture}[font=\small]
\node[state] (omega) at (0,0) {$\tau$};
\coordinate (Y) at ([yshiftu=0.3]omega);
\draw (omega) to (Y);
\end{tikzpicture}}}
\;\;=\;\;
\vcenter{\hbox{%
\begin{tikzpicture}[font=\small]
\node[state] (omega) at (0,0) {\;$\omega$\;};
\coordinate (omega1) at ([xshiftu=-0.25]omega);
\coordinate (omega2) at ([xshiftu=0.25]omega);
\coordinate (Y) at ([yshiftu=0.45]omega2);
\node[discarder] (d) at ([yshiftu=0.2]omega1) {};
\draw (omega2) to (Y);
\draw (omega1) to (d);
\end{tikzpicture}}}
\;\;=\;\;
\vcenter{\hbox{%
\begin{tikzpicture}[font=\small]
\node[state] (omega) at (0,0) {\;$\omega$\;};
\coordinate (omega1) at ([xshiftu=-0.25]omega);
\coordinate (omega2) at ([xshiftu=0.25]omega);
\node[copier] (copier) at ([yshiftu=0.45]omega1) {};
\node[arrow box] (c) at (0.25,1.05) {$c$};
\node[discarder] (X) at (-0.75,1.4) {};
\coordinate (Y) at ([yshiftu=0.2]c.north);
\node[discarder] (d) at ([yshiftu=0.2]omega2) {};
\draw (omega1) to (copier);
\draw (omega2) to (d);
\draw (copier) to[out=150,in=-90] (X);
\draw (copier) to[out=15,in=-90] (c);
\draw (c) to (Y);
\end{tikzpicture}}}
\;\;=\;\;
\vcenter{\hbox{%
\begin{tikzpicture}[font=\small]
\node[state] (omega) at (0,0) {\;$\omega$\;};
\coordinate (omega1) at ([xshiftu=-0.25]omega);
\coordinate (omega2) at ([xshiftu=0.25]omega);
\node[arrow box,anchor=south] (c) at ([yshiftu=0.4]omega1) {$c$};
\coordinate (X) at ([yshiftu=0.2]c.north);
\node[discarder] (d) at ([yshiftu=0.2]omega2) {};
\draw (omega1) to (c) to (X);
\draw (omega2) to (d);
\end{tikzpicture}}}
\;\;=\;\;
\vcenter{\hbox{%
\begin{tikzpicture}[font=\small]
\node[state] (omega) at (0,0) {$\sigma$};
\node[arrow box,anchor=south] (c) at ([yshiftu=0.2]omega) {$c$};
\coordinate (X) at ([yshiftu=0.2]c.north);
\draw (omega) to (c) to (X);
\end{tikzpicture}}}
\]
Moreover we have $\sigma\ogreaterthan c=\omega$, as desired.
\end{proof}

Via the symmetry $X\otimes Y\overset{\cong}{\to} Y\otimes X$
we have the obvious bijection
$\Coupl((X,\sigma),(Y,\tau))\cong\Coupl((Y,\tau),(X,\sigma))$.
This immediately gives the following corollary.

\begin{corollary}
Let $\catC$ be an affine CD-category that admits disintegration.
For any states $(X,\sigma),(Y,\tau)\in\commacat{I}{\catC}$
we have
\[
\commacat{I}{\catC}\paren[\big]{(X,\sigma),(Y,\tau)}/\sigma
\;\;
\cong
\;\;
\commacat{I}{\catC}\paren[\big]{(Y,\tau),(X,\sigma)}/\tau
\]
The bijection sends a channel $c\colon X\to Y$
to a Bayesian inversion $d\colon Y\to X$ for $\sigma$ along $c$.
\qed
\end{corollary}

Theorem~2 of~\cite{ClercDDG2017} is obtained as an instance, for the
category $\pKrnsb$. This bijective correspondence yields a `dagger'
$(-)^{\dag}$ functor on (a suitable quotient of) the comma category
$\commacat{I}{\catC}$ --- as noted by the authors
of~\cite{ClercDDG2017}.

\subsection*{Strong almost equality}

Unfortunately,
the almost equality defined above
is not the most useful notion
for equational reasoning
between string diagrams.
Indeed, later
in the proofs of
Proposition~\ref{prop:disint-identitiy}
and Theorem~\ref{thm:likelihood}
we encounter situations
where we need a stronger notion of
almost equality.
We define the stronger one as follows.

\begin{definition}
\label{def:str-almost-eq}
Let $c,d\colon X\to Y$ be channels
and $\sigma\colon I\to X$ be a state.
We say that $c$ is \emph{strongly $\sigma$-almost equal}
to $d$ if
\[
\begin{ctikzpicture}[font=\small]
\node[state] (omega) at (0,0) {\;$\omega$\;};
\node[arrow box] (c) at (-0.3,0.5) {$c$};
\coordinate (X) at (-0.3,1) {};
\coordinate (Y) at (0.3,1) {};
\draw (omega) ++(-.3,0) to (c) to (X);
\draw (omega) ++(.3,0) to (Y);
\end{ctikzpicture}
\;\;=\;\;
\begin{ctikzpicture}[font=\small]
\node[state] (omega) at (0,0) {\;$\omega$\;};
\node[arrow box] (c) at (-0.3,0.5) {$d$};
\coordinate (X) at (-0.3,1) {};
\coordinate (Y) at (0.3,1) {};
\draw (omega) ++(-.3,0) to (c) to (X);
\draw (omega) ++(.3,0) to (Y);
\end{ctikzpicture}
\]
holds for any $\omega\colon I\to X\otimes Z$ such
that $\sigma$ is the
first marginal of $\omega$, that is:
\[
\begin{ctikzpicture}[font=\small]
\node[state] (omega) at (0,0) {$\sigma$};
\node (X) at (0,0.5) {};
\draw (omega) to (X);
%
\end{ctikzpicture}
\;\;=\;\;
\begin{ctikzpicture}[font=\small]
\node[state] (omega) at (0,0) {\;$\omega$\;};
\coordinate (X) at (-0.25,0.5) {} {};
\node [discarder] (ground) at (0.25,0.2) {};
\draw (omega) ++(-0.25, 0) to (X);
\draw (omega) ++(0.25, 0) to (ground);
%
\end{ctikzpicture}
\;\;.
\]
\end{definition}

Notice that strong almost quality implies
almost quality,
via the following $\omega$:
\[
\begin{ctikzpicture}[font=\small]
\node[state] (omega) at (0,0) {\;$\omega$\;};
\coordinate (X) at (-0.25,0.4);
\coordinate (Y) at (0.25,0.4);
\draw (omega) ++(-0.25, 0) to (X);
\draw (omega) ++(0.25, 0) to (Y);
%
\end{ctikzpicture}
\;\;\coloneqq\;\;
\begin{ctikzpicture}[font=\small]
\node[state] (omega) at (0,0) {$\sigma$};
\node[copier] (copier) at (0,0.2) {};
\coordinate (X) at (-0.3,0.5);
\coordinate (Y) at (0.3,0.5);
\draw (omega) to (copier);
\draw (copier) to[out=165,in=-90] (X);
\draw (copier) to[out=15,in=-90] (Y);
\end{ctikzpicture}
\;\;.
\]

\noindent
The good news is that
the converse often holds too.
We say that an affine CD-category
admits \emph{equality strengthening}
if $c$ is strongly $\sigma$-almost equal
to $d$ whenever
$c\equpto[\sigma]d$,
i.e.\ $c$ is $\sigma$-almost equal to $d$.
We present two propositions that guarantee this property.

\begin{proposition}
\label{prop:disint-eq-str}
If an affine CD-category admits disintegration,
then it admits equality strengthening.
\end{proposition}
\begin{proof}
Suppose that
$c\equpto[\sigma]d$
holds for a state $\sigma\colon I\to X$
and for channels $c,d\colon X\to Y$.
Let $\omega\colon I\to X\otimes Z$
be a state whose first marginal is $\sigma$.
We can disintegrate $\omega$ as in:
\[
\vcenter{\hbox{%
\begin{tikzpicture}[font=\small]
\node[state] (omega) at (0,0) {\;$\omega$\;};
\coordinate (X) at (-0.3,.5);
\coordinate (Y) at (0.3,.5);
\draw (omega) ++(-.3,0) to (X);
\draw (omega) ++(.3,0) to (Y);
\end{tikzpicture}}}
\;\;=\;\;
\vcenter{\hbox{%
\begin{tikzpicture}[font=\small]
\node[state] (omega) at (0,0) {$\sigma$};
\node[copier] (copier) at (0,0.3) {};
\node[arrow box] (c) at (0.5,0.95) {$e$};
\coordinate (X) at (-0.5,1.5);
\coordinate (Y) at (0.5,1.5);
\draw (omega) to (copier);
\draw (copier) to[out=150,in=-90] (X);
\draw (copier) to[out=15,in=-90] (c);
\draw (c) to (Y);
\end{tikzpicture}}}
\;\;\;\;.
\]
Then we have
\[
\vcenter{\hbox{%
\begin{tikzpicture}[font=\small]
\node[state] (omega) at (0,0) {\;$\omega$\;};
\node[arrow box] (c) at (-0.3,0.5) {$c$};
\coordinate (X) at (-0.3,1) {};
\coordinate (Y) at (0.3,1) {};
\draw (omega) ++(-.3,0) to (c) to (X);
\draw (omega) ++(.3,0) to (Y);
\end{tikzpicture}}}
\;\;=\;\;
\vcenter{\hbox{%
\begin{tikzpicture}[font=\small]
\node[state] (omega) at (0,0) {$\sigma$};
\node[copier] (copier) at (0,0.3) {};
\node[arrow box] (c) at (-0.5,0.95) {$c$};
\node[arrow box] (e) at (0.5,0.95) {$e$};
\coordinate (X) at (-0.5,1.5);
\coordinate (Y) at (0.5,1.5);
\draw (omega) to (copier);
\draw (copier) to[out=165,in=-90] (c);
\draw (c) to (X);
\draw (copier) to[out=15,in=-90] (e);
\draw (e) to (Y);
\end{tikzpicture}}}
\;\;=\;\;
\vcenter{\hbox{%
\begin{tikzpicture}[font=\small]
\node[state] (omega) at (0,0) {$\sigma$};
\node[copier] (copier) at (0,0.3) {};
\node[arrow box] (c) at (-0.5,0.95) {$d$};
\node[arrow box] (e) at (0.5,0.95) {$e$};
\coordinate (X) at (-0.5,1.5);
\coordinate (Y) at (0.5,1.5);
\draw (omega) to (copier);
\draw (copier) to[out=165,in=-90] (c);
\draw (c) to (X);
\draw (copier) to[out=15,in=-90] (e);
\draw (e) to (Y);
\end{tikzpicture}}}
\;\;=\;\;
\vcenter{\hbox{%
\begin{tikzpicture}[font=\small]
\node[state] (omega) at (0,0) {\;$\omega$\;};
\node[arrow box] (c) at (-0.3,0.5) {$d$};
\coordinate (X) at (-0.3,1) {};
\coordinate (Y) at (0.3,1) {};
\draw (omega) ++(-.3,0) to (c) to (X);
\draw (omega) ++(.3,0) to (Y);
\end{tikzpicture}}}
\;\;.
\]
Therefore $c$ is strongly $\sigma$-almost equal to $d$.
\end{proof}

Recall that $\Kl(\Giry)$ does not admit disintegration.
Nevertheless, it admits equality strengthening.

\begin{proposition}
\label{prop:KlG-eq-str}
The category $\Kl(\Giry)$ admits equality strengthening.
\end{proposition}
\begin{proof}
Assume $c\equpto[\sigma]d$
for $\sigma\in\Giry(X)$
and $c,d\colon X\to \Giry(Y)$.
Let $\omega\in\Giry(X\otimes Z)$
be a probability measure whose
first marginal is $\sigma$,
i.e.\ $(\pi_1)_*(\omega)=\sigma$.
We need to prove
$(c\otimes \eta_Z)\klcirc \omega=(d\otimes \eta_Z)\klcirc \omega$,
which is equivalent to:
\begin{equation}
\label{eq:desired-eq1}
\int_{X\times Z}
c(x)(B)\indic{C}(z)
\,\omega(\dd(x,z))
=
\int_{X\times Z}
d(x)(B)\indic{C}(z)
\,\omega(\dd(x,z))
\end{equation}
for all $B\in\Sigma_Y$ and $C\in\Sigma_Z$.
Using
\[
\abs[\big]{c(x)(B)\indic{C}(z)
-d(x)(B)\indic{C}(z)}
=
\abs[\big]{c(x)(B) -d(x)(B)}\indic{C}(z)
\le
\abs[\big]{c(x)(B) -d(x)(B)}
\enspace,
\]
we have
\begin{align*}
&
\abs[\Big]{\int_{X\times Z}
\paren[\big]{c(x)(B)\indic{C}(z)
-d(x)(B)\indic{C}(z)}
\,\omega(\dd(x,z))}
\\
&\le
\int_{X\times Z}
\abs[\big]{c(x)(B)\indic{C}(z)
-d(x)(B)\indic{C}(z)}
\,\omega(\dd(x,z))
\\
&\le
\int_{X\times Z}
\abs[\big]{c(x)(B) -d(x)(B)}
\,\omega(\dd(x,z))
\\
&=
\int_{X}
\abs[\big]{c(x)(B) -d(x)(B)}
\,(\pi_1)_*(\omega)(\dd x)
\\
&=
\int_{X}
\abs[\big]{c(x)(B) -d(x)(B)}
\,\sigma(\dd x)
\\
&=
0\enspace,
\end{align*}
where the last equality holds
since $c(-)(B)=d(-)(B)$
$\sigma$-almost everywhere
by Proposition~\ref{prop:equality-up-to-giry}.
Therefore the desired equality~\eqref{eq:desired-eq1} holds.
\end{proof}


\section{Conditional independence}\label{sec:conditionalindependence}

Throughout this section, we consider an affine CD-category
that admits disintegration.

\subsection{Disintegration of multipartite states}
\tikzextname{ss_disint}

So far we have concentrated on bipartite states --- except in the
classification example in Section~\ref{sec:classifier}. In order to
deal with a general $n$-partite state $\omega\colon I\to
X_1\otimes\dotsb\otimes X_n$, we will introduce several notations and
conventions in Definitions~\ref{def:notation-marginal},
\ref{def:notation-disintegration}
and~\ref{def:convention-almost-equal} below; they are in line with
standard practice in probability theory.

In the conventions, an $n$-partite state, as below, is fixed, and used
implicitly.
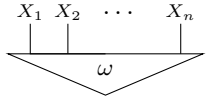
\[
\begin{tikzpicture}[font=\small]
\node[state] (omega) at (0,0) {\hspace*{1em}$\omega$\hspace*{1em}};
\draw (omega) -- (-1.0, 0) to (-1.0,0.4);
\draw (omega) ++(-0.5, 0) to (-0.5,0.4);
\draw (omega) ++(1.0, 0) to (1.0,0.4);
\node[font=\normalsize] at (0.2,0.55) {$\dotso$};
\path[font=\normalsize,
execute at begin node=\everymath{\scriptstyle}]
node at (-1,0.55) {$X_1$}
node at (-0.5,0.55) {$X_2$}
node at (1,0.55) {$X_n$};
\end{tikzpicture}
\]

\begin{definition}
\label{def:notation-marginal}
When we write
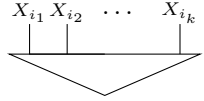
\[
\begin{ctikzpicture}[font=\small]
\node[state] (omega) at (0,0) {\hspace*{2.6em}};
\draw (omega) -- (-1.0, 0) to (-1.0,0.4);
\draw (omega) ++(-0.5, 0) to (-0.5,0.4);
\draw (omega) ++(1.0, 0) to (1.0,0.4);
\node[font=\normalsize] at (0.2,0.55) {$\dotso$};
\path[font=\normalsize,
execute at begin node=\everymath{\scriptstyle}]
node at (-1,0.55) {$X_{i_1}$}
node at (-0.5,0.55) {$X_{i_2}$}
node at (1,0.55) {$X_{i_k}$};
\end{ctikzpicture}
\]
where $i_1,\dotsc,i_k$ are distinct,
it denotes the state $I\to X_{i_1}\otimes\dotsb \otimes X_{i_k}$
obtained from $\omega$ by marginalisation and permutation of wires (if necessary).
Let us give a couple of examples, for $n=5$.
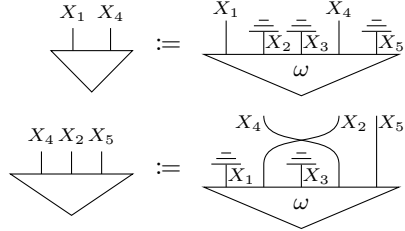
\begin{align*}
\begin{ctikzpicture}[font=\small]
\node[state] (omega) at (0,0) {\;\;\;\;};
\draw (omega) ++(-0.25, 0) to +(0,0.3);
\draw (omega) ++(0.25, 0) to +(0,0.3);
\path[font=\normalsize,
execute at begin node=\everymath{\scriptstyle}]
node at (-0.25,0.5) {$X_1$}
node at (0.25,0.5) {$X_4$};
\end{ctikzpicture}
&\;\;\coloneqq\;\;
\begin{ctikzpicture}[font=\small]
\node[state] (omega) at (0,0) {\hspace*{1em}$\omega$\hspace*{1em}};
\node[discarder] (d1) at (-0.5,0.3) {};
\node[discarder] (d2) at (0,0.3) {};
\node[discarder] (d3) at (1,0.3) {};
\draw (omega) ++(-1.0, 0) to (-1,0.45);
\draw (omega) ++(-0.5, 0) to (d1);
\draw (omega) ++(0, 0) to (d2);
\draw (omega) ++(0.5, 0) to (0.5,0.45);
\draw (omega) ++(1.0, 0) to (d3);
\path[font=\normalsize,
execute at begin node=\everymath{\scriptstyle}]
node at (-1,0.6) {$X_1$}
node at (-0.3,0.15) {$X_2$}
node at (0.2,0.15) {$X_3$}
node at (0.5,0.6) {$X_4$}
node at (1.2,0.15) {$X_5$};
\end{ctikzpicture}
\\
\begin{ctikzpicture}[font=\small]
\node[state] (omega) at (0,0) {\;\;\;\;\;\;};
\draw (omega) ++(-0.4, 0) to +(0,0.3);
\draw (omega) to +(0,0.3);
\draw (omega) ++(0.4, 0) to +(0,0.3);
\path[font=\normalsize,
execute at begin node=\everymath{\scriptstyle}]
node at (-0.4,0.5) {$X_4$}
node at (0,0.5) {$X_2$}
node at (0.4,0.5) {$X_5$};
\end{ctikzpicture}
&\;\;\coloneqq\;\;
\begin{ctikzpicture}[font=\small]
\node[state] (omega) at (0,0) {\hspace*{1em}$\omega$\hspace*{1em}};
\node[discarder] (d1) at (-1,0.3) {};
\node[discarder] (d2) at (0,0.3) {};
\draw (omega) ++(-1.0, 0) to (d1);
\draw (omega) ++(-0.5, 0) to (-0.5,0.3)
to[out=90,in=-90] (0.5,0.95);
\draw (omega) ++(0, 0) to (d2);
\draw (omega) ++(0.5, 0) to (0.5,0.3)
to[out=90,in=-90] (-0.5,0.95);
\draw (omega) ++(1.0, 0) to (1,0.95);
\path[font=\normalsize,
execute at begin node=\everymath{\scriptstyle}]
node at (-0.8,0.15) {$X_1$}
node at (0.7,0.85) {$X_2$}
node at (0.2,0.15) {$X_3$}
node at (-0.7,0.85) {$X_4$}
node at (1.2,0.85) {$X_5$};
\end{ctikzpicture}
\end{align*}
We permute wires via a combination of crossing.
This is unambiguous by the coherence theorem.
\end{definition}

Below we will use symbols $X,Y,Z,W,\dotsc$ to denote not only a single
wire $X_i$ but also multiple wires $X_i\otimes X_j\otimes\dotsb$.
Disintegrations more general than in the bipartite case are now
introduced as follows.

\begin{definition}
\label{def:notation-disintegration}
For $X=X_{i_1}\otimes \dotsb \otimes X_{i_k}$
and $Y=X_{j_1}\otimes \dotsb \otimes X_{j_l}$,
with all $i_1,\dotsc,i_k,j_1,\dotsc,j_l$ distinct,
a disintegration $X\to Y$ is defined to be a disintegration
of
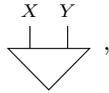
\[
\begin{ctikzpicture}[font=\small]
\node[state] (omega) at (0,0) {\;\;\;\;};
\draw (omega) ++(-0.25, 0) to +(0,0.3);
\draw (omega) ++(0.25, 0) to +(0,0.3);
\path[font=\normalsize,
execute at begin node=\everymath{\scriptstyle}]
node at (-0.25,0.5) {$X$}
node at (0.25,0.5) {$Y$};
\end{ctikzpicture}
\enspace,
\]
the marginal state given by the previous convention.
We denote the disintegration simply as on the left below,
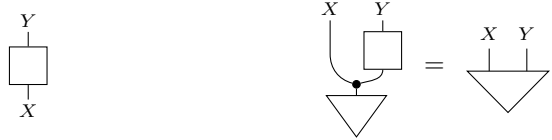
\[
\begin{ctikzpicture}[font=\small]
\node[arrow box] (a) at (0,0.5) {};
\draw (0,0.05) to (a);
\draw (a) to (0,0.95);
\path[font=\normalsize,
execute at begin node=\everymath{\scriptstyle}]
node at (0,1.1) {$Y$}
node at (0,-0.1) {$X$};
\end{ctikzpicture}
\qquad\qquad\qquad\qquad\qquad
\begin{ctikzpicture}[font=\small]
\node[state] (omega) at (0,0) {};
\node[copier] (copier) at (0,0.15) {};
\node[arrow box] (c) at (0.35,0.6) {};
\coordinate (X) at (-0.35,1) {} {};
\coordinate (Y) at (0.35,1) {} {};
\draw (omega) to (copier);
\draw (copier) to[out=165,in=-90]
(-0.35,0.55) to (X);
\draw (copier) to[out=15,in=-90] (c);
\draw (c) to (Y);
\path[font=\normalsize,
execute at begin node=\everymath{\scriptstyle}]
node at (0.35,1.15) {$Y$}
node at (-0.35,1.15) {$X$};
\end{ctikzpicture}
\;\;=\;\;
\begin{ctikzpicture}[font=\small]
\node[state] (omega) at (0,0) {\;\;\;\;};
\draw (omega) ++(-0.25, 0) to +(0,0.3);
\draw (omega) ++(0.25, 0) to +(0,0.3);
\path[font=\normalsize,
execute at begin node=\everymath{\scriptstyle}]
node at (-0.25,0.5) {$X$}
node at (0.25,0.5) {$Y$};
\end{ctikzpicture}
\]
By definition, it must satisfy the equation on the right above.
Let us give an example.
The disintegration $X_5\otimes X_2\to X_1\otimes X_4$
on the left below is defined by the equation on the right.
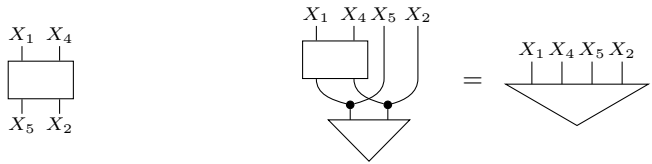
\[
\begin{ctikzpicture}[font=\small]
\node[arrow box] (a) at (0,0.5) {\;\;\;\;\;\;\;};
\draw (-0.25,0.05) to ([xshiftu=-0.25]a.south);
\draw (0.25,0.05) to ([xshiftu=0.25]a.south);
\draw ([xshiftu=-0.25]a.north) to (-0.25,0.95);
\draw ([xshiftu=0.25]a.north) to (0.25,0.95);
\path[font=\normalsize,
execute at begin node=\everymath{\scriptstyle}]
node at (-0.25,1.1) {$X_1$}
node at (-0.25,-0.1) {$X_5$}
node at (0.25,-0.1) {$X_2$}
node at (0.25,1.1) {$X_4$};
\end{ctikzpicture}
\qquad\qquad\qquad\qquad
\begin{ctikzpicture}[font=\small]
\node[arrow box] (a) at (-0.45,0.8) {\;\;\;\;\;\;\;};
\node[copier] (c1) at (-0.25,0.2) {};
\node[copier] (c2) at (0.25,0.2) {};
\node[state] (s) at (0,0) {\;\;\;\;};
\draw (s) ++(-0.25,0) to (c1);
\draw (s) ++(0.25,0) to (c2);
\draw (c1) to[out=165,in=-90] ([xshiftu=-0.25]a.south);
\draw (c2) to[out=165,in=-90] ([xshiftu=0.25]a.south);
\draw ([xshiftu=-0.25]a.north) to (-0.7,1.25);
\draw ([xshiftu=0.25]a.north) to (-0.2,1.25);
\draw (c1) to[out=15,in=-90] (0.2,0.55) to (0.2,1.25);
\draw (c2) to[out=15,in=-90] (0.65,0.55) to (0.65,1.25);
\path[font=\normalsize,
execute at begin node=\everymath{\scriptstyle}]
node at (-0.7,1.4) {$X_1$}
node at (0.2,1.4) {$X_5$}
node at (0.65,1.4) {$X_2$}
node at (-0.2,1.4) {$X_4$};
\end{ctikzpicture}
\;\;=\;\;
\begin{ctikzpicture}[font=\small]
\node[state] (omega) at (0,0) {\;\;\;\;\;\;\;};
\draw (omega) ++(-0.6, 0) to +(0,0.3);
\draw (omega) ++(-0.2, 0) to +(0,0.3);
\draw (omega) ++(0.2, 0) to +(0,0.3);
\draw (omega) ++(0.6, 0) to +(0,0.3);
\path[font=\normalsize,
execute at begin node=\everymath{\scriptstyle}]
node at (-0.6,0.45) {$X_1$}
node at (-0.2,0.45) {$X_4$}
node at (0.2,0.45) {$X_5$}
node at (0.6,0.45) {$X_2$};
\end{ctikzpicture}
\]
More specifically, assuming $n=5$ and
expanding the notation for marginals,
the equation is:
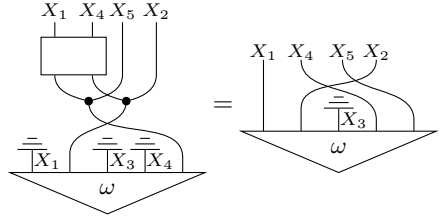
\[
\begin{ctikzpicture}[font=\small]
\node[arrow box] (a) at (-0.45,0.8) {\;\;\;\;\;\;\;};
\node[copier] (c1) at (-0.25,0.2) {};
\node[copier] (c2) at (0.25,0.2) {};
\node[state] (omega) at (0,-0.75) {\hspace*{1em}$\omega$\hspace*{1em}};
\node[discarder] (d1) at (-1,-0.45) {};
\node[discarder] (d3) at (0,-0.45) {};
\node[discarder] (d4) at (0.5,-0.45) {};
\draw (c1) to[out=165,in=-90] ([xshiftu=-0.25]a.south);
\draw (c2) to[out=165,in=-90] ([xshiftu=0.25]a.south);
\draw ([xshiftu=-0.25]a.north) to (-0.7,1.25);
\draw ([xshiftu=0.25]a.north) to (-0.2,1.25);
\draw (c1) to[out=15,in=-90] (0.2,0.55) to (0.2,1.25);
\draw (c2) to[out=15,in=-90] (0.65,0.55) to (0.65,1.25);
\draw (omega) ++(-1,0) to (d1);
\draw (omega) ++(-0.5,0) to (-0.5,-0.5)
to[out=90,in=-90] (0.25,0.2);
\draw (omega) ++(0,0) to (d3);
\draw (omega) ++(0.5,0) to (d4);
\draw (omega) ++(1,0) to (1,-0.55) to[out=90,in=-90] (-0.25,0.2);
\path[font=\normalsize,
execute at begin node=\everymath{\scriptstyle}]
node at (-0.7,1.4) {$X_1$}
node at (0.65,1.4) {$X_2$}
node at (-0.8,-0.6) {$X_1$}
node at (0.2,-0.6) {$X_3$}
node at (0.7,-0.6) {$X_4$}
node at (-0.2,1.4) {$X_4$}
node at (0.2,1.4) {$X_5$};
\end{ctikzpicture}
=
\begin{ctikzpicture}[font=\small]
\node[state] (omega) at (0,0) {\hspace*{1em}$\omega$\hspace*{1em}};
\node[discarder] (d2) at (0,0.3) {};
\draw (omega) ++(-1.0, 0) to (-1,0.95);
\draw (omega) ++(-0.5, 0) to (-0.5,0.3)
to[out=90,in=-90] (0.5,0.95);
\draw (omega) ++(0, 0) to (d2);
\draw (omega) ++(0.5, 0) to (0.5,0.15)
to[out=90,in=-90] (-0.5,0.95);
\draw (omega) ++(1.0, 0) to (1,0.15) to[out=90,in=-90] (0.05,0.95);
\path[font=\normalsize,
execute at begin node=\everymath{\scriptstyle}]
node at (-1,1.05) {$X_1$}
node at (0.5,1.05) {$X_2$}
node at (0.2,0.15) {$X_3$}
node at (-0.5,1.05) {$X_4$}
node at (0.05,1.05) {$X_5$};
\end{ctikzpicture}
\]
Note that disintegrations need not be unique.
Thus when we write
$\begin{ctikzpicture}[font=\small]
\node[arrow box] (a) at (0,0.5) {};
\draw (0,0.05) to (a);
\draw (a) to (0,0.95);
\path[font=\normalsize,
execute at begin node=\everymath{\scriptstyle}]
node at (0.15,0.9) {$Y$}
node at (0.15,0.1) {$X$};
\end{ctikzpicture}$,
we in fact \emph{choose} one of them.
Nevertheless, such disintegrations are unique up to almost-equality
with respect to $\begin{ctikzpicture}[font=\small]
\node[state] (a) at (0,0.5) {};
\draw (a) to (0,0.75);
\path[font=\normalsize,
execute at begin node=\everymath{\scriptstyle}]
node at (0.15,0.65) {$X$};
\end{ctikzpicture}$,
which is good enough for our purpose.
\end{definition}

Finally we make a convention about almost equality
(Definition~\ref{def:almost-equal}).
\begin{definition}
\label{def:convention-almost-equal}
Let $S$ and $T$ be string diagrams of type $X\to Y$
that are made from marginals and disintegrations of $\omega$
as defined in Definitions~\ref{def:notation-marginal} and \ref{def:notation-disintegration}.
When we say $S$ is almost equal to $T$ (or write $S\equpto T$) without reference to
a state, it means that
$S$ is almost equal to $T$ with respect to the state $\begin{ctikzpicture}[font=\small]
\node[state] (a) at (0,0.5) {};
\draw (a) to (0,0.75);
\path[font=\normalsize,
execute at begin node=\everymath{\scriptstyle}]
node at (0.15,0.65) {$X$};
\end{ctikzpicture}$.
\end{definition}

We shall make use of the following auxiliary equations involving
discarding and composition of disintegrations.

\begin{proposition}
\label{prop:disint-identitiy}
\tikzextname{prop_disint_identitiy}
In the conventions and notations above,
the following hold.
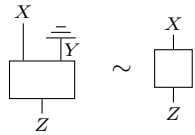
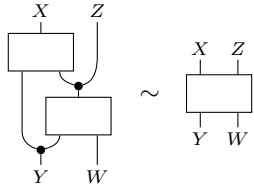
\begin{enumerate}
\setlength{\abovedisplayskip}{-.5\baselineskip}
\setlength{\abovedisplayshortskip}{-.5\baselineskip}
\setlength{\belowdisplayskip}{0pt}
\setlength{\belowdisplayshortskip}{0pt}
\item\label{propen:disint-marg}
\[
\begin{ctikzpicture}[font=\small]
\node[arrow box] (a) at (0,0.5) {\;\;\;\;\;\;\;};
\node[discarder] (d) at (0.25,1.05) {};
\draw (0,0.05) to (a);
\draw ([xshiftu=-0.25]a.north) to (-0.25,1.25);
\draw ([xshiftu=0.25]a.north) to (d);
\path[font=\normalsize,
execute at begin node=\everymath{\scriptstyle}]
node at (-0.25,1.4) {$X$}
node at (0.4,0.9) {$Y$}
node at (0,-0.1) {$Z$};
\end{ctikzpicture}
\;\;\equpto\;\;
\begin{ctikzpicture}[font=\small]
\node[arrow box] (a) at (0,0.5) {};
\draw (0,0.05) to (a);
\draw (a) to (0,0.95);
\path[font=\normalsize,
execute at begin node=\everymath{\scriptstyle}]
node at (0,1.1) {$X$}
node at (0,-0.1) {$Z$};
\end{ctikzpicture}
\]
\item\label{propen:comp-disintegration}
\[
\begin{ctikzpicture}[font=\small]
\node[arrow box] (a) at (0,0.5) {\;\;\;\;\;\;\;};
\node[arrow box] (b) at (0.5,-0.35) {\;\;\;\;\;\;\;};
\node[copier] (c) at (0,-0.8) {};
\node[copier] (c2) at (0.5,0.05) {};
\draw (a) to (0,0.9);
\draw (c) to[out=165,in=-90] (-0.25,-0.55) to ([xshiftu=-0.25]a.south);
\draw (c) to[out=15,in=-90] ([xshiftu=-0.25]b.south);
\draw (b) to (c2);
\draw (c2) to[out=165,in=-90] ([xshiftu=0.25]a.south);
\draw (c2) to[out=15,in=-90] (0.75,0.45) to (0.75,0.9);
\draw (0,-1) to (c);
\draw (0.75,-1) to ([xshiftu=0.25]b.south);
\path[font=\normalsize,
execute at begin node=\everymath{\scriptstyle}]
node at (0,1.05) {$X$}
node at (0,-1.15) {$Y$}
node at (0.75,-1.15) {$W$}
node at (0.75,1.05) {$Z$};
\end{ctikzpicture}
\;\;\equpto\;\;
\begin{ctikzpicture}[font=\small]
\node[arrow box] (a) at (0,0.5) {\;\;\;\;\;\;\;};
\draw (-0.25,0.05) to ([xshiftu=-0.25]a.south);
\draw (0.25,0.05) to ([xshiftu=0.25]a.south);
\draw ([xshiftu=-0.25]a.north) to (-0.25,0.95);
\draw ([xshiftu=0.25]a.north) to (0.25,0.95);
\path[font=\normalsize,
execute at begin node=\everymath{\scriptstyle}]
node at (-0.25,1.1) {$X$}
node at (-0.25,-0.1) {$Y$}
node at (0.25,-0.1) {$W$}
node at (0.25,1.1) {$Z$};
\end{ctikzpicture}
\]
\end{enumerate}
\end{proposition}
\begin{proof}
By the definition of almost equality,
\ref{propen:disint-marg} is proved by:
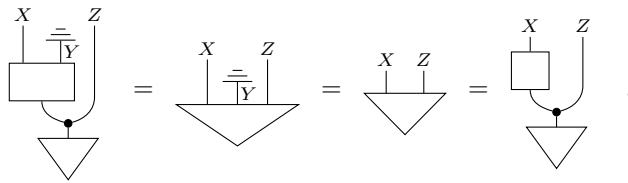
\[
\begin{ctikzpicture}[font=\small]
\node[arrow box] (a) at (0,0.5) {\;\;\;\;\;\;\;};
\node[discarder] (d) at (0.25,1.05) {};
\node[state] (s) at (0.35,-0.25) {};
\node[copier] (c) at (0.35,-0.05) {};
\draw ([xshiftu=-0.25]a.north) to (-0.25,1.25);
\draw ([xshiftu=0.25]a.north) to (d);
\draw (s) to (c);
\draw (c) to[out=165,in=-90] (a);
\draw (c) to[out=15,in=-90] (0.7,0.3) to (0.7,1.25);
\path[scriptstyle]
node at (-0.25,1.4) {$X$}
node at (0.4,0.9) {$Y$}
node at (0.7,1.4) {$Z$};
\end{ctikzpicture}
\;\;=\;\;
\begin{ctikzpicture}[font=\small]
\node[state] (omega) at (0,0) {\;\;\;\;\;\;};
\node[discarder] (d) at (0,0.3) {};
\draw (omega) ++(-0.4, 0) to +(0,0.6);
\draw (omega) to (d);
\draw (omega) ++(0.4, 0) to +(0,0.6);
\path[scriptstyle]
node at (-0.4,0.75) {$X$}
node at (0.15,0.15) {$Y$}
node at (0.4,0.75) {$Z$};
\end{ctikzpicture}
\;\;=\;\;
\begin{ctikzpicture}[font=\small]
\node[state] (omega) at (0,0) {\;\;\;\;};
\draw (omega) ++(-0.25, 0) to +(0,0.3);
\draw (omega) ++(0.25, 0) to +(0,0.3);
\path[scriptstyle]
node at (-0.25,0.45) {$X$}
node at (0.25,0.45) {$Z$};
\end{ctikzpicture}
\;\;=\;\;
\begin{ctikzpicture}[font=\small]
\node[arrow box] (a) at (0,0.5) {};
\node[state] (s) at (0.35,-0.25) {};
\node[copier] (c) at (0.35,-0.05) {};
\draw (a) to (0,0.95);
\draw (s) to (c);
\draw (c) to[out=165,in=-90] (a);
\draw (c) to[out=15,in=-90] (0.7,0.3) to (0.7,0.95);
\path[scriptstyle]
node at (0,1.1) {$X$}
node at (0.7,1.1) {$Z$};
\end{ctikzpicture}
\quad.
\]
Similarly, we prove~\ref{propen:comp-disintegration} as follows.
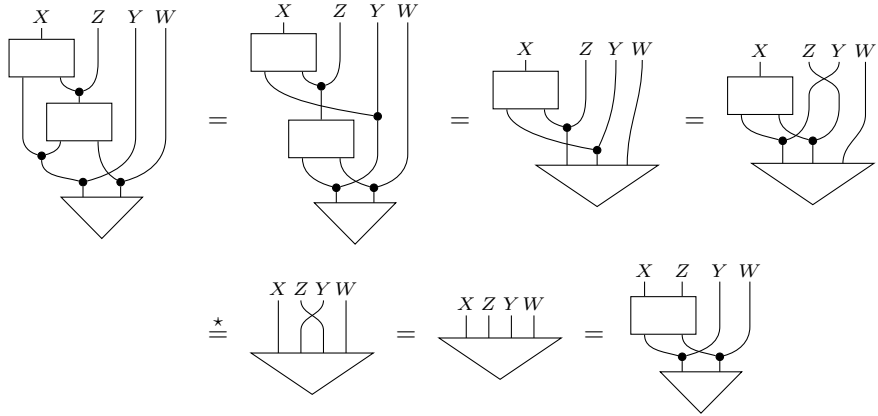
\begin{align*}
\begin{ctikzpicture}[font=\small]
\node[arrow box] (a) at (0,0.5) {\;\;\;\;\;\;\;};
\node[arrow box] (b) at (0.5,-0.35) {\;\;\;\;\;\;\;};
\node[copier] (c) at (0,-0.8) {};
\node[copier] (c2) at (0.5,0.05) {};
\node[copier] (c3) at (0.55,-1.15) {};
\node[copier] (c4) at (1.05,-1.15) {};
\node[state] (s) at (0.8,-1.35) {\;\;\;\;};
\draw (a) to (0,0.9);
\draw (c) to[out=165,in=-90] (-0.25,-0.55) to ([xshiftu=-0.25]a.south);
\draw (c) to[out=15,in=-90] ([xshiftu=-0.25]b.south);
\draw (b) to (c2);
\draw (c2) to[out=165,in=-90] ([xshiftu=0.25]a.south);
\draw (c2) to[out=15,in=-90] (0.75,0.45) to (0.75,0.9);
\draw (c3) to[out=165,in=-90] (c);
\draw (c4) to[out=150,in=-90] ([xshiftu=0.25]b.south);
\draw (c3) to[out=15,in=-90] (1.25,-0.55) to (1.25,0.9);
\draw (c4) to[out=15,in=-90] (1.65,-0.55) to (1.65,0.9);
\draw (s) ++(-0.25, 0) to (c3);
\draw (s) ++(0.25, 0) to (c4);
\path[scriptstyle]
node at (0,1.05) {$X$}
node at (1.25,1.05) {$Y$}
node at (1.65,1.05) {$W$}
node at (0.75,1.05) {$Z$};
\end{ctikzpicture}
&\;\;=\;\;
\begin{ctikzpicture}[font=\small]
\node[arrow box] (a) at (0,0.8) {\;\;\;\;\;\;\;};
\node[arrow box] (b) at (0.5,-0.35) {\;\;\;\;\;\;\;};
\node[copier] (c) at (1.25,-0.05) {};
\node[copier] (c2) at (0.5,0.35) {};
\node[copier] (c3) at (0.7,-1) {};
\node[copier] (c4) at (1.2,-1) {};
\node[state] (s) at (0.95,-1.2) {\;\;\;\;};
\draw (a) to (0,1.2);
\draw (b) to (c2);
\draw (c2) to[out=165,in=-90] ([xshiftu=0.25]a.south);
\draw (c2) to[out=15,in=-90] (0.75,0.75) to (0.75,1.2);
\draw (c3) to[out=165,in=-90] ([xshiftu=-0.25]b.south);
\draw (c4) to[out=165,in=-90] ([xshiftu=0.25]b.south);
\draw (c3) to[out=15,in=-90] (1.25,-0.55) to (c) to (1.25,1.2);
\draw (c) .. controls (0.3,0.1) and (-0.25,0.25) .. ([xshiftu=-0.25]a.south);
\draw (c4) to[out=15,in=-90] (1.65,-0.55) to (1.65,1.2);
\draw (s) ++(-0.25,0) to (c3);
\draw (s) ++(0.25,0) to (c4);
\path[scriptstyle]
node at (0,1.35) {$X$}
node at (1.25,1.35) {$Y$}
node at (1.65,1.35) {$W$}
node at (0.75,1.35) {$Z$};
\end{ctikzpicture}
\;\;=\;\;
\begin{ctikzpicture}[font=\small]
\node[arrow box] (a) at (0.2,0.75) {\;\;\;\;\;\;\;};
\node[copier] (c) at (1.15,-0.05) {};
\node[copier] (c2) at (0.75,0.25) {};
\node[state] (s) at (1.15,-0.25) {\;\;\;\;\;\;};
\draw (s) ++(-0.4,0) to (c2);
\draw (s) to (c) to[out=65,in=-90] (1.4,1.15);
\draw (s) ++(0.4,0) to[out=90,in=-90] (1.75,1.15);
\draw (a) to (0.2,1.15);
\draw (c2) to[out=165,in=-90] ([xshiftu=0.25]a.south);
\draw (c2) to[out=15,in=-90] (1,0.7) to (1,1.15);
\draw (c) .. controls (0.3,0.1) and (-0.05,0.3) .. ([xshiftu=-0.25]a.south);
\path[font=\normalsize,
execute at begin node=\everymath{\scriptstyle}]
node at (0.2,1.3) {$X$}
node at (1.4,1.3) {$Y$}
node at (1.75,1.3) {$W$}
node at (1,1.3) {$Z$};
\end{ctikzpicture}
\;\;=\;\;
\begin{ctikzpicture}[font=\small]
\node[arrow box] (a) at (0.2,0.7) {\;\;\;\;\;\;\;};
\node[copier] (c) at (0.9,0.1) {};
\node[copier] (c2) at (0.5,0.1) {};
\node[state] (s) at (0.9,-0.2) {\;\;\;\;\;\;};
\draw (s) ++(-0.4,0) to (c2);
\draw (s) to (c);
\draw (s) ++(0.4,0) to[out=90,in=-90] (1.6,0.45) to (1.6,1.15);
\draw (a) to (0.2,1.15);
\draw (c) to[out=165,in=-90] ([xshiftu=0.25]a.south);
\draw (c2) to[out=165,in=-90] ([xshiftu=-0.25]a.south);
\draw (c) to[out=15,in=-90] (1.25,0.55) to (1.25,0.7) to[out=90,in=-90] (0.85,1.15);
\draw (c2) to[out=15,in=-90] (0.85,0.55) to (0.85,0.7) to[out=90,in=-90] (1.25,1.15);
\path[scriptstyle]
node at (0.2,1.3) {$X$}
node at (1.25,1.3) {$Y$}
node at (1.6,1.3) {$W$}
node at (0.85,1.3) {$Z$};
\end{ctikzpicture}
\\
&\;\;\overset{\star}{=}\;\;
\begin{ctikzpicture}[font=\small]
\node[state] (omega) at (0,0) {\;\;\;\;\;\;};
\draw (omega) ++(-0.45, 0) to (-0.45,0.7);
\draw (omega) ++(-0.15, 0) to (-0.15,0.3) to[out=90,in=-90] (0.15,0.7);
\draw (omega) ++(0.15, 0) to (0.15,0.3) to[out=90,in=-90] (-0.15,0.7);
\draw (omega) ++(0.45, 0) to (0.45,0.7);
\path[scriptstyle]
node at (-0.45,0.85) {$X$}
node at (-0.15,0.85) {$Z$}
node at (0.15,0.85) {$Y$}
node at (0.45,0.85) {$W$};
\end{ctikzpicture}
\;\;=\;\;
\begin{ctikzpicture}[font=\small]
\node[state] (omega) at (0,0) {\;\;\;\;\;\;};
\draw (omega) ++(-0.45, 0) to (-0.45,0.3);
\draw (omega) ++(-0.15, 0) to (-0.15,0.3);
\draw (omega) ++(0.15, 0) to (0.15,0.3);
\draw (omega) ++(0.45, 0) to (0.45,0.3);
\path[scriptstyle]
node at (-0.45,0.45) {$X$}
node at (-0.15,0.45) {$Z$}
node at (0.15,0.45) {$Y$}
node at (0.45,0.45) {$W$};
\end{ctikzpicture}
\;\;=\;\;
\begin{ctikzpicture}[font=\small]
\node[state] (omega) at (0,-0.05) {\;\;\;\;};
\node[copier] (c) at (-0.25,0.15) {};
\node[copier] (c2) at (0.25,0.15) {};
\node[arrow box] (b) at (-0.5,0.7) {\;\;\;\;\;\;\;};
\draw (omega) ++ (-0.25,0) to (c);
\draw (c) to[out=165,in=-90] ([xshiftu=-0.25]b.south);
\draw (c) to[out=15,in=-90] (0.25,0.55)
to (0.25,1.15);
\draw (omega) ++ (0.25,0) to (c2);
\draw (c2) to[out=165,in=-90] ([xshiftu=0.25]b.south);
\draw (c2) to[out=15,in=-90] (0.65,0.55) to (0.65,1.15);
\draw (b.north) ++ (-0.25,0) to (-0.75,1.15);
\draw (b.north) ++ (0.25,0) to (-0.25,1.15);
\path[scriptstyle]
node at (-0.75,1.3) {$X$}
node at (0.25,1.3) {$Y$}
node at (-0.25,1.3) {$Z$}
node at (0.65,1.3) {$W$};
\end{ctikzpicture}
\end{align*}
The marked equality $\overset{\star}{=}$
holds by strong almost equality, see Proposition~\ref{prop:disint-eq-str}.
\end{proof}

The equations correspond respectively to
$\sum_y\Pr(x,y|z)=\Pr(x|z)$
and $\Pr(x|y,z)\cdot \Pr(z|y,w)=\Pr(x,z|y,w)$
in discrete probability.


\begin{remark}
We here keep our notation somewhat informal,
e.g.\ using symbols $X,Y,Z,\dotsc$ as a sort of meta-variables.
We refer to \cite{JoyalS1991,Selinger2010}
for more formal aspects of string diagrams.
\end{remark}

\subsection{Conditional independence}
\tikzextname{ss_cind}

We continue using the notations in the previous subsection.
Recall that we fix an $n$-partite state
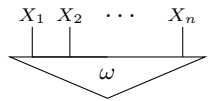
\[
\begin{tikzpicture}[font=\small]
\node[state] (omega) at (0,0) {\hspace*{1em}$\omega$\hspace*{1em}};
\draw (omega) -- (-1.0, 0) to (-1.0,0.4);
\draw (omega) ++(-0.5, 0) to (-0.5,0.4);
\draw (omega) ++(1.0, 0) to (1.0,0.4);
\node[font=\normalsize] at (0.2,0.55) {$\dotso$};
\path[font=\normalsize,
execute at begin node=\everymath{\scriptstyle}]
node at (-1,0.55) {$X_1$}
node at (-0.5,0.55) {$X_2$}
node at (1,0.55) {$X_n$};
\end{tikzpicture}
\]
and use symbols $X,Y,Z,W,\dotsc$
to denote a wire $X_i$ or multiple wires $X_i\otimes X_j\otimes\dotsb$.

We now introduce the notion of conditional independence.  Although it
is defined with respect to the underlying state $\omega$, we leave the
state $\omega$ implicit, like an underlying probability space $\Omega$
in conventional probability theory.

\begin{definition}
Let $X,Y,Z$ denote distinct wires.
Then we say $X$ and $Y$ are conditionally independent given $Z$,
written as $\cind{X}{Y}{Z}$, if
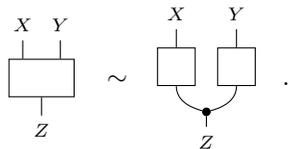
\[
\begin{ctikzpicture}[font=\small]
\node[arrow box] (a) at (0,0.5) {\;\;\;\;\;\;\;};
\draw (0,0) to (a);
\draw (a.north) ++(-0.25,0) to (-0.25,1);
\draw (a.north) ++(0.25,0) to (0.25,1);
\path[font=\normalsize,
execute at begin node=\everymath{\scriptstyle}]
node at (-0.25,1.2) {$X$}
node at (0.25,1.2) {$Y$}
node at (0,-0.2) {$Z$};
\end{ctikzpicture}
\;\;\;\equpto\;\;\;
\begin{ctikzpicture}[font=\small]
\node[copier] (c) at (0,0.2) {};
\node[arrow box] (a) at (-0.4,0.8) {};
\node[arrow box] (b) at (0.4,0.8) {};
\draw (0,0) to (c);
\draw (c) to[out=165,in=-90] (a);
\draw (a) to (-0.4,1.3);
\draw (b) to (0.4,1.3);
\draw (c) to[out=15,in=-90] (b);
\path[font=\normalsize,
execute at begin node=\everymath{\scriptstyle}]
node at (-0.4,1.5) {$X$}
node at (0.4,1.5) {$Y$}
node at (0,-0.2) {$Z$};
\end{ctikzpicture}
\quad.
\]
\end{definition}

The definition is analogous to
the condition $\Pr(x,y|z)=\Pr(x|z)\Pr(y|z)$
in discrete probability.
Indeed our definition coincides with the usual conditional
independence,
as explained below.

\begin{example}
In $\Kl(\Dst)$,
let $c_{X|Z}\colon Z\to \Dst(X)$,
$c_{Y|Z}\colon Z\to \Dst(Y)$,
$c_{XY|Z}\colon Z\to \Dst(X\times Y)$
be disintegrations of some joint state, say $\omega\in\Dst(X\times Y\times Z)$.
Let $\omega_Z\in\Dst(Z)$ be the marginal on $Z$.
Then $\cind{X}{Y}{Z}$ if and only if
\[
c_{XY|Z}(z)(x,y)=c_{X|Z}(z)(x)\cdot c_{Y|Z}(z)(y)
\quad \text{whenever} \quad
\omega_Z(z)\ne 0
\]
for all $x\in X$, $y\in Y$ and $z\in Z$.
If we write $\Pr(x,y|z)=c_{XY|Z}(z)(x,y)$,
$\Pr(x|z)=c_{X|Z}(z)(x)$,
$\Pr(y|z)=c_{Y|Z}(z)(y)$,
and $\Pr(z)=\omega_Z(z)$,
then the condition will look more familiar:
\[
\Pr(x,y|z) = \Pr(x|z)\cdot \Pr(y|z)
\quad \text{whenever} \quad
\Pr(z)\ne 0
\enspace.
\]
\end{example}

\begin{example}
Similarly, in $\Kl(\Giry)$,
let $c_{X|Z}\colon Z\to \Giry(X)$,
$c_{Y|Z}\colon Z\to \Giry(Y)$,
$c_{XY|Z}\colon Z\to \Giry(X\times Y)$,
and $\omega_Z\in\Giry(Z)$
be appropriate disintegrations and a marginal of some joint probability measure $\omega$.
Then $\cind{X}{Y}{Z}$ if and only if
\[
c_{XY|Z}(z)(A\times B)=c_{X|Z}(z)(A)\cdot c_{Y|Z}(z)(B)
\qquad
\text{for $\omega_Z$-almost all $z\in Z$}
\]
for all $A\in \Sigma_X$ and $B\in \Sigma_Y$.
\end{example}

The equivalences in the next result are well-known in conditional
probability. Our contribution is that we formulate and prove them at
an abstract, graphical level.

\begin{proposition}
\label{prop:equiv-cind}\tikzextname{prop_equiv_cind}
The following are equivalent.
\begin{enumerate}
\setlength{\abovedisplayskip}{-.5\baselineskip}
\setlength{\abovedisplayshortskip}{-.5\baselineskip}
\setlength{\belowdisplayskip}{0pt}
\setlength{\belowdisplayshortskip}{0pt}
\item\label{propen:cind}
$\cind{X}{Y}{Z}$
\item\label{propen:factor}
\[
\begin{ctikzpicture}[font=\small]
\node[state] (omega) at (0,0) {\;\;\;\;\;\;};
\draw (omega) ++(-0.4, 0) to +(0,0.3);
\draw (omega) to +(0,0.3);
\draw (omega) ++(0.4, 0) to +(0,0.3);
\path[font=\normalsize,
execute at begin node=\everymath{\scriptstyle}]
node at (-0.4,0.5) {$X$}
node at (0,0.5) {$Y$}
node at (0.4,0.5) {$Z$};
\end{ctikzpicture}
\;\;=\;\;
\begin{ctikzpicture}[font=\small]
\node[state] (omega) at (0,0) {};
\node[copier] (c) at (0,0.2) {};
\node[arrow box] (a) at (-0.7,0.8) {};
\node[arrow box] (b) at (0,0.8) {};
\draw (omega) to (c);
\draw (c) to[out=165,in=-90] (a);
\draw (a) to (-0.7,1.3);
\draw (c) to (b);
\draw (b) to (0,1.3);
\draw (c) to[out=15,in=-90] (0.7,0.6)
to (0.7,1.3);
\path[font=\normalsize,
execute at begin node=\everymath{\scriptstyle}]
node at (-0.7,1.5) {$X$}
node at (0,1.5) {$Y$}
node at (0.7,1.5) {$Z$};
\end{ctikzpicture}
\]
\item\label{propen:alt}
\[
\begin{ctikzpicture}[font=\small]
\node[arrow box] (a) at (0,0.5) {\;\;\;\;\;\;\;};
\draw (a) to (0,1);
\draw (0.25,0) to ([xshiftu=0.25]a.south);
\draw (-0.25,0) to ([xshiftu=-0.25]a.south);
\path[font=\normalsize,
execute at begin node=\everymath{\scriptstyle}]
node at (0,1.2) {$X$}
node at (-0.25,-0.2) {$Y$}
node at (0.25,-0.2) {$Z$};
\end{ctikzpicture}
\;\;\equpto\;\;
\begin{ctikzpicture}[font=\small]
\node[arrow box] (a) at (0,0.5) {};
\node[discarder] (d) at (-0.6,0.4) {};
\draw (a) to (0,1);
\draw (0,0) to (a);
\draw (-0.6,0) to (d);
\path[font=\normalsize,
execute at begin node=\everymath{\scriptstyle}]
node at (0,1.2) {$X$}
node at (-0.6,-0.2) {$Y$}
node at (0,-0.2) {$Z$};
\end{ctikzpicture}
\]
\item\label{propen:factor2}
\[
\begin{ctikzpicture}[font=\small]
\node[state] (omega) at (0,0) {\;\;\;\;\;\;};
\draw (omega) ++(-0.4, 0) to +(0,0.3);
\draw (omega) to +(0,0.3);
\draw (omega) ++(0.4, 0) to +(0,0.3);
\tikzset{font=\normalsize,
execute at begin node=\everymath{\scriptstyle}}
\node at (-0.4,0.5) {$X$};
\node at (0,0.5) {$Y$};
\node at (0.4,0.5) {$Z$};
\end{ctikzpicture}
\;\;=\;\;
\begin{ctikzpicture}[font=\small]
\node[state] (omega) at (0,0) {\;\;\;\;};
\node[copier] (c) at (0.25,0.2) {};
\node[arrow box] (a) at (-0.45,1) {};
\draw (omega) ++(-0.25, 0) to +(0,0.15)
to[out=90,in=-90] (0.1,0.6) to (0.1,1.45);
\draw (omega) ++(0.25, 0) to (c);
\draw (c) to[out=150,in=-90] (a);
\draw (c) to[out=60,in=-90] (0.5,1.45);
\draw (a) to (-0.45,1.45);
\path[font=\normalsize,
execute at begin node=\everymath{\scriptstyle}]
node at (-0.45,1.6) {$X$}
node at (0.1,1.6) {$Y$}
node at (0.5,1.6) {$Z$};
\end{ctikzpicture}
\]
\end{enumerate}
\end{proposition}

As we will see below,
conditional independence $\cind{X}{Y}{Z}$
is symmetric in $X$ and~$Y$.
Therefore the obvious symmetric counterparts
of \ref{propen:alt} and \ref{propen:factor2}
are also equivalent to them.

\begin{proof}\tikzextname{pf_equiv_cind}
By definition of almost equality,
\ref{propen:cind} is equivalent to
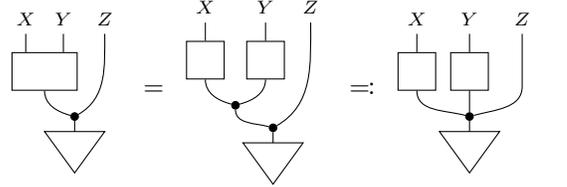
\[
\begin{ctikzpicture}[font=\small]
\node[arrow box] (a) at (0,0.5) {\;\;\;\;\;\;\;};
\node[copier] (c2) at (0.4,-0.1) {};
\node[state] (s) at (0.4,-0.3) {};
\draw (s) to (c2);
\draw (c2) to[out=165,in=-90] (a);
\draw (c2) to[out=30,in=-90] (0.8,1);
\draw (a.north) ++(-0.25,0) to (-0.25,1);
\draw (a.north) ++(0.25,0) to (0.25,1);
\path[font=\normalsize,
execute at begin node=\everymath{\scriptstyle}]
node at (-0.25,1.2) {$X$}
node at (0.25,1.2) {$Y$}
node at (0.8,1.2) {$Z$};
\end{ctikzpicture}
\;\;=\;\;
\begin{ctikzpicture}[font=\small]
\node[copier] (c) at (0,0.2) {};
\node[copier] (c2) at (0.5,-0.1) {};
\node[arrow box] (a) at (-0.4,0.8) {};
\node[arrow box] (b) at (0.4,0.8) {};
\node[state] (s) at (0.5,-0.3) {};
\draw (s) to (c2);
\draw (c2) to[out=165,in=-90] (c);
\draw (c2) to[out=30,in=-90] (1,1.3);
\draw (c) to[out=165,in=-90] (a);
\draw (a) to (-0.4,1.3);
\draw (b) to (0.4,1.3);
\draw (c) to[out=15,in=-90] (b);
\path[font=\normalsize,
execute at begin node=\everymath{\scriptstyle}]
node at (-0.4,1.5) {$X$}
node at (0.4,1.5) {$Y$}
node at (1,1.5) {$Z$};
\end{ctikzpicture}
\;\;\eqqcolon\;\;
\begin{ctikzpicture}[font=\small]
\node[state] (omega) at (0,0) {};
\node[copier] (c) at (0,0.2) {};
\node[arrow box] (a) at (-0.7,0.8) {};
\node[arrow box] (b) at (0,0.8) {};
\draw (omega) to (c);
\draw (c) to[out=165,in=-90] (a);
\draw (a) to (-0.7,1.3);
\draw (c) to (b);
\draw (b) to (0,1.3);
\draw (c) to[out=15,in=-90] (0.7,0.6)
to (0.7,1.3);
\path[font=\normalsize,
execute at begin node=\everymath{\scriptstyle}]
node at (-0.7,1.5) {$X$}
node at (0,1.5) {$Y$}
node at (0.7,1.5) {$Z$};
\end{ctikzpicture}
\quad.
\]
We then have \ref{propen:cind} $\Leftrightarrow$ \ref{propen:factor},
since the identity below holds by the definition of disintegration.
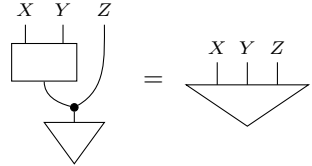
\[
\begin{ctikzpicture}[font=\small]
\node[arrow box] (a) at (0,0.5) {\;\;\;\;\;\;\;};
\node[copier] (c2) at (0.4,-0.1) {};
\node[state] (s) at (0.4,-0.3) {};
\draw (s) to (c2);
\draw (c2) to[out=165,in=-90] (a);
\draw (c2) to[out=30,in=-90] (0.8,1);
\draw (a.north) ++(-0.25,0) to (-0.25,1);
\draw (a.north) ++(0.25,0) to (0.25,1);
\path[font=\normalsize,
execute at begin node=\everymath{\scriptstyle}]
node at (-0.25,1.2) {$X$}
node at (0.25,1.2) {$Y$}
node at (0.8,1.2) {$Z$};
\end{ctikzpicture}
\;\;=\;\;
\begin{ctikzpicture}[font=\small]
\node[state] (omega) at (0,0) {\;\;\;\;\;\;};
\draw (omega) ++(-0.4, 0) to +(0,0.3);
\draw (omega) to +(0,0.3);
\draw (omega) ++(0.4, 0) to +(0,0.3);
\tikzset{font=\normalsize,
execute at begin node=\everymath{\scriptstyle}}
\node at (-0.4,0.5) {$X$};
\node at (0,0.5) {$Y$};
\node at (0.4,0.5) {$Z$};
\end{ctikzpicture}
\]
Next, assuming \ref{propen:alt}, we obtain \ref{propen:factor2} as follows.
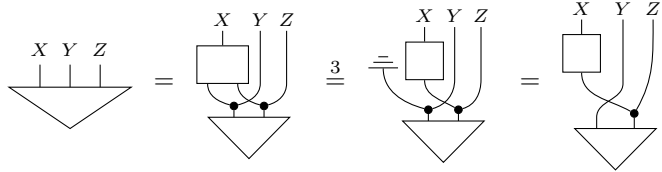
\[
\begin{ctikzpicture}[font=\small]
\node[state] (omega) at (0,0) {\;\;\;\;\;\;};
\draw (omega) ++(-0.4, 0) to +(0,0.3);
\draw (omega) to +(0,0.3);
\draw (omega) ++(0.4, 0) to +(0,0.3);
\tikzset{font=\normalsize,
execute at begin node=\everymath{\scriptstyle}}
\node at (-0.4,0.5) {$X$};
\node at (0,0.5) {$Y$};
\node at (0.4,0.5) {$Z$};
\end{ctikzpicture}
\;\;=\;\;
\begin{ctikzpicture}[font=\small]
\node[state] (omega) at (0,0) {\;\;\;\;};
\node[copier] (c) at (-0.2,0.15) {};
\node[copier] (c2) at (0.2,0.15) {};
\node[arrow box] (b) at (-0.35,0.7) {\;\;\;\;\;};
\draw (omega) ++ (-0.2,0) to (c);
\draw (c) to[out=165,in=-90] ([xshiftu=-0.2]b.south);
\draw (b) to (-0.35,1.15);
\draw (c) to[out=15,in=-90] (0.15,0.45)
to (0.15,1.15);
\draw (omega) ++ (0.2,0) to (c2);
\draw (c2) to[out=165,in=-90] ([xshiftu=0.2]b.south);
\draw (c2) to[out=15,in=-90] (0.5,0.45)
to (0.5,1.15);
\path[font=\normalsize,
execute at begin node=\everymath{\scriptstyle}]
node at (-0.35,1.3) {$X$}
node at (0.15,1.3) {$Y$}
node at (0.5,1.3) {$Z$};
\end{ctikzpicture}
\;\;\overset{\text{\ref{propen:alt}}}=\;\;
\begin{ctikzpicture}[font=\small]
\node[state] (omega) at (0,0) {\;\;\;\;};
\node[copier] (c) at (-0.2,0.15) {};
\node[copier] (c2) at (0.2,0.15) {};
\node[arrow box] (b) at (-0.25,0.8) {};
\node[discarder] (d) at (-0.8,0.7) {};
\draw (omega) ++ (-0.2,0) to (c);
\draw (c) to[out=165,in=-90] (d);
\draw (b) to (-0.25,1.25);
\draw (c) to[out=15,in=-90] (0.15,0.55)
to (0.15,1.25);
\draw (omega) ++ (0.2,0) to (c2);
\draw (c2) to[out=150,in=-90] (b);
\draw (c2) to[out=15,in=-90] (0.5,0.55)
to (0.5,1.25);
\path[font=\normalsize,
execute at begin node=\everymath{\scriptstyle}]
node at (-0.25,1.4) {$X$}
node at (0.15,1.4) {$Y$}
node at (0.5,1.4) {$Z$};
\end{ctikzpicture}
\;\;=\;\;
\begin{ctikzpicture}[font=\small]
\node[state] (omega) at (0,0) {\;\;\;\;};
\node[copier] (c) at (0.25,0.2) {};
\node[arrow box] (a) at (-0.45,1) {};
\draw (omega) ++(-0.25, 0) to +(0,0.15)
to[out=90,in=-90] (0.1,0.6) to (0.1,1.45);
\draw (omega) ++(0.25, 0) to (c);
\draw (c) to[out=150,in=-90] (a);
\draw (c) to[out=60,in=-90] (0.5,1.45);
\draw (a) to (-0.45,1.45);
\path[font=\normalsize,
execute at begin node=\everymath{\scriptstyle}]
node at (-0.45,1.6) {$X$}
node at (0.1,1.6) {$Y$}
node at (0.5,1.6) {$Z$};
\end{ctikzpicture}
\]
We prove
\ref{propen:factor2} $\Rightarrow$ \ref{propen:alt}
similarly.
Finally note that the following equation holds.
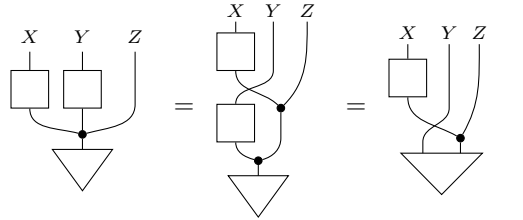
\[
\begin{ctikzpicture}[font=\small]
\node[state] (omega) at (0,0) {};
\node[copier] (c) at (0,0.2) {};
\node[arrow box] (a) at (-0.7,0.8) {};
\node[arrow box] (b) at (0,0.8) {};
\draw (omega) to (c);
\draw (c) to[out=165,in=-90] (a);
\draw (a) to (-0.7,1.3);
\draw (c) to (b);
\draw (b) to (0,1.3);
\draw (c) to[out=15,in=-90] (0.7,0.6)
to (0.7,1.3);
\path[font=\normalsize,
execute at begin node=\everymath{\scriptstyle}]
node at (-0.7,1.5) {$X$}
node at (0,1.5) {$Y$}
node at (0.7,1.5) {$Z$};
\end{ctikzpicture}
\;\;=\;\;
\begin{ctikzpicture}[font=\small]
\node[state] (omega) at (0,0) {};
\node[copier] (c) at (0,0.2) {};
\node[arrow box] (a) at (-0.3,1.65) {};
\node[arrow box] (b) at (-0.3,0.7) {};
\node[copier] (c2) at (0.3,0.9) {};
\draw (omega) to (c);
\draw (a) to (-0.3,2.05);
\draw (c) to[out=165,in=-90] (b);
\draw (b) to[out=90,in=-90] (0.2,1.4) to (0.2,2.05);
\draw (c) to[out=15,in=-90] (0.3,0.5)
to (c2);
\draw (c2) to[out=135,in=-90] (a);
\draw (c2) to[out=45,in=-90] (0.65,2.05);
\path[font=\normalsize,
execute at begin node=\everymath{\scriptstyle}]
node at (-0.3,2.2) {$X$}
node at (0.2,2.2) {$Y$}
node at (0.65,2.2) {$Z$};
\end{ctikzpicture}
\;\;=\;\;
\begin{ctikzpicture}[font=\small]
\node[state] (omega) at (0,0) {\;\;\;\;};
\node[copier] (c) at (0.25,0.2) {};
\node[arrow box] (a) at (-0.45,1) {};
\draw (omega) ++(-0.25, 0) to +(0,0.15)
to[out=90,in=-90] (0.1,0.6) to (0.1,1.45);
\draw (omega) ++(0.25, 0) to (c);
\draw (c) to[out=150,in=-90] (a);
\draw (c) to[out=60,in=-90] (0.5,1.45);
\draw (a) to (-0.45,1.45);
\path[font=\normalsize,
execute at begin node=\everymath{\scriptstyle}]
node at (-0.45,1.6) {$X$}
node at (0.1,1.6) {$Y$}
node at (0.5,1.6) {$Z$};
\end{ctikzpicture}
\quad.
\]
From this
\ref{propen:factor}~$\Leftrightarrow$~\ref{propen:factor2}
is immediate.
\end{proof}

Note that the condition~\ref{propen:alt} of the proposition is an
analogue of $\Pr(x|y,z)=\Pr(x|z)$.  The other conditions
\ref{propen:factor} and \ref{propen:factor2} say
that the joint state can be factorised in certain ways, corresponding
to the following equations:
\[
\Pr(x,y,z)
 =
\Pr(x|z)\Pr(y|z)\Pr(z)
 =
\Pr(x|z)\Pr(y,z).
\]

The proposition below shows that our abstract formulation of
conditional independence
does satisfy the basic `rules' of conditional independence,
which are known as
\emph{(semi-) graphoids axioms}~\cite{VermaP1988,GeigerVP1990}.

\begin{proposition}
\label{prop:graphoid}\tikzextname{prop_graphoid}
Conditional independence $\cind{(-)}{(-)}{(-)}$ satisfies:
\begin{enumerate}
\item\label{propen:symm}
(Symmetry)
$\cind{X}{Y}{Z}$ if and only if $\cind{Y}{X}{Z}$.
\item\label{propen:decomp}
(Decomposition)
$\cind{X}{Y\otimes Z}{W}$ implies $\cind{X}{Y}{W}$
and $\cind{X}{Z}{W}$.
\item\label{propen:weakunion}
(Weak union)
$\cind{X}{Y\otimes Z}{W}$ implies $\cind{X}{Y}{Z\otimes W}$.
\item\label{propen:contr}
(Contraction)
$\cind{X}{Z}{W}$ and
$\cind{X}{Y}{Z\otimes W}$
imply $\cind{X}{Y\otimes Z}{W}$.
\end{enumerate}
\end{proposition}
\begin{proof}\tikzextname{prf_graphoid}
We will freely use Proposition~\ref{prop:equiv-cind}.

(\ref{propen:symm})
Suppose $\cind{X}{Y}{Z}$. Then
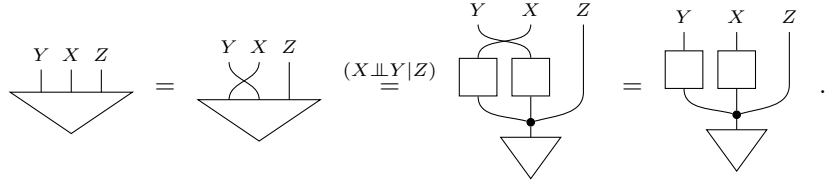
\[
\begin{ctikzpicture}[font=\small]
\node[state] (omega) at (0,0) {\;\;\;\;\;\;};
\draw (omega) ++(-0.4, 0) to +(0,0.3);
\draw (omega) to +(0,0.3);
\draw (omega) ++(0.4, 0) to +(0,0.3);
\tikzset{font=\normalsize,
execute at begin node=\everymath{\scriptstyle}}
\node at (-0.4,0.5) {$Y$};
\node at (0,0.5) {$X$};
\node at (0.4,0.5) {$Z$};
\end{ctikzpicture}
\;\;=\;\;
\begin{ctikzpicture}[font=\small]
\node[state] (omega) at (0,0) {\;\;\;\;\;\;};
\draw (omega) ++(-0.4, 0) to[out=90,in=-90] +(0.4,0.5);
\draw (omega) to[out=90,in=-90] +(-0.4,0.5);
\draw (omega) ++(0.4, 0) to +(0,0.5);
\path[font=\normalsize,
execute at begin node=\everymath{\scriptstyle}]
node at (0,0.7) {$X$}
node at (-0.4,0.7) {$Y$}
node at (0.4,0.7) {$Z$};
\end{ctikzpicture}
\;\;\overset{(\cind{X}{Y}{Z})}{=}\;\;
\begin{ctikzpicture}[font=\small]
\node[state] (omega) at (0,0) {};
\node[copier] (c) at (0,0.2) {};
\node[arrow box] (a) at (-0.7,0.8) {};
\node[arrow box] (b) at (0,0.8) {};
\draw (omega) to (c);
\draw (c) to[out=165,in=-90] (a);
\draw (a.north) to[out=90,in=-90] (0,1.5);
\draw (c) to (b);
\draw (b.north) to[out=90,in=-90] (-0.7,1.5);
\draw (c) to[out=15,in=-90] (0.7,0.6)
to (0.7,1.5);
\path[font=\normalsize,
execute at begin node=\everymath{\scriptstyle}]
node at (-0.7,1.7) {$Y$}
node at (0,1.7) {$X$}
node at (0.7,1.7) {$Z$};
\end{ctikzpicture}
\;\;=\;\;
\begin{ctikzpicture}[font=\small]
\node[state] (omega) at (0,0) {};
\node[copier] (c) at (0,0.2) {};
\node[arrow box] (a) at (-0.7,0.8) {};
\node[arrow box] (b) at (0,0.8) {};
\draw (omega) to (c);
\draw (c) to[out=165,in=-90] (a);
\draw (a) to (-0.7,1.3);
\draw (c) to (b);
\draw (b) to (0,1.3);
\draw (c) to[out=15,in=-90] (0.7,0.6)
to (0.7,1.3);
\path[font=\normalsize,
execute at begin node=\everymath{\scriptstyle}]
node at (-0.7,1.5) {$Y$}
node at (0,1.5) {$X$}
node at (0.7,1.5) {$Z$};
\end{ctikzpicture}
\enspace.
\]
This means $\cind{Y}{X}{Z}$.

(\ref{propen:decomp})
Suppose $\cind{X}{Y\otimes Z}{W}$, namely:
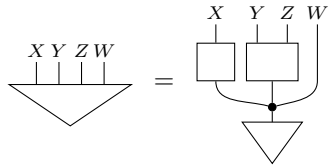
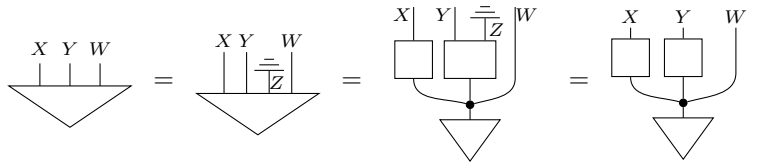
\[
\begin{ctikzpicture}[font=\small]
\node[state] (omega) at (0,0) {\;\;\;\;\;\;};
\draw (omega) ++(-0.45, 0) to +(0,0.3);
\draw (omega) ++(-0.15, 0) to +(0,0.3);
\draw (omega) ++(0.15, 0) to +(0,0.3);
\draw (omega) ++(0.45, 0) to +(0,0.3);
\path[font=\normalsize,
execute at begin node=\everymath{\scriptstyle}]
node at (-0.45,0.45) {$X$}
node at (-0.15,0.45) {$Y$}
node at (0.15,0.45) {$Z$}
node at (0.45,0.45) {$W$};
\end{ctikzpicture}
\;\;=\;\;
\begin{ctikzpicture}[font=\small]
\node[state] (omega) at (0,0) {};
\node[copier] (c) at (0,0.2) {};
\node[arrow box] (a) at (-0.75,0.8) {};
\node[arrow box] (b) at (0,0.8) {\;\;\;\;\;};
\draw (omega) to (c);
\draw (c) to[out=165,in=-90] (a);
\draw (a) to (-0.75,1.3);
\draw (c) to (b);
\draw ([xshiftu=-0.2]b.north) to (-0.2,1.3);
\draw ([xshiftu=0.2]b.north) to (0.2,1.3);
\draw (c) to[out=15,in=-90] (0.6,0.55)
to (0.6,1.3);
\path[font=\normalsize,
execute at begin node=\everymath{\scriptstyle}]
node at (-0.75,1.45) {$X$}
node at (-0.2,1.45) {$Y$}
node at (0.2,1.45) {$Z$}
node at (0.6,1.45) {$W$};
\end{ctikzpicture}
\]
Marginalising $Z$, we obtain
\[
\begin{ctikzpicture}[font=\small]
\node[state] (omega) at (0,0) {\;\;\;\;\;\;};
\draw (omega) ++(-0.4, 0) to +(0,0.3);
\draw (omega) to +(0,0.3);
\draw (omega) ++(0.4, 0) to +(0,0.3);
\path[font=\normalsize,
execute at begin node=\everymath{\scriptstyle}]
node at (-0.4,0.5) {$X$}
node at (0,0.5) {$Y$}
node at (0.4,0.5) {$W$};
\end{ctikzpicture}
\;\;=\;\;
\begin{ctikzpicture}[font=\small]
\node[state] (omega) at (0,0) {\;\;\;\;\;\;};
\node[discarder] (d) at (0.15,0.3) {};
\draw (omega) ++(-0.45, 0) to +(0,0.55);
\draw (omega) ++(-0.15, 0) to +(0,0.55);
\draw (omega) ++(0.15, 0) to (d);
\draw (omega) ++(0.45, 0) to +(0,0.55);
\path[font=\normalsize,
execute at begin node=\everymath{\scriptstyle}]
node at (-0.45,0.7) {$X$}
node at (-0.15,0.7) {$Y$}
node at (0.25,0.15) {$Z$}
node at (0.45,0.7) {$W$};
\end{ctikzpicture}
\;\;=\;\;
\begin{ctikzpicture}[font=\small]
\node[state] (omega) at (0,0) {};
\node[copier] (c) at (0,0.2) {};
\node[arrow box] (a) at (-0.75,0.8) {};
\node[arrow box] (b) at (0,0.8) {\;\;\;\;\;};
\node[discarder] (d) at (0.2,1.35) {};
\draw (omega) to (c);
\draw (c) to[out=165,in=-90] (a);
\draw (a) to (-0.75,1.5);
\draw (c) to (b);
\draw ([xshiftu=-0.2]b.north) to (-0.2,1.5);
\draw ([xshiftu=0.2]b.north) to (d);
\draw (c) to[out=15,in=-90] (0.6,0.55)
to (0.6,1.5);
\path[font=\normalsize,
execute at begin node=\everymath{\scriptstyle}]
node at (-0.9,1.4) {$X$}
node at (-0.35,1.4) {$Y$}
node at (0.35,1.2) {$Z$}
node at (0.75,1.4) {$W$};
\end{ctikzpicture}
\;\;=\;\;
\begin{ctikzpicture}[font=\small]
\node[state] (omega) at (0,0) {};
\node[copier] (c) at (0,0.2) {};
\node[arrow box] (a) at (-0.7,0.8) {};
\node[arrow box] (b) at (0,0.8) {};
\draw (omega) to (c);
\draw (c) to[out=165,in=-90] (a);
\draw (a) to (-0.7,1.2);
\draw (c) to (b);
\draw (b) to (0,1.2);
\draw (c) to[out=15,in=-90] (0.7,0.6)
to (0.7,1.2);
\path[font=\normalsize,
execute at begin node=\everymath{\scriptstyle}]
node at (-0.7,1.35) {$X$}
node at (0,1.35) {$Y$}
node at (0.7,1.35) {$W$};
\end{ctikzpicture}
\quad,
\]
by Proposition~\ref{prop:disint-identitiy}.\ref{propen:disint-marg}.
Thus $\cind{X}{Y}{W}$.
Similarly we prove $\cind{X}{Z}{W}$.

Finally, we prove \ref{propen:weakunion} and \ref{propen:contr} at the same time.
Note that $\cind{X}{Y\otimes Z}{W}$ implies $\cind{X}{Z}{W}$, as shown above.
Therefore what we need to prove is that
$\cind{X}{Y\otimes Z}{W}$ if and only if $\cind{X}{Y}{Z\otimes W}$,
under $\cind{X}{Z}{W}$.
Assume $\cind{X}{Z}{W}$, so we have
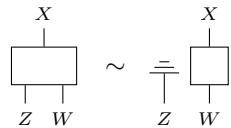
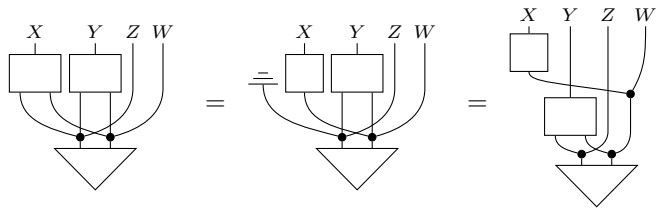
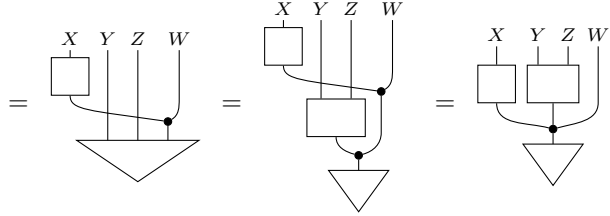
\[
\begin{ctikzpicture}[font=\small]
\node[arrow box] (a) at (0,0.5) {\;\;\;\;\;\;\;};
\draw (a) to (0,1);
\draw (0.25,0) to ([xshiftu=0.25]a.south);
\draw (-0.25,0) to ([xshiftu=-0.25]a.south);
\path[font=\normalsize,
execute at begin node=\everymath{\scriptstyle}]
node at (0,1.2) {$X$}
node at (-0.25,-0.2) {$Z$}
node at (0.25,-0.2) {$W$};
\end{ctikzpicture}
\;\;\equpto\;\;
\begin{ctikzpicture}[font=\small]
\node[arrow box] (a) at (0,0.5) {};
\node[discarder] (d) at (-0.6,0.4) {};
\draw (a) to (0,1);
\draw (0,0) to (a);
\draw (-0.6,0) to (d);
\path[font=\normalsize,
execute at begin node=\everymath{\scriptstyle}]
node at (0,1.2) {$X$}
node at (-0.6,-0.2) {$Z$}
node at (0,-0.2) {$W$};
\end{ctikzpicture}
\]
Then
\begin{align*}
\begin{ctikzpicture}[font=\small]
\node[state] (omega) at (0,0) {\;\;\;\;};
\node[copier] (c) at (-0.2,0.15) {};
\node[copier] (c2) at (0.2,0.15) {};
\node[arrow box] (a) at (-0.8,1) {\;\;\;\;\;};
\node[arrow box] (b) at (0,1) {\;\;\;\;\;};
\draw (omega) ++ (-0.2,0) to (c);
\draw (c) to[out=165,in=-90] ([xshiftu=-0.2]a.south);
\draw (a) to (-0.8,1.4);
\draw (c) to ([xshiftu=-0.2]b.south);
\draw (b) to (0,1.4);
\draw (c) to[out=15,in=-90] (0.5,0.8)
to (0.5,1.4);
\draw (omega) ++ (0.2,0) to (c2);
\draw (c2) to[out=165,in=-90] ([xshiftu=0.2]a.south);
\draw (c2) to ([xshiftu=0.2]b.south);
\draw (c2) to[out=15,in=-90] (0.9,0.8)
to (0.9,1.4);
\path[font=\normalsize,
execute at begin node=\everymath{\scriptstyle}]
node at (-0.8,1.55) {$X$}
node at (0,1.55) {$Y$}
node at (0.5,1.55) {$Z$}
node at (0.9,1.55) {$W$};
\end{ctikzpicture}
&\;\;=\;\;
\begin{ctikzpicture}[font=\small]
\node[state] (omega) at (0,0) {\;\;\;\;};
\node[copier] (c) at (-0.2,0.15) {};
\node[copier] (c2) at (0.2,0.15) {};
\node[arrow box] (a) at (-0.7,1) {};
\node[arrow box] (b) at (0,1) {\;\;\;\;\;};
\node[discarder] (d) at (-1.25,0.85) {};
\draw (omega) ++ (-0.2,0) to (c);
\draw (c) to[out=165,in=-90] (d);
\draw (a) to (-0.7,1.4);
\draw (c) to ([xshiftu=-0.2]b.south);
\draw (b) to (0,1.4);
\draw (c) to[out=15,in=-90] (0.5,0.8)
to (0.5,1.4);
\draw (omega) ++ (0.2,0) to (c2);
\draw (c2) to[out=165,in=-90] (a.south);
\draw (c2) to ([xshiftu=0.2]b.south);
\draw (c2) to[out=15,in=-90] (0.9,0.8)
to (0.9,1.4);
\path[font=\normalsize,
execute at begin node=\everymath{\scriptstyle}]
node at (-0.7,1.55) {$X$}
node at (0,1.55) {$Y$}
node at (0.5,1.55) {$Z$}
node at (0.9,1.55) {$W$};
\end{ctikzpicture}
\;\;=\;\;
\begin{ctikzpicture}[font=\small]
\node[state] (omega) at (0,0) {\;\;\;\;};
\node[copier] (c) at (-0.2,0.15) {};
\node[copier] (c2) at (0.2,0.15) {};
\node[arrow box] (a) at (-0.9,1.5) {};
\node[arrow box] (b) at (-0.35,0.65) {\;\;\;\;\;};
\node[copier] (c3) at (0.45,0.95) {};
\draw (omega) ++ (-0.2,0) to (c);
\draw (a) to (-0.9,1.85);
\draw (c) to[out=165,in=-90] ([xshiftu=-0.2]b.south);
\draw (b) to (-0.35,1.85);
\draw (c) to[out=15,in=-90] (0.15,0.45)
to (0.15,1.85);
\draw (omega) ++ (0.2,0) to (c2);
\draw (c2) to[out=165,in=-90] ([xshiftu=0.2]b.south);
\draw (c2) to[out=15,in=-90] (0.45,0.45)
to (c3);
\draw (c3) .. controls (-0.5,1.1) and (-0.9,1.1) .. (a.south);
\draw (c3) to[out=60,in=-90] (0.65,1.85);
\path[font=\normalsize,
execute at begin node=\everymath{\scriptstyle}]
node at (-0.9,2) {$X$}
node at (-0.35,2) {$Y$}
node at (0.15,2) {$Z$}
node at (0.65,2) {$W$};
\end{ctikzpicture}
\\
&\;\;=\;\;
\begin{ctikzpicture}[font=\small]
\node[state] (omega) at (0,0) {\;\;\;\;\;\;};
\node[arrow box] (a) at (-0.9,0.85) {};
\node[copier] (c3) at (0.4,0.25) {};
\draw (omega) ++ (-0.4,0) to (-0.4,1.2);
\draw (omega) to (0,1.2);
\draw (omega) ++ (0.4,0) to (c3);
\draw (a) to (-0.9,1.2);
\draw (c3) .. controls (-0.65,0.4) and (-0.9,0.4) .. (a.south);
\draw (c3) to[out=30,in=-90] (0.55,0.55) to (0.55,1.2);
\path[font=\normalsize,
execute at begin node=\everymath{\scriptstyle}]
node at (-0.9,1.35) {$X$}
node at (-0.4,1.35) {$Y$}
node at (0,1.35) {$Z$}
node at (0.55,1.35) {$W$};
\end{ctikzpicture}
\;\;=\;\;
\begin{ctikzpicture}[font=\small]
\node[state] (omega) at (0.1,-0.8) {};
\node[arrow box] (b) at (-0.2,-0.1) {\;\;\;\;\;\;};
\node[arrow box] (a) at (-0.9,0.85) {};
\node[copier] (c3) at (0.4,0.25) {};
\node[copier] (c) at (0.1,-0.6) {};
\draw (omega) to (c);
\draw (c) to[out=165,in=-90] (b);
\draw (c) to[out=15,in=-90] (0.4,-0.3) to (c3);
\draw (a) to (-0.9,1.2);
\draw (c3) .. controls (-0.65,0.4) and (-0.9,0.4) .. (a.south);
\draw (c3) to[out=30,in=-90] (0.55,0.55) to (0.55,1.2);
\draw (b.north) ++(-0.2,0) to (-0.4,1.2);
\draw (b.north) ++(0.2,0) to (0,1.2);
\path[font=\normalsize,
execute at begin node=\everymath{\scriptstyle}]
node at (-0.9,1.35) {$X$}
node at (-0.4,1.35) {$Y$}
node at (0,1.35) {$Z$}
node at (0.55,1.35) {$W$};
\end{ctikzpicture}
\;\;=\;\;
\begin{ctikzpicture}[font=\small]
\node[state] (omega) at (0,0) {};
\node[copier] (c) at (0,0.2) {};
\node[arrow box] (a) at (-0.75,0.8) {};
\node[arrow box] (b) at (0,0.8) {\;\;\;\;\;};
\draw (omega) to (c);
\draw (c) to[out=165,in=-90] (a);
\draw (a) to (-0.75,1.3);
\draw (c) to (b);
\draw ([xshiftu=-0.2]b.north) to (-0.2,1.3);
\draw ([xshiftu=0.2]b.north) to (0.2,1.3);
\draw (c) to[out=15,in=-90] (0.6,0.55)
to (0.6,1.3);
\path[font=\normalsize,
execute at begin node=\everymath{\scriptstyle}]
node at (-0.75,1.45) {$X$}
node at (-0.2,1.45) {$Y$}
node at (0.2,1.45) {$Z$}
node at (0.6,1.45) {$W$};
\end{ctikzpicture}
\end{align*}
This proves $\cind{X}{Y\otimes Z}{W}$ if and only if $\cind{X}{Y}{Z\otimes W}$.
\end{proof}

The four properties from the graphoid axioms
are essential in reasoning of conditional independence
with DAGs or Bayesian networks~\cite{VermaP1988,GeigerVP1990}.

Conditional independence has been studied categorically
in~\cite{Simpson2017}. There, a categorical notion of
\emph{(conditional) independence structure} is introduced,
generalising algebraic axiomatisations of conditional independence,
such as with graphoids and separoids~\cite{Dawid2001}.  We leave it to
future work to precisely relate our approach to conditional
probability to Simpson's categorical framework.

\section{Beyond causal channels}\label{sec:beyondcausal}
\tikzextname{s_beyond_causal}

All CD-categories $\catC$ that we have considered so far are affine in
the sense that all arrows $f\colon X\to Y$ are causal: $\ground\circ
f=\ground$.  We now drop the affineness, in order to enlarge our
category to include `non-causal' arrows,
which enables us to have new notions such as scalars and effects.
Essentially, we lose nothing
by this change: all the arguments so far can still be applied to the
subcategory $\Caus(\catC)\subseteq\catC$ containing all the objects
and causal arrows.  The category $\Caus(\catC)$ is an affine CD-category,
inheriting the monoidal
structure $(\otimes,I)$
and the comonoid structures $(\copier,\ground)$
from $\catC$.

Recall that \emph{channels} in $\catC$
are causal arrows, \textit{i.e.}\ arrows in $\Caus(\catC)$.
\emph{States} are channels of the form $\sigma\colon I\to X$.
We call endomaps $I\to I$ on the tensor
unit \emph{scalars}.  The set $\catC(I,I)$ of scalars forms a monoid
via the composition $s\cdot t=s\circ t$ and $1=\id_I$.  The
monoid of scalars is always commutative
--- in fact, this is the case for
any monoidal category, see \textit{e.g.}~\cite[\S3.2]{AbramskyC2009}.
In string diagram scalars are written as
$\vcenter{\hbox{\begin{tikzpicture}[font=\small] \node[scalar] (s) at
      (-1,0.6) {$s$};
\end{tikzpicture}}}$ or simply as $s$.
We can multiply scalars $s$ to any arrows $f\colon X\to Y$
by the parallel composition, or
diagrammatically by juxtaposition:
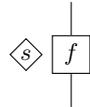
\[
\vcenter{\hbox{%
\begin{tikzpicture}[font=\small]
\node[arrow box] (c) at (0,0) {$f$};
\node[scalar] at (-0.6,0) {$s$};
\draw (c) -- (0,0.7);
\draw (c) -- (0,-0.7);
\end{tikzpicture}}}
\]

We call an arrow $\sigma\colon I\to X$ is \emph{normalisable} if the
scalar $\ground \circ \sigma\colon I\to I$ is (multiplicatively)
invertible.  In that case we can normalise $\sigma$ into a proper
state as follows.
\[
\nrm(\sigma)\coloneqq
\;\;
\vcenter{\hbox{%
\begin{tikzpicture}[font=\small]
\node[state] (s) at (-1.45,0) {$\sigma$};
\node[discarder] (d) at ([yshiftu=0.2]s) {};
\node[state] (s2) at (0,0) {$\sigma$};
\draw (s) -- (d);
\draw (s2) -- +(0,0.5);
\path[font=\normalsize,text height=1.5ex,text depth=0.25ex]
(-2,-0.1) node {$\biggl($}
(-0.75,-0.1) node {$\biggr)^{\!\!-1}$};
\end{tikzpicture}}}
\]

\emph{Effects} in $\catC$ are arrows of the form $p\colon X\to I$;
they correspond to observables, with predicates as special case.
Diagrammatically they are written as on the left below.
\[
\vcenter{\hbox{%
\begin{tikzpicture}[font=\small]
\node[state,hflip] (p) at (0,0) {$p$};
\coordinate (X) at (0,-0.4);
\draw (X) to (p);
\end{tikzpicture}}}
\qquad\qquad\qquad\qquad
\sigma\models p
\;\coloneqq\;\;
\vcenter{\hbox{%
\begin{tikzpicture}[font=\small]
\node[state] (s) at (0,0) {$\sigma$};
\node[state,hflip] (p) at (0,0.2) {$p$};
\draw (s) to (p);
\end{tikzpicture}}}
\]
On the right
the \emph{validity} $\sigma\models p$
of a state $\sigma\colon I\to X$ and a effect $p\colon X\to I$
is defined. It is the scalar given by composition.
Note that effects are not causal in general;
by definition, only discarders $\ground$ are causal ones.
States $\sigma\colon I\to X$ can be \emph{conditioned}
by effects $p\colon X\to I$ via normalisation, as follows.
\[
\sigma|_{p}
\coloneqq
\nrm
\Biggl(\;
\vcenter{\hbox{%
\begin{tikzpicture}[font=\small]
\node[state] (omega) at (0,0) {$\sigma$};
\node[copier] (copier) at (0,0.3) {};
\node[state,hflip] (p) at (-0.4,0.6) {$p$};
\coordinate (X) at (0.4,1.3) {};
\draw (omega) to (copier);
\draw (copier) to[out=165,in=-90] (p);
\draw (copier) to[out=30,in=-90] (X);
\end{tikzpicture}}}
\;\;\Biggr)
\;\;=\;\;
\vcenter{\hbox{%
\begin{tikzpicture}[font=\small]
\node[state] (s) at (-1.82,0.19) {$\sigma$};
\node[state,hflip] (p) at ([yshiftu=0.2]s) {$p$};
\draw (s) -- (p);
\path[font=\normalsize,text height=1.5ex,text depth=0.25ex]
(-2.37,0.29) node {$\Biggl($}
(-1.12,0.29) node {$\Biggr)^{\!\!-1}$};
\node[state] (omega) at (0,0) {$\sigma$};
\node[copier] (copier) at (0,0.3) {};
\node[state,hflip] (p) at (-0.4,0.6) {$p$};
\coordinate (X) at (0.4,1.3) {};
\draw (omega) to (copier);
\draw (copier) to[out=165,in=-90] (p);
\draw (copier) to[out=30,in=-90] (X);
\end{tikzpicture}}}
\]
The conditional state $\sigma|_p$ is defined if
the validity $\sigma\models p$ is invertible.
Conditioning $\sigma|_p$ generalises
conditional probability $\Pr(A|B)$
given \emph{event} $B$,
see Example~\ref{ex:CD-cat} below.
At the end of the section
we will explain how conditioning
and disintegration are related.

Recall, from Examples~\ref{ex:KlD} and~\ref{ex:KlG},
that our previous examples $\Kl(\Dst)$ and $\Kl(\Giry)$
are both affine.
We give two non-affine CD-categories that have
$\Kl(\Dst)$ and $\Kl(\Giry)$ as subcategories, respectively.

\begin{example}
For discrete probability,
we use
\emph{multisets} (or \emph{unnormalised distributions}) over
nonnegative real numbers $\pRR=[0,\infty)$,
such as
\[
1\ket{x}+ 0.5\ket{y} + 3\ket{z}
\qquad
\text{on a set}
\;
X=\{x,y,z,\dotsc\}
\]
We denote by $\Mlt(X)$
the set of multisets over $\pRR$ on $X$.
More formally:
\[
\Mlt(X)
=\set{
\phi\colon X\to \pRR
| \text{$\phi$ has finite support}}
\enspace.
\]
It extends to a commutative monad $\Mlt\colon \Set\to\Set$,
see~\cite{CoumansJ13}.
In a similar way to the distribution monad $\Dst$,
we can check that the Kleisli category $\Kl(\Mlt)$ is a CD-category.
For a Kleisli map $f\colon X\to\Mlt(Y)$,
causality $\ground\circ f=\ground$ amounts to
the condition $\sum_y f(x)(y)=1$ for all $x\in X$.
It is thus easy to see that $\Caus(\Kl(\Mlt))\cong \Kl(\Dst)$.
In fact, the distribution monad $\Dst$ can be obtained
from $\Mlt$ as its \emph{affine submonad}, see~\cite{Jacobs17}.

An effect $p\colon X \rightarrow 1$ in $\Kl(\Mlt)$ is a function $p
\colon X \rightarrow \pRR$. Its validity $\sigma\models p$ in a state
$\omega$ is given by the expected value $\sum_{x}\sigma(x)\cdot
p(x)$. The state $\sigma|_{p}$ updated with `evidence' $p$ is defined
as $\sigma|_{p}(x) = \frac{\sigma(x)\cdot p(x)}{\sigma\models p}$.
\end{example}

\begin{example}
\label{ex:CD-cat}
For general, measure-theoretic probability,
we use \emph{s-finite kernels} between measurable spaces
\cite{Kallenberg2017,Staton2017}.
Let $X$ and $Y$ be measurable spaces.
A function $f\colon X\times \Sigma_Y\to[0,\infty]$
is called a \emph{kernel} from $X$ to $Y$ if
\begin{itemize}
\item $f(x,-)\colon \Sigma_Y\to [0,\infty]$ is a measure for each $X$; and
\item $f(-,B)\colon X\to [0,\infty]$ is measurable\footnote{%
The $\sigma$-algebra on $[0,\infty]$
is the standard one generated by
$\set{\infty}$ and measurable subsets of $\pRR$.
} for each $B\in\Sigma_Y$.
\end{itemize}
We write $f\colon X\krnto Y$ when $f$ is a kernel from $X$ to $Y$.
A \emph{probability kernel} is a kernel $f\colon X\krnto Y$
with $f(x,Y)=1$ for all $x\in X$.
A kernel $f\colon X\krnto Y$ is \emph{finite} if
there exists $r\in[0,\infty)$ such that for all $x\in X$, $f(x,Y)\le r$.
(Note that it must be `uniformly' finite.)
A kernel $f\colon X\krnto Y$ is \emph{s-finite} if
$f=\sum_n f_n$ for some countable family $(f_n\colon X\to Y)_{n\in\NN}$ of finite kernels.

For two s-finite kernels $f\colon X\krnto Y$
and $g\colon Y\krnto Z$, we define the (sequential) composite
$g\circ f\colon X\krnto Z$
by
\[
(g\circ f)(x,C)=\int_Y g(y,C) \, f(x,\dd y)
\]
for $x\in X$ and $C\in\Sigma_Z$.
There are identity kernels $\eta_X\colon X\krnto X$
given by $\eta_X(x,A)=\indic{A}(x)$.
With these data, measurable spaces and s-finite kernels form a category,
which we denote by $\sfKrn$.
There is a monoidal structure on $\sfKrn$.
For measurable spaces $X,Y$ we define
the tensor product $X\otimes Y=X\times Y$ to be the cartesian product of measurable spaces.
The tensor unit $I=1$ is the singleton space.
For s-finite kernels $f\colon X\krnto Y$ and $g\colon Z\krnto W$,
we define $f\otimes g\colon X\times Z\krnto Y\times W$ by
\begin{align*}
(f\otimes g)((x,z),E)
&=
\int_{Y}
\paren[\Big]{
\int_{W}
\indic{E}(y,w)
\,g(z,\dd w)
}
f(x,\dd y)
\\
&=
\int_{W}
\paren[\Big]{
\int_{Y}
\indic{E}(y,w)
\,f(x,\dd y)
}
g(z,\dd w)
\end{align*}
for $x\in X, z\in Z, E\in\Sigma_{Y\times W}$.  The latter equality
holds by the Fubini-Tonelli theorem for s-finite measures.  These make
the category $\sfKrn$ symmetric monoidal.  Finally, for each
measurable space $X$ there is a `copier' $\copier\colon X\krnto
X\times X$ and a `discarder' $\ground\colon X\krnto 1$, given by
$\copier(x, E)=\indic{E}(x,x)$ and $\ground(x, 1)=1$,
so that $\sfKrn$ is a CD-category.
For more technical details we refer to \cite{Kallenberg2017,Staton2017}.

Note that an s-finite kernel $f\colon X\krnto Y$ is causal
if and only if it is a probability kernel,
which is nothing but a Kleisli map $X\to \Giry(Y)$ for the Giry monad.
Therefore the causal subcategory of $\sfKrn$ is the Kleisli category
of the Giry monad: $\Caus(\sfKrn)\cong \Kl(\Giry)$.
In particular, states in $\sfKrn$ are probability measures $\sigma\in\Giry(X)$.

An effect $p\colon X\krnto 1$ in $\sfKrn$,
\textit{i.e.}\ an s-finite kernel $p\colon X\times \Sigma_1\to [0,\infty]$,
can be identified with a measurable function $p\colon X\to [0,\infty]$.
The validity $\sigma\models p$ is then the integral $\int_X p(x)\,\sigma(\dd x)$,
defined in $[0,\infty]$.
The conditional state $\sigma|_{p}\in\Giry(X)$ is defined by:
\[
\sigma|_{p}(A) = \frac{\int_A p(x)\,\sigma(\dd x)}{\sigma\models p}
\]
for $A\in\Sigma_X$, when
the validity $\sigma\models p$ is neither $0$ nor $\infty$.
In particular, for any `event' $B\in \Sigma_X$,
the obvious
indicator function $\indic{B}\colon X\to [0,\infty]$
is an effect.
Then conditioning yields a new state
on $X$:
\[
\sigma|_{\indic{B}}(A)
= \frac{\int_A \indic{B}(x)\,\sigma(\dd x)}{\sigma\models \indic{B}}
= \frac{\sigma(A\cap B)}{\sigma(B)}
\quad
\text{for $A\in\Sigma_X$.}
\]
This amounts to
conditional probability $\Pr(A|B)=\Pr(A,B)/\Pr(B)$
given event $B$.
\end{example}

We need to generalise definitions
from the previous sections
in the non-affine/causal setting.
Here we make only a minimal generalisation
that is required in the next section.
For example, we still restrict ourselves to
disintegration of (causal) states,
although in the literature,
the notion of disintegration exists
also for non-probability (i.e.\ non-causal)
measures and even for non-finite measures,
see e.g.\ \cite[Definition~1]{ChangP1997}.
We leave such generalisations to future work.

Let $\sigma\colon I\to X$ be a state.
We define $\sigma$-almost equality
$f\equpto[\sigma]g$ between
arbitrary arrows $f,g\colon X\to Y$
in the same way as Definition~\ref{def:almost-equal}.
Similarly,
strong $\sigma$-almost equality
between arbitrary arrows
is defined as in Definition~\ref{def:str-almost-eq}
(here $\omega$ still ranges over states).

Disintegrations of a joint state $\omega\colon I\to X\otimes Y$
are defined as in Definition~\ref{def:disintegration},
except that an arrow $c_1\colon X\to Y$ (or $c_2\colon Y\to X$)
is not necessarily causal.
Nevertheless,
disintegrations of a state
are \emph{almost causal} in the following sense.

\begin{definition}
Let $\sigma\colon I\to X$ be a state.  We say that an arrow $c\colon
X\to Y$ is \emph{$\sigma$-almost causal} if $\ground\circ
c\equpto[\sigma]\ground$.
\end{definition}

\noindent
Then the following is immediate from the definition.

\begin{proposition}
Let $c_1\colon X\to Y$ be a disintegration
of a state $\omega\colon I\to X\otimes Y$.
Then $c_1$ is $\omega_1$-almost causal,
where $\omega_1\colon I\to X$ is
the first marginal of $\omega$.
\qed
\end{proposition}

\begin{example}
In $\sfKrn$,
a s-finite kernel $f\colon X\krnto Y$
is $\sigma$-almost causal if and only if
$f(x,Y)=1$ for $\sigma$-almost all $x\in X$.
In that case, we can find
a probability kernel $f'\colon X\krnto Y$
such that $f\equpto[\sigma]f'$,
by tweaking $f$ in the obvious way.
From this it follows that
a disintegration of
a joint probability measure $\omega\in\Giry(X\times Y)$
exists in $\sfKrn$
if and only if it exists in $\Kl(\Giry)$.

By Proposition~\ref{prop:KlG-eq-str},
we can strengthen almost equality between
channels in $\sfKrn$.
It is easy to see that
equality strengthening is valid also for
almost causal maps:
if $f,g\colon X\krnto Y$
are $\sigma$-almost causal maps
with $f\equpto[\sigma] g$,
then $f,g$ are strongly $\sigma$-almost equal.
(It is not clear whether equality strengthening
is valid for arbitrary maps in $\sfKrn$,
but we will not need such a general result in this paper.)
\end{example}

In the remainder of the section,
we present a basic relationship
between disintegration
and conditioning~$\sigma|_p$,
introduced above.
We assume a
joint state $\omega$ with its two disintegrations $c_{1}$ and $c_{2}$ in:
\begin{equation}
\label{eqn:disintegrationforcrossover}
\vcenter{\hbox{%
\begin{tikzpicture}[font=\small]
\node[state] (omega) at (0.25,0) {$\;\omega\;$};
\node[copier] (copier) at (0,0.4) {};
\node[arrow box] (c) at (0.5,0.95) {$c_1$};
\coordinate (X) at (-0.5,1.5);
\coordinate (Y) at (0.5,1.5);
\coordinate (omega1) at ([xshiftu=-0.25]omega);
\coordinate (omega2) at ([xshiftu=0.25]omega);
\node[discarder] (d) at ([yshiftu=0.2]omega2) {};
\draw (omega1) to (copier);
\draw (omega2) to (d);
\draw (copier) to[out=150,in=-90] (X);
\draw (copier) to[out=15,in=-90] (c);
\draw (c) to (Y);
\path[scriptstyle]
node at (-0.65,1.5) {$X$}
node at (0.65,1.5) {$Y$};
\end{tikzpicture}}}
\qquad=\qquad
\vcenter{\hbox{%
\begin{tikzpicture}[font=\small]
\node[state] (omega) at (0,0) {\;$\omega$\;};
\coordinate (X) at (-0.25,0.55) {};
\coordinate (Y) at (0.25,0.55) {};
\draw (omega) ++(-0.25, 0) to (X);
\draw (omega) ++(0.25, 0) to (Y);
\path[scriptstyle]
node at (-0.4,0.55) {$X$}
node at (0.4,0.55) {$Y$};
\end{tikzpicture}}}
\qquad=\qquad
\vcenter{\hbox{%
\begin{tikzpicture}[font=\small]
\node[state] (omega) at (-0.25,0) {$\;\omega\;$};
\coordinate (omegaX) at ([xshiftu=-0.25]omega);
\coordinate (omegaY) at ([xshiftu=0.25]omega);
\node[discarder] (d) at ([yshiftu=0.2]omegaX) {};
\node[copier] (copier) at (0,0.4) {};
\node[arrow box] (c) at (-0.5,0.95) {$c_2$};
\coordinate (X) at (-0.5,1.5);
\coordinate (Y) at (0.5,1.5);
\draw (omegaX) to (d);
\draw (omegaY) to (copier);
\draw (copier) to[out=165,in=-90] (c);
\draw (copier) to[out=30,in=-90] (Y);
\draw (c) to (X);
\path[scriptstyle]
node at (-0.65,1.5) {$X$}
node at (0.65,1.5) {$Y$};
\end{tikzpicture}}}
\end{equation}

\noindent We write $\omega_{1}$ and $\omega_{2}$ for the first
and second marginals of $\omega$.
(Thus the equations~\eqref{eqn:disintegrationforcrossover} above
are the same as \eqref{eqn:disintegrations}.)

Let $q$ be an effect on $Y$. It can be extended to an effect
$\one\otimes q$ on $X\otimes Y$, where:
\[
\one\otimes q
\quad\coloneqq\quad
\vcenter{\hbox{%
\begin{tikzpicture}[font=\small]
\node[discarder] (d) at (-0.8,0) {};
\node[state,hflip] (q) at (0,0) {$q$};
\coordinate (X) at (-0.8,-0.3);
\coordinate (Y) at (0,-0.3);
\draw (X) to (d);
\draw (Y) to (q);
\path[scriptstyle]
node at (-0.65,-0.3) {$X$}
node at (0.15,-0.3) {$Y$};
\end{tikzpicture}}}
\]

\noindent Then we can form the conditioned state
$\omega|_{\one\otimes q}$. In a next step we take its first marginal,
written as $\big(\omega|_{\one\otimes q}\big)_{1}$. It turns out that,
in general, this first marginal is different from the original first
marginal $\omega_{1}$, even though the effect $q$ only applies to
the second coordinate. This is called `crossover influence'
in~\cite{JacobsZ17}.  It happens when the state $\omega$ is
`entwined', that is, when its two coordinates are correlated.

A fundamental result in this context is that this crossover influence
can also be captured via the channels $c_{1}, c_{2}$ that are
extracted from $\omega$ via disintegrations. This works via effect
transformation $c^*(p)\coloneqq p\circ c$ and
state transformation $c_{*}(\sigma)\coloneqq c\circ\sigma$
along a channel.

\begin{theorem}
\label{thm:crossover}
In the above setting, assuming that the relevant conditioned states exist,
there are equalities of states:
\begin{equation}
\label{eqn:crossover}
\begin{array}{rcccl}
\omega_{1}|_{c_{1}^{*}(q)}
& = &
\big(\omega|_{\one\otimes q}\big)_{1}
& = &
\big(c_{2}\big)_{*}(\omega_{2}|_{q}).
\end{array}
\end{equation}
\end{theorem}

Following~\cite{JacobsZ16} we can say that the expression on the left
in~\eqref{eqn:crossover} uses \emph{backward} inference, and the one
on the right uses \emph{forward} inference.

\begin{proof}
We first note that the state in the middle of~\eqref{eqn:crossover}
is the first marginal of:
\[ \vcenter{\hbox{%
\begin{tikzpicture}[font=\small]
\node[state] (oma) at (-2,0) {$\hspace*{0.4em}\omega\hspace*{0.4em}$};
\coordinate (oma1) at ([xshiftu=-0.4]oma);
\coordinate (oma2) at ([xshiftu=0.4]oma);
\node[discarder] (1a) at ([yshiftu=0.2]oma1) {};
\node[state,hflip,scale=0.75] (qa) at ([yshiftu=0.2]oma2) {$q$};
\draw (oma1) -- (1a);
\draw (oma2) -- (qa);
\path[font=\normalsize,text height=1.5ex,text depth=0.25ex]
(-2.9,0) node {$\Biggl($}
(-1,0) node {$\Biggr)^{\!\!-1}$};
\node[state] (omb) at (0.5,0) {$\hspace*{0.5em}\omega\hspace*{0.5em}$};
\coordinate (ombX) at ([xshiftu=-0.5]omb);
\coordinate (ombY) at ([xshiftu=0.5]omb);
\node[copier] (copierX) at ([yshiftu=0.2]ombX) {};
\node[copier] (copierY) at ([yshiftu=0.2]ombY) {};
\node[discarder] (1b) at ([xshiftu=-0.3,yshiftu=0.2]copierX) {};
\node[state,hflip,scale=0.75] (qb) at ([xshiftu=-0.3,yshiftu=0.2]copierY) {$q$};
\coordinate (X) at ([xshiftu=0.3,yshiftu=1.0]copierX) {};
\coordinate (Y) at ([xshiftu=0.3,yshiftu=1.0]copierY) {};
\draw (ombX) to (copierX);
\draw (ombY) to (copierY);
\draw (copierX) to[out=165,in=-90] (1b);
\draw (copierX) to[out=30,in=-90] (X);
\draw (copierY) to[out=165,in=-90] (qb);
\draw (copierY) to[out=30,in=-90] (Y);
\end{tikzpicture}}}
\]

\noindent Hence:
\begin{equation}
\label{eqn:crossover1}
\big(\omega|_{\one\otimes q}\big)_{1}
\quad=\quad
\vcenter{\hbox{%
\begin{tikzpicture}[font=\small]
\node[state] (oma) at (-2,0) {$\hspace*{0.4em}\omega\hspace*{0.4em}$};
\coordinate (oma1) at ([xshiftu=-0.4]oma);
\coordinate (oma2) at ([xshiftu=0.4]oma);
\node[discarder] (1a) at ([yshiftu=0.2]oma1) {};
\node[state,hflip,scale=0.75] (qa) at ([yshiftu=0.2]oma2) {$q$};
\draw (oma1) -- (1a);
\draw (oma2) -- (qa);
\path[font=\normalsize,text height=1.5ex,text depth=0.25ex]
(-2.9,0) node {$\Biggl($}
(-1,0) node {$\Biggr)^{\!\!-1}$};
\node[state] (omb) at (0.0,0) {$\hspace*{0.4em}\omega\hspace*{0.4em}$};
\coordinate (ombX) at ([xshiftu=-0.4]omb);
\coordinate (ombY) at ([xshiftu=0.4]omb);
\node[state,hflip,scale=0.75] (qb) at ([yshiftu=0.2]ombY) {$q$};
\coordinate (X) at ([yshiftu=0.6]ombX) {};
\draw (ombX) to (X);
\draw (ombY) to (qb);
\end{tikzpicture}}}
\end{equation}

\noindent We note that the above scalar (that is inverted) can also be
obtained as:
\begin{equation}
\label{eqn:crossover2}
\vcenter{\hbox{%
\begin{tikzpicture}[font=\small]
\node[state] (omega) at (0,0) {$\hspace*{0.4em}\omega\hspace*{0.4em}$};
\coordinate (omegaX) at ([xshiftu=-0.4]omega);
\coordinate (omegaY) at ([xshiftu=0.4]omega);
\node[arrow box] (c) at ([yshiftu=0.5]omegaX) {$c_1$};
\node[discarder] (d) at ([yshiftu=0.2]omegaY) {};
\node[state,hflip,scale=0.75] (q) at ([yshiftu=0.5]c) {$q$};
\draw (omegaX) to (c);
\draw (omegaY) to (d);
\draw (c) to (q);
\end{tikzpicture}}}
\quad
=
\quad
\vcenter{\hbox{%
\begin{tikzpicture}[font=\small]
\node[state] (omega) at (0,0) {$\hspace*{0.4em}\omega\hspace*{0.4em}$};
\coordinate (omegaX) at ([xshiftu=-0.4]omega);
\coordinate (omegaY) at ([xshiftu=0.4]omega);
\node[copier] (copier) at ([yshiftu=0.3]omegaX) {};
\node[arrow box] (c) at ([xshiftu=-0.0,yshiftu=1.0]omegaY) {$c_1$};
\node[discarder] (dX) at ([yshiftu=1.5]omegaX) {};
\node[discarder] (dY) at ([yshiftu=0.2]omegaY) {};
\node[state,hflip,scale=0.75] (q) at ([yshiftu=0.5]c) {$q$};
\draw (omegaX) to (copier);
\draw (omegaY) to (dY);
\draw (copier) to[out=150,in=-90] (dX);
\draw (copier) to[out=15,in=-90] (c);
\draw (c) to (q);
\end{tikzpicture}}}
\quad
\smash{\stackrel{\eqref{eqn:disintegrationforcrossover}}{=}}
\quad
\vcenter{\hbox{%
\begin{tikzpicture}[font=\small]
\node[state] (omega) at (-2,0) {$\hspace*{0.4em}\omega\hspace*{0.4em}$};
\coordinate (omegaX) at ([xshiftu=-0.4]omega);
\coordinate (omegaY) at ([xshiftu=0.4]omega);
\node[discarder] (dX) at ([yshiftu=0.2]omegaX) {};
\node[state,hflip,scale=0.75] (q) at ([yshiftu=0.2]omegaY) {$q$};
\draw (omegaX) -- (dX);
\draw (omegaY) -- (q);
\end{tikzpicture}}}
\end{equation}

\noindent Hence we can prove the equation on the left
in~\eqref{eqn:crossover}:
\[
\omega_{1}|_{c_{1}^{*}(q)}
\;\;
=
\quad
\vcenter{\hbox{%
\begin{tikzpicture}[font=\small]
\node[state] (omega) at (-2,0) {$\hspace*{0.4em}\omega\hspace*{0.4em}$};
\coordinate (omegaX) at ([xshiftu=-0.4]omega);
\coordinate (omegaY) at ([xshiftu=0.4]omega);
\node[arrow box] (c) at ([yshiftu=0.5]omegaX) {$c_1$};
\node[discarder] (d) at ([yshiftu=0.2]omegaY) {};
\node[state,hflip,scale=0.75] (q) at ([yshiftu=0.5]c) {$q$};
\draw (omegaX) to (c);
\draw (omegaY) to (d);
\draw (c) to (q);
\path[font=\normalsize,text height=1.5ex,text depth=0.25ex]
(-3.0,0.4) node {$\Biggl($}
(-1,0.4) node {$\Biggr)^{\!\!-1}$};
\node[state] (omegab) at (0.2,0) {$\hspace*{0.4em}\omega\hspace*{0.4em}$};
\coordinate (omegabX) at ([xshiftu=-0.4]omegab);
\coordinate (omegabY) at ([xshiftu=0.4]omegab);
\node[copier] (copier) at ([yshiftu=0.3]omegabX) {};
\node[arrow box] (cb) at ([xshiftu=-0.2,yshiftu=0.9]omegabY) {$c_1$};
\node[discarder] (dY) at ([yshiftu=0.2]omegabY) {};
\node[state,hflip,scale=0.75] (qb) at ([yshiftu=0.5]cb) {$q$};
\coordinate (X) at ([yshiftu=1.5]omegabX) {};
\draw (omegabX) to (copier);
\draw (omegabY) to (dY);
\draw (copier) to[out=150,in=-90] (X);
\draw (copier) to[out=15,in=-90] (cb);
\draw (cb) to (qb);
\end{tikzpicture}}}
\quad
\smash{\stackrel{\eqref{eqn:disintegrationforcrossover},\eqref{eqn:crossover2}}{=}}
\quad
\vcenter{\hbox{%
\begin{tikzpicture}[font=\small]
\node[state] (oma) at (-2,0) {$\hspace*{0.4em}\omega\hspace*{0.4em}$};
\coordinate (oma1) at ([xshiftu=-0.4]oma);
\coordinate (oma2) at ([xshiftu=0.4]oma);
\node[discarder] (1a) at ([yshiftu=0.2]oma1) {};
\node[state,hflip,scale=0.75] (qa) at ([yshiftu=0.2]oma2) {$q$};
\draw (oma1) -- (1a);
\draw (oma2) -- (qa);
\path[font=\normalsize,text height=1.5ex,text depth=0.25ex]
(-2.9,0) node {$\Biggl($}
(-1,0) node {$\Biggr)^{\!\!-1}$};
\node[state] (omb) at (0.0,0) {$\hspace*{0.4em}\omega\hspace*{0.4em}$};
\coordinate (ombX) at ([xshiftu=-0.4]omb);
\coordinate (ombY) at ([xshiftu=0.4]omb);
\node[state,hflip,scale=0.75] (qb) at ([yshiftu=0.2]ombY) {$q$};
\coordinate (X) at ([yshiftu=0.6]ombX) {};
\draw (ombX) to (X);
\draw (ombY) to (qb);
\end{tikzpicture}}}
\]

\noindent In a similar way we prove the equation on the right
in~\eqref{eqn:crossover}, since $\big(c_{2}\big)_{*}(\omega_{2}|_{q})$ equals:
\[
\vcenter{\hbox{%
\begin{tikzpicture}[font=\small]
\node[state] (oma) at (-2,0) {$\hspace*{0.4em}\omega\hspace*{0.4em}$};
\coordinate (omaX) at ([xshiftu=-0.4]oma);
\coordinate (omaY) at ([xshiftu=0.4]oma);
\node[discarder] (daX) at ([yshiftu=0.2]omaX) {};
\node[state,hflip,scale=0.75] (qa) at ([yshiftu=0.2]omaY) {$q$};
\draw (omaX) -- (daX);
\draw (omaY) -- (qa);
\path[font=\normalsize,text height=1.5ex,text depth=0.25ex]
(-2.9,0) node {$\Biggl($}
(-1,0) node {$\Biggr)^{\!\!-1}$};
\node[state] (omb) at (0.0,0) {$\hspace*{0.4em}\omega\hspace*{0.4em}$};
\coordinate (ombX) at ([xshiftu=-0.4]omb);
\coordinate (ombY) at ([xshiftu=0.4]omb);
\node[discarder] (dbX) at ([yshiftu=0.2]ombX) {};
\node[copier] (copier) at ([yshiftu=0.3]ombY) {};
\node[state,hflip,scale=0.75] (qb) at ([xshiftu=-0.3,yshiftu=0.6]ombY) {$q$};
\node[arrow box] (cb) at ([xshiftu=0.3,yshiftu=1.5]ombY) {$c_2$};
\coordinate (Y) at ([yshiftu=0.5]cb) {};
\draw (ombX) to (dbX);
\draw (ombY) to (copier);
\draw (copier) to[out=150,in=-90] (qb);
\draw (copier) to[out=15,in=-90] (cb);
\draw (cb) to (Y);
\end{tikzpicture}}}
\;
=
\;
\vcenter{\hbox{%
\begin{tikzpicture}[font=\small]
\node[state] (oma) at (-2,0.5) {$\hspace*{0.4em}\omega\hspace*{0.4em}$};
\coordinate (omaX) at ([xshiftu=-0.4]oma);
\coordinate (omaY) at ([xshiftu=0.4]oma);
\node[discarder] (1a) at ([yshiftu=0.2]omaX) {};
\node[state,hflip,scale=0.75] (qa) at ([yshiftu=0.2]omaY) {$q$};
\draw (omaX) -- (1a);
\draw (omaY) -- (qa);
\path[font=\normalsize,text height=1.5ex,text depth=0.25ex]
(-2.9,0.5) node {$\Biggl($}
(-1,0.5) node {$\Biggr)^{\!\!-1}$};
\node[state] (omb) at (0.1,0) {$\hspace*{0.4em}\omega\hspace*{0.4em}$};
\coordinate (ombX) at ([xshiftu=-0.4]omb);
\coordinate (ombY) at ([xshiftu=0.4]omb);
\node[discarder] (dbX) at ([yshiftu=0.2]ombX) {};
\node[copier] (copier) at ([yshiftu=0.3]ombY) {};
\node[state,hflip,scale=0.75] (qb) at ([yshiftu=1.5]ombY) {$q$};
\node[arrow box] (cb) at ([yshiftu=1.1]ombX) {$c_2$};
\coordinate (Y) at ([yshiftu=0.7]cb) {};
\draw (ombX) to (dbX);
\draw (ombY) to (copier);
\draw (copier) to[out=150,in=-90] (cb);
\draw (copier) to[out=15,in=-90] (qb);
\draw (cb) to (Y);
\end{tikzpicture}}}
\;
\smash{\stackrel{\eqref{eqn:disintegrationforcrossover}}{=}}
\;
\vcenter{\hbox{%
\begin{tikzpicture}[font=\small]
\node[state] (oma) at (-2,0) {$\hspace*{0.4em}\omega\hspace*{0.4em}$};
\coordinate (oma1) at ([xshiftu=-0.4]oma);
\coordinate (oma2) at ([xshiftu=0.4]oma);
\node[discarder] (1a) at ([yshiftu=0.2]oma1) {};
\node[state,hflip,scale=0.75] (qa) at ([yshiftu=0.2]oma2) {$q$};
\draw (oma1) -- (1a);
\draw (oma2) -- (qa);
\path[font=\normalsize,text height=1.5ex,text depth=0.25ex]
(-2.9,0) node {$\Biggl($}
(-1,0) node {$\Biggr)^{\!\!-1}$};
\node[state] (omb) at (-0.1,0) {$\hspace*{0.4em}\omega\hspace*{0.4em}$};
\coordinate (ombX) at ([xshiftu=-0.4]omb);
\coordinate (ombY) at ([xshiftu=0.4]omb);
\node[state,hflip,scale=0.75] (qb) at ([yshiftu=0.2]ombY) {$q$};
\coordinate (X) at ([yshiftu=0.6]ombX) {};
\draw (ombX) to (X);
\draw (ombY) to (qb);
\end{tikzpicture}}}
\]
\end{proof}

The two equations in Theorem~\ref{eqn:crossover} will be illustrated
in the `disease and mood' example below, where a particular state
(probability) will be calculated in three different ways. For a gentle
introduction to Bayesian inference with channels, and many more
examples, we refer to~\cite{JacobsZ18}.

\begin{remark}
Another way of relating conditioning
to disintegration / Bayesian inversion was pointed
out by a reviewer of the paper:
for a state $\sigma\colon I\to X$
and an effect $p\colon X\to I$,
a conditional state $\sigma|_p\colon I\to X$
is a Bayesian inversion for $\sigma$
along $p$.
This, however, involves disintegration of non-causal maps,
which we leave to future work.
\end{remark}

\subsection*{Disease and mood example}

We describe an example of probabilistic (Bayesian)
reasoning. The setting is the following. We consider a joint state
about the occurrence and non-occurrence of a disease, written as a
two-element set $D = \{d,\no{d}\}$, jointly with the occurrence and
non-occurrence of a good mood, written as $M = \{m, \no{m}\}$, where
$\no{m}$ stands for a bad (not good) mood. The joint distribution,
called $\omega\in \Dst(M\times D)$, that we start from is of the form:
\begin{equation}
\label{eqn:DMstate}
\begin{array}{rcl}
\omega
& = &
0.05\ket{m,d} + 0.4\ket{m,\no{d}} + 0.5\ket{\no{m},d} + 0.05\ket{\no{m},\no{d}}.
\end{array}
\end{equation}

\noindent Suppose there is a test for the disease with the following
sensitivity and specificity. It is positive in 90\% of all cases of
people having the disease, and also 5\% positive for people without
the disease. Suppose the disease comes out positive. What is then the
mood? It is expected that the mood will deteriorate, since the disease
and the mood are `entwined' (correlated) in the above joint
state~\eqref{eqn:DMstate}: a high likelihood of disease corresponds to
a low mood.

The test's sensitivity and specificity is expressed as a channel
$s\colon D\rightarrow 2$, where $2 = \{t,f\}$, with:
\[ \begin{array}{rclcrcl}
c(d)
& = &
\frac{9}{10}\ket{t} + \frac{1}{10}\ket{f}
& \qquad &
c(\no{d})
& = &
\frac{1}{20}\ket{t} + \frac{19}{20}\ket{f}.
\end{array} \]

\noindent We calculate the first and second marginal of $\omega$:
\[ \begin{array}{rccclcrcccl}
\omega_{1}
& \coloneqq &
\marg{\omega}{[1,0]}
& = &
0.45\ket{m} + 0.55\ket{\no{m}}
& \qquad &
\omega_{2}
& \coloneqq &
\marg{\omega}{[0,1]}
& = &
0.55\ket{d} + 0.45\ket{\no{d}}.
\end{array} \]

\noindent The a priori probability of a positive test is obtained by
state transformation, applied to the second marginal:
\[ \begin{array}{rcl}
s_{*}\big(\omega_{2}\big)
& = &
0.518\ket{t} + 0.482\ket{f}.
\end{array} \]

As explained above, we are interested in the mood after a positive
test. We write $\tt$ for the predicate on $2 = \{t,f\}$ given by
$\tt(t) = 1$ and $\tt(f) = 0$. We can transform it to a predicate
$s^{*}(\tt)$ on $D = \{d,\no{d}\}$, namely $s^{*}(\tt)(d) =
\frac{9}{10}$ and $s^{*}(\tt)(\no{d}) = \frac{1}{20}$. We now describe
three equivalent ways to calculate the posterior mood, as
in Theorem~\ref{thm:crossover}, with predicate $q = s^{*}(\tt)$.
\begin{enumerate}
\item First we use the truth predicate $\one$ on $M = \{m,\no{m}\}$,
  which is always $1$, to extend (weaken) the predicate $s^{*}(\tt)$
  on $D$ to $\one\otimes s^{*}(\tt)$ on $M\times D$. The latter
  predicate can be used to update the joint state
  $\omega\in\Dst(M\times D)$. If we then take the first marginal we
  obtain the updated mood probability:
\[ \begin{array}{rcl}
\marg{\big(\omega\big|_{\one\otimes s^{*}(\tt)}\big)}{[1,0]}
& = &
0.126\ket{m} + 0.874\ket{\no{m}}.
\end{array} \]

\noindent Clearly, a positive test leads to a lower mood: a reduction
from $0.45$ to $0.126$. 

\item Next we extract channels $c_{1} \colon M \rightarrow D$ and
  $c_{2} \colon D \rightarrow M$ from $\omega$ via disintegration as
  in~\eqref{eqn:disintegrationforcrossover}. For instance, $c_{1}$ is
defined as:
\[ \begin{array}{rcl}
c_{1}(m)
& = &
\frac{\omega(m,d)}{\omega_{1}(m)}\ket{d} +
   \frac{\omega(m,\no{d})}{\omega_{1}(m)}\ket{\no{d}}
\hspace*{\arraycolsep}=\hspace*{\arraycolsep}
\frac{1}{9}\ket{d} + \frac{8}{9}\ket{\no{d}}
\\
c_{1}(\no{m})
& = &
\frac{\omega(\no{m},d)}{\omega_{1}(\no{m})}\ket{d} +
   \frac{\omega(\no{m},\no{d})}{\omega_{1}(\no{m})}\ket{\no{d}}
\hspace*{\arraycolsep}=\hspace*{\arraycolsep}
\frac{10}{11}\ket{d} + \frac{1}{11}\ket{\no{d}}.
\end{array} \]



\noindent We can now transform the predicate $s^{*}(\tt)$ on $D$ along
$c_{1}$ to get the predicate $c_{1}^{*}(s^{*}(\tt)) = (s_{1} \klafter
c_{1})^{*}(\tt))$ on $M$ given by: $d \mapsto 13/90$ and $\no{d}
\mapsto 181/220$. We can now use this predicate to update the a priori
`mood' marginal $\omega_{1}\in\Dst(M)$. This gives the same a
posteriori outcome as before:
\[ \begin{array}{rcl}
\omega_{1}\big|_{c_{1}^{*}(s^{*}(\tt))}
& = &
0.126\ket{m} + 0.874\ket{\no{m}}.
\end{array} \]

\item We can also update the second `disease' marginal
  $\omega_{2}\in\Dst(D)$ directly with the predicate $s^{*}(\tt)$ and
  then do state transformation along the channel $c_{2} \colon
  D\rightarrow M$. This gives the same updated mood:
\[ \begin{array}{rcl}
(c_{2})_{*}\big(\omega_{2}|_{s^{*}(\tt)}\big)
& = &
0.126\ket{m} + 0.874\ket{\no{m}}.
\end{array} \]
\end{enumerate}

\section{Disintegration via likelihoods}\label{sec:likelihood}
\tikzextname{s_disint_lik}

We continue in the setting of Section~\ref{sec:beyondcausal} in
a CD-category that is not necessarily affine.
The goal of this section is to present Theorem~\ref{thm:likelihood},
which generalises a construction of Bayesian inversions using
densities/likelihoods shown in Example~\ref{ex:inversion}.

We first introduce `likelihoods' in our setting.

\begin{definition}
\label{def:likelihood}
We say a channel $c\colon X\to Y$ is represented by a effect $\ell$ on
$X\otimes Y$ with respect to an arrow $\nu\colon I\to Y$ if
\begin{equation}
\label{eq:likelihood}
\begin{ctikzpicture}[font=\small]
\node[arrow box] (c) at (0,0) {$c$};
\draw (c) -- (0,0.7);
\draw (c) -- (0,-0.7);
\end{ctikzpicture}
\;\;=\;\;
\begin{ctikzpicture}[font=\small]
\node[state,hflip] (L) at (-1.1,0.5) {\;\,$\ell$\;\,};
\node[copier] (copier2) at (-0.45,0.2) {};
\node[state] (tau) at (-0.45,0) {$\nu$};
\coordinate (Y) at (-0.1,1.2) {} {} {};
\coordinate (X) at (-1.4,-0.7) {} {};
\draw (X) to ([xshiftu=-.3]L);
\draw (tau) to (copier2);
\draw (copier2) to[out=165,in=-90] ([xshiftu=.3]L);
\draw (copier2) to[out=30,in=-90] (Y);
\end{ctikzpicture}
\end{equation}
We call $\ell$ a \emph{likelihood relation} for
the channel $c$ with respect to $\nu$.
\end{definition}

Interpreted in the category $\sfKrn$, the definition says:
a kernel $c\colon X\krnto Y$ satisfies
\[
c(x,B)
=\int_B \ell(x,y)\,\nu(\dd y)
\]
for a kernel $\ell\colon X\times Y\krnto 1$
(identified with a measurable function $\ell\colon X\times Y\to [0,\infty]$)
and a measure $\nu\colon 1\krnto Y$.
This is basically the same as what we have
in Example~\ref{ex:inversion}, but
here $\nu$ is not necessarily the Lebesgue measure.

\begin{definition}
Let $\sigma\colon I\to X$ be a state.
An effect $p\colon X\to I$ is
\emph{$\sigma$-almost invertible} if there is an effect $q\colon X\to
I$ such that:
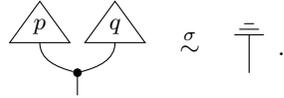
\[
\vcenter{\hbox{%
\begin{tikzpicture}[font=\small]
\node[state,hflip] (p) at (-0.5,0.4) {$p$};
\node[state,hflip] (q) at (0.5,0.4) {$q$};
\node[copier] (copier) at (0,0) {};
\coordinate (X) at (0,-0.3);
\draw (X) to (copier);
\draw (copier) to[out=165,in=-90] (p);
\draw (copier) to[out=15,in=-90] (q);
\end{tikzpicture}}}
\quad\equpto[\sigma]\quad
\vcenter{\hbox{%
\begin{tikzpicture}[font=\small]
\node[discarder] (d) at (0,0.5) {};
\coordinate (X) at (0,0);
\draw (X) to (d);
\end{tikzpicture}}}
\enspace.
\]
\end{definition}

The definition allows us to normalise an arrow $f\colon X\to Y$
into an almost causal one, as follows.
If an effect $\ground\circ f$ is $\sigma$-almost inverted by $q$,
as on the left below,
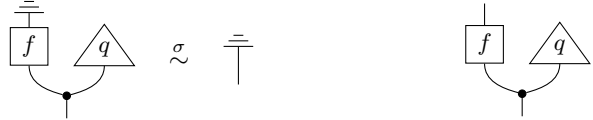
\[
\vcenter{\hbox{%
\begin{tikzpicture}[font=\small]
\node[discarder] (p) at (-0.5,1.1) {};
\node[arrow box,anchor=south] (f) at (-0.5,0.4) {$f$};
\node[state,hflip] (q) at (0.5,0.4) {$q$};
\node[copier] (copier) at (0,0) {};
\coordinate (X) at (0,-0.3);
\draw (X) to (copier);
\draw (copier) to[out=165,in=-90] (f);
\draw (f) to (p);
\draw (copier) to[out=15,in=-90] (q);
\end{tikzpicture}}}
\quad\equpto[\sigma]\quad
\vcenter{\hbox{%
\begin{tikzpicture}[font=\small]
\node[discarder] (d) at (0,0.5) {};
\coordinate (X) at (0,0);
\draw (X) to (d);
\end{tikzpicture}}}
\qquad\qquad\qquad\qquad
\vcenter{\hbox{%
\begin{tikzpicture}[font=\small]
\coordinate (Y) at (-0.5,1.2) {};
\node[arrow box,anchor=south] (f) at (-0.5,0.4) {$f$};
\node[state,hflip] (q) at (0.5,0.4) {$q$};
\node[copier] (copier) at (0,0) {};
\coordinate (X) at (0,-0.3);
\draw (X) to (copier);
\draw (copier) to[out=165,in=-90] (f);
\draw (f) to (Y);
\draw (copier) to[out=15,in=-90] (q);
\end{tikzpicture}}}
\]
then clearly
the arrow $X\to Y$ on the right is $\sigma$-almost causal.

We can now formulate and prove our main technical result.

\begin{theorem}
\label{thm:likelihood}
Let $\sigma$ be a state on $X$, and $c\colon X\to Y$
be a channel represented by a likelihood relation $\ell$ with respect to $\nu$
as in~\eqref{eq:likelihood} above.
Assume that the category admits equality strengthening
for almost causal maps,
and that the effect
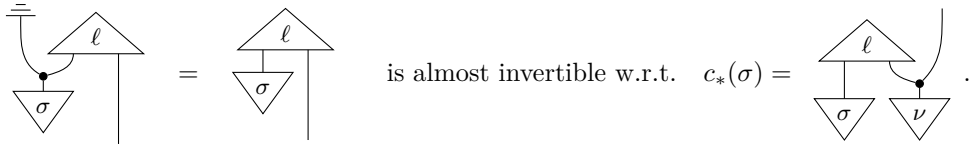
\[
\vcenter{\hbox{%
\begin{tikzpicture}[font=\small]
\node[state,hflip] (L) at (0,0) {\;\,$\ell$\;\,};
\node[copier] (copier2) at (-0.7,-0.3) {};
\node[state] (tau) at (-0.7,-0.5) {$\sigma$};
\node[discarder] (Y) at (-1,0.5) {};
\coordinate (X) at (0.3,-1.2);
\draw (X) to ([xshiftu=.3]L);
\draw (tau) to (copier2);
\draw (copier2) to[out=15,in=-90] ([xshiftu=-.3]L);
\draw (copier2) to[out=150,in=-90] (Y);
\end{tikzpicture}}}
\quad=\quad
\vcenter{\hbox{%
\begin{tikzpicture}[font=\small]
\node[state,hflip] (L) at (0,0) {\;\,$\ell$\;\,};
\node[state] (tau) at (-0.3,-0.3) {$\sigma$};
\coordinate (X) at (0.3,-1.2);
\draw (X) to ([xshiftu=.3]L);
\draw (tau) to ([xshiftu=-.3]L);
\end{tikzpicture}}}
\qquad\text{is almost invertible w.r.t.}\quad
c_*(\sigma)
=\;\;
\vcenter{\hbox{%
\begin{tikzpicture}[font=\small]
\node[state,hflip] (L) at (-1.1,0.5) {\;\,$\ell$\;\,};
\node[copier] (copier2) at (-0.4,0.2) {};
\node[state] (tau) at (-0.4,0) {$\nu$};
\node[state] (sigma) at (-1.4,0) {$\sigma$};
\coordinate (Y) at (-0.1,1.2);
\draw (sigma) to ([xshiftu=-.3]L);
\draw (tau) to (copier2);
\draw (copier2) to[out=165,in=-90] ([xshiftu=.3]L);
\draw (copier2) to[out=45,in=-90] (Y);
\end{tikzpicture}}}
\enspace.
\]
Then, writing $q\colon Y\to I$ for an almost inverse to the effect,
the channel
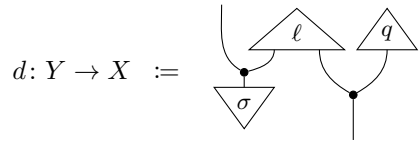
\[
d\colon Y\to X\;\;\coloneqq\quad
\vcenter{\hbox{%
\begin{tikzpicture}[font=\small]
\node[state,hflip] (L) at (0,0) {\;\,$\ell$\;\,};
\node[copier] (copier2) at (-0.7,-0.3) {};
\node[state] (tau) at (-0.7,-0.5) {$\sigma$};
\coordinate (Y) at (-1,0.6);
\coordinate (X) at (0.75,-1.2);
\node[copier] (copierY) at (0.75,-0.6) {};
\node[state,hflip] (q) at (1.2,0) {$q$};
\draw (X) to (copierY);
\draw (copierY) to[out=150,in=-90] ([xshiftu=.3]L);
\draw (copierY) to[out=30,in=-90] (q);
\draw (tau) to (copier2);
\draw (copier2) to[out=15,in=-90] ([xshiftu=-.3]L);
\draw (copier2) to[out=150,in=-90] (Y);
\end{tikzpicture}}}
\]
is a Bayesian inversion for $\sigma$ along $c\colon X\to Y$.  Namely,
together they satisfy the equation~\eqref{eq:bayesian-inversion}.
\end{theorem}

\begin{myproof}
We reason as follows.
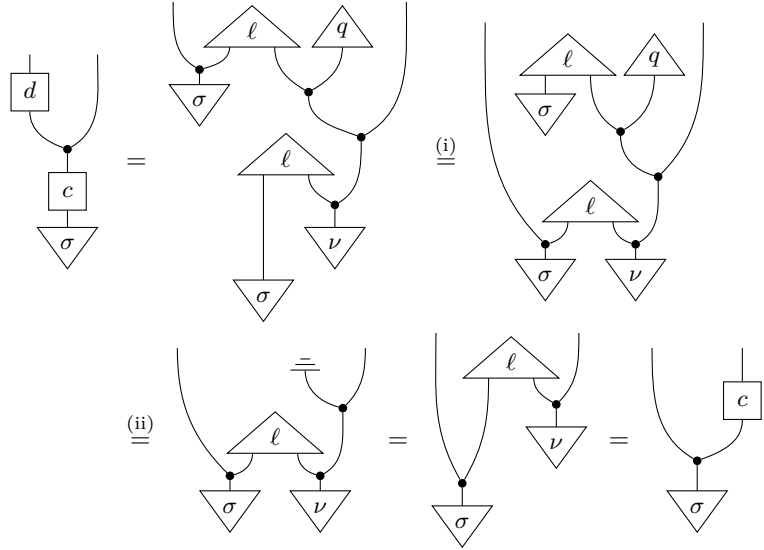
\begin{align*}
\vcenter{\hbox{%
\begin{tikzpicture}[font=\small]
\node[state] (sigma) at (0,0) {$\sigma$};
\node[arrow box] (tauL) at (0,0.48) {$c$};
\node[copier] (copier) at (0,1.04) {};
\node[arrow box] (sigmaL) at (-0.5,1.8) {$d$};
\coordinate (X) at (-0.5,2.3);
\coordinate (Y) at (0.4,2.3);
\draw (sigma) to (tauL);
\draw (tauL) to (copier);
\draw (copier) to[out=160,in=-90] (sigmaL);
\draw (sigmaL) to (X);
\draw (copier) to[out=30,in=-90] (Y);
\end{tikzpicture}}}
\;\;&=\;\;
\vcenter{\hbox{%
\begin{tikzpicture}[font=\small]
\node[state] (sigma) at (1.1,-3.9) {$\sigma$};
\node[copier] (copier) at (2.4,-2) {};
\coordinate (Y) at (3,-0.2) {};
\node[state,hflip] (L) at (1.4,-2.5) {\;\,$\ell$\;\,};
\node[copier] (copier2) at (2.05,-2.9) {};
\node[state] (tau) at (2.05,-3.2) {$\nu$};
\node[copier] (copierY) at (1.7,-1.4) {};
\draw (copier) to[out=160,in=-90] (copierY);
\draw (copier) to[out=30,in=-90] (Y);
\draw (sigma) to ([xshiftu=-.3]L);
\draw (tau) to (copier2);
\draw (copier2) to[out=150,in=-90] ([xshiftu=.3]L);
\draw (copier2) to[out=30,in=-90] (copier);
\node[state,hflip] (L) at (0.95,-0.8) {\;\,$\ell$\;\,};
\node[copier] (copier2) at (0.25,-1.1) {};
\node[state] (tau) at (0.25,-1.3) {$\sigma$};
\coordinate (X) at (-0.1,-0.2) {};
\node[state,hflip] (q) at (2.15,-0.8) {$q$};
\draw (copierY) to[out=150,in=-90] ([xshiftu=.3]L);
\draw (copierY) to[out=30,in=-90] (q);
\draw (tau) to (copier2);
\draw (copier2) to[out=15,in=-90] ([xshiftu=-.3]L);
\draw (copier2) to[out=150,in=-90] (X);
\end{tikzpicture}}}
\;\;\overset{(\mathrm{i})}{=}\;\;
\vcenter{\hbox{%
\begin{tikzpicture}[font=\small]
\node[state] (sigma) at (-0.6,1.7) {$\sigma$};
\node[copier] (copier) at (0.6,-0.3) {};
\node[state,hflip] (q) at (0.85,1.95) {$q$};
\coordinate (X) at (-1.4,2.65) {} {} {} {} {} {} {} {};
\coordinate (Y) at (1.5,2.65) {} {} {} {} {} {} {} {};
\node[state,hflip] (L) at (-0.3,1.95) {\;\,$\ell$\;\,};
\node[copier] (copier2) at (0.4,1.2) {};
\node[state] (tau) at (0.6,-0.5) {$\nu$};
\node[state,hflip] (L2) at (0,0) {\;\,$\ell$\;\,};
\node[state] (sigma2) at (-0.6,-0.5) {$\sigma$};
\node[copier] (copier3) at (0.9,0.6) {};
\node[copier] (copier4) at (-0.6,-0.3) {};
\draw (sigma) to ([xshiftu=-.3]L);
\draw (tau) to (copier);
\draw (copier2) to[out=150,in=-90] ([xshiftu=.3]L);
\draw (copier2) to[out=30,in=-90] (q);
\draw (sigma2) to (copier4);
\draw (copier4) to[out=135,in=-90] (X);
\draw (copier4) to[out=15,in=-90] ([xshiftu=-.3]L2);
\draw (copier) to[out=165,in=-90] ([xshiftu=.3]L2);
\draw (copier) to[out=30,in=-90] (copier3);
\draw (copier3) to[out=165,in=-90] (copier2);
\draw (copier3) to[out=45,in=-90] (Y);
\end{tikzpicture}}}
\\
&\overset{(\mathrm{ii})}{=}\;\;
\vcenter{\hbox{%
\begin{tikzpicture}[font=\small]
\node[copier] (copier) at (0.6,-0.3) {};
\coordinate (X) at (-1.3,1.4) {};
\coordinate (Y) at (1.2,1.4) {};
\node[state] (tau) at (0.6,-0.5) {$\nu$};
\node[discarder] (d) at (0.4,1.1) {};
\node[state,hflip] (L2) at (0,0) {\;\,$\ell$\;\,};
\node[state] (sigma2) at (-0.6,-0.5) {$\sigma$};
\node[copier] (copier3) at (0.9,0.6) {};
\node[copier] (copier4) at (-0.6,-0.3) {};
\draw (tau) to (copier);
\draw (sigma2) to (copier4);
\draw (copier4) to[out=135,in=-90] (X);
\draw (copier4) to[out=15,in=-90] ([xshiftu=-.3]L2);
\draw (copier) to[out=165,in=-90] ([xshiftu=.3]L2);
\draw (copier) to[out=30,in=-90] (copier3);
\draw (copier3) to[out=165,in=-90] (d);
\draw (copier3) to[out=45,in=-90] (Y);
\end{tikzpicture}}}
\;\;=\;\;
\vcenter{\hbox{%
\begin{tikzpicture}[font=\small]
\node[state] (sigma) at (0.15,0.3) {$\sigma$};
\node[copier] (copier) at (0.15,0.6) {};
\node[state,hflip] (L) at (0.8,2) {\;\,$\ell$\;\,};
\node [state] (tau) at (1.4,1.35) {$\nu$};
\node[copier] (copier2) at (1.4,1.65) {};
\coordinate (X) at (-0.2,2.6);
\coordinate (Y) at (1.7,2.6);
\draw (sigma) to (copier);
\draw (copier) to[out=120,in=-90] (X);
\draw (copier) to[out=60,in=-90] ([xshiftu=-.3]L);
\draw  (tau) to (copier2);
\draw  (copier2) to[out=150,in=-90] ([xshiftu=.3]L);
\draw  (copier2) to[out=45,in=-90] (Y);
\end{tikzpicture}}}
\;\;=\;\;
\vcenter{\hbox{%
\begin{tikzpicture}[font=\small]
\node[state] (sigma) at (0,0) {$\sigma$};
\node[copier] (copier) at (0,0.4) {};
\node[arrow box] (tauL) at (0.6,1.2) {$c$};
\coordinate (X) at (-0.6,1.9);
\coordinate (Y) at (0.6,1.9);
\draw (sigma) to (copier);
\draw (copier) to[out=150,in=-90] (X);
\draw (copier) to[out=20,in=-90] (tauL);
\draw (tauL) to (Y);
\end{tikzpicture}}}
\end{align*}
For the equality $\overset{(\mathrm{i})}{=}$
we use associativity and commutativity
of copiers $\copier$.
The equality $\overset{(\mathrm{ii})}{=}$
follows by:
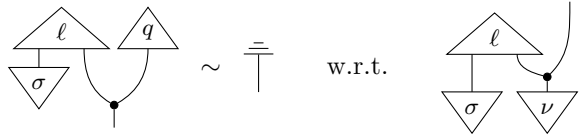
\[
\begin{ctikzpicture}[font=\small]
\node[state] (sigma) at (-0.6,1.7) {$\sigma$};
\node[state,hflip] (q) at (0.85,1.95) {$q$};
\coordinate (Y) at (0.4,0.9);
\node[state,hflip] (L) at (-0.3,1.95) {\;\,$\ell$\;\,};
\node[copier] (copier2) at (0.4,1.2) {};
\draw (sigma) to ([xshiftu=-.3]L);
\draw (copier2) to[out=150,in=-90] ([xshiftu=.3]L);
\draw (copier2) to[out=30,in=-90] (q);
\draw  (copier2) to (Y);
\end{ctikzpicture}
\;\;\equpto\;\;
\begin{ctikzpicture}[font=\small]
\node[discarder] (d) at (0,0) {};
\draw (d) -- (0,-0.5);
\end{ctikzpicture}
\qquad\text{w.r.t.}\qquad
\vcenter{\hbox{%
\begin{tikzpicture}[font=\small]
\node[state,hflip] (L) at (-1.1,0.5) {\;\,$\ell$\;\,};
\node[copier] (copier2) at (-0.4,0.2) {};
\node[state] (tau) at (-0.4,0) {$\nu$};
\node[state] (sigma) at (-1.4,0) {$\sigma$};
\coordinate (Y) at (-0.1,1.2);
\draw (sigma) to ([xshiftu=-.3]L);
\draw (tau) to (copier2);
\draw (copier2) to[out=165,in=-90] ([xshiftu=.3]L);
\draw (copier2) to[out=45,in=-90] (Y);
\end{tikzpicture}}}
\]
via equality strengthening.
\end{myproof}

\begin{example}
We instantiate the Theorem~\ref{thm:likelihood} in $\sfKrn$.  Let
$c\colon X\krnto Y$ be a probability kernel represented by a
likelihood relation $\ell\colon X\times Y\krnto 1$ with respect to
$\nu\colon 1\krnto Y$.  The relation $\ell$ is identified with a
measurable function $\ell\colon X\times Y\to[0,\infty]$ and $\nu$ with
a measure $\nu\colon\Sigma_Y\to[0,\infty]$.  The
equation~\eqref{eq:likelihood} amounts to
\[
c(x,B)=
\int_B \ell(x,y)\,\nu(\dd y)
\enspace.
\]
In particular, each $\ell(x,-)$
is a probability density function, satisfying
$\int_Y \ell(x,y)\,\nu(\dd y)=1$.
Typically, we use the Lebesgue measure as $\nu$,
with $Y$ a subspace of $\RR$.
Let $\sigma\colon 1\krnto X$ be a probability measure.
Then $c_*(\sigma)\colon 1\krnto Y$ is given as:
\[
c_*(\sigma)(B)
=\int_X c(x, B)\,\sigma(\dd x)
=\int_X \int_B \ell(x,y)\,\nu(\dd y)\,\sigma(\dd x)
\]
The effect
\[
p\colon Y\krnto 1\quad=\quad
\begin{ctikzpicture}
\node[state,hflip] (L) at (0,0) {\;\,$\ell$\;\,};
\node[state] (tau) at (-0.3,-0.3) {$\sigma$};
\coordinate (X) at (0.3,-1);
\draw (X) to ([xshiftu=.3]L);
\draw (tau) to ([xshiftu=-.3]L);
\end{ctikzpicture}
\qquad
\text{is given as:}
\qquad
p(y)= \int_X \ell(x,y)\,\sigma(\dd x)
\enspace.
\]
To define an inverse of $p$,
we claim that $0<p<\infty$, $c_*(\sigma)$-almost everywhere.
We prove that $p^{-1}(\set{0,\infty})=p^{-1}(0)\cup p^{-1}(\infty)$
is $c_*(\sigma)$-negligible, as:
\begin{align*}
c_*(\sigma)\paren[\big]{p^{-1}(0)}
&=\int_X \int_{p^{-1}(0)} \ell(x,y)\,\nu(\dd y)\,\sigma(\dd x)
\\
&=\int_{p^{-1}(0)} p(x)\,\nu(\dd y)
\\
&=\int_{p^{-1}(0)} 0\,\nu(\dd y) =0
\end{align*}
and, similarly we have
\[
\int_{p^{-1}(\infty)} \infty \,\nu(\dd y)
= \int_{p^{-1}(\infty)} p(x) \,\nu(\dd y)
= c_*(\sigma)\paren[\big]{p^{-1}(\infty)} \le 1
\]
but this is possible only when
$\nu(p^{-1}(\infty))=0$,
hence $c_*(\sigma)\paren[\big]{p^{-1}(\infty)}=\int_{p^{-1}(\infty)} p(x) \,\nu(\dd y)=0$.
Now define an effect $q\colon Y\krnto 1$ by
\[
q(y)=
\begin{dcases*}
p(y)^{-1} & if $0<p(y)<\infty$ \\
0 & otherwise.
\end{dcases*}
\]
Then $p$ is $c_*(\sigma)$-almost inverted by $q$.
By Theorem~\ref{thm:likelihood}, the Bayesian inversion
for $\sigma$ along $c$ is given by
\[
d\colon Y\to X\;\;\coloneqq\quad
\vcenter{\hbox{%
\begin{tikzpicture}[font=\small]
\node[state,hflip] (L) at (0,0) {\;\,$\ell$\;\,};
\node[copier] (copier2) at (-0.7,-0.3) {};
\node[state] (tau) at (-0.7,-0.5) {$\sigma$};
\coordinate (Y) at (-1,0.6);
\coordinate (X) at (0.75,-1.2);
\node[copier] (copierY) at (0.75,-0.6) {};
\node[state,hflip] (q) at (1.2,0) {$q$};
\draw (X) to (copierY);
\draw (copierY) to[out=150,in=-90] ([xshiftu=.3]L);
\draw (copierY) to[out=30,in=-90] (q);
\draw (tau) to (copier2);
\draw (copier2) to[out=15,in=-90] ([xshiftu=-.3]L);
\draw (copier2) to[out=150,in=-90] (Y);
\end{tikzpicture}}}
\quad,
\]
namely,
\begin{equation}
\label{eq:inversion-by-likelihood}
\begin{aligned}
d(y,A)
&=q(y) \int_A \ell(x,y) \,\sigma(\dd x)
\\
&=
\frac{\int_A \ell(x,y) \,\sigma(\dd x)}{\int_X \ell(x,y) \,\sigma(\dd x)}
\qquad\text{whenever}\quad
0<\int_X \ell(x,y) \,\sigma(\dd x)<\infty
\end{aligned}
\end{equation}
This may be seen as a variant of the Bayes formula.
The calculation in Example~\ref{ex:inversion} is reproduced when
$\sigma$ is also given via a density function.
\end{example}

In the end we reconsider the naive Bayesian classification example
from Section~\ref{sec:classifier}. There we only considered the
discrete version. The original source~\cite{WittenFH11} also contains
a `hybrid' version, combining discrete and continuous probability.

In that hybrid form the Temperature and Humidity columns in
Figure~\ref{fig:discreteplay} are different, and are given by
numerical values. What these values are does not matter too much here,
since they are only used to calculate \emph{mean} ($\mu$) and
\emph{standard deviation} ($\sigma$) values. This is done separately
for the cases where Play is \emph{yes} or \emph{no}. It results in
the following table.
\begin{center}
{\setlength\tabcolsep{2em}\renewcommand{\arraystretch}{1.0}
\begin{tabular}{c|c|c}
& \textbf{Temperature} & \textbf{Humidity} 
\\
\hline\hline
\textbf{Play} = yes 
   & $\mu = 73, \, \sigma = 6.2$
   & $\mu = 79.1, \, \sigma = 10.2$
\\
\textbf{Play} = no
   & $\mu = 74.6, \, \sigma = 7.9$
   & $\mu = 86.2, \, \sigma = 9.7$
\\
\end{tabular}}
\end{center}

\noindent These $\mu,\sigma$ values are used to define two functions
$c_{T} \colon P \rightarrow \Giry(\mathbb{R})$ and $c_{H} \colon P
\rightarrow \Giry(\mathbb{R})$, with normal distributions
$\mathcal{N}$ as pdf. To be precise, these channels are defined as
follows, where $A\subseteq\mathbb{R}$ is a measurable subset.
\[ \begin{array}{rclcrcl}
c_{T}(y)(A)
& = &
\displaystyle\int_{A} \mathcal{N}(73, 6.2)(x) \intd x
& \qquad\qquad &
c_{H}(y)(A)
& = &
\displaystyle\int_{A} \mathcal{N}(79.1, 10.2)(x) \intd x
\\[+.8em]
c_{T}(n)(A)
& = &
\displaystyle\int_{A} \mathcal{N}(74.6, 7.9)(x) \intd x
& &
c_{H}(y)(A)
& = &
\displaystyle\int_{A} \mathcal{N}(86.2, 9.7)(x) \intd x.
\end{array} \]

\noindent These functions $c_T$ and $c_H$ form channels for the
Giry monad $\Giry$. 

The two columns Outlook and Windy in Figure~\ref{fig:discreteplay}
remain the same. The corresponding maps $d_{O} \colon P \rightarrow
\Dst(O)$ and $d_{W} \colon P \rightarrow \Dst(W)$ for the discrete
probability monad $\Dst$ --- in Diagram~\eqref{diag:discreteplay} ---
are now written as maps $c_{O} \colon P \rightarrow \Giry(O)$ and
$c_{W} \colon P \rightarrow \Giry(W)$ for the monad $\Giry$, using the
obvious inclusion $\Dst \hookrightarrow \Giry$.

We now have four channels $c_{O} \colon P \rightarrow O$, $c_{T}
\colon P \rightarrow \mathbb{R}$, $c_{H} \colon P \rightarrow
\mathbb{R}$ and $c_{W} \colon P \rightarrow W$. We combine them, as
before, into a single channel $c\colon P \rightarrow
T\times\mathbb{R}\times\mathbb{R}\times W$. In combination with the
Play marginal distribution $\pi$ from Section~\ref{sec:classifier}, we
can compute $c$'s Bayesian inversion $f\colon
T\times\mathbb{R}\times\mathbb{R}\times W \rightarrow \Giry(P)$
via the formula \eqref{eq:inversion-by-likelihood}.  We
apply it to the input data used in~\cite{WittenFH11}, and get the
following Play distribution.
\[ \begin{array}{rcl}
f(s,66,90,t)
& = &
0.207\ket{y} + 0.793\ket{n}.
\end{array} \]


\noindent The latter inversion computation produces the probability of
$0.207$ for playing when the outlook is \emph{Sunny}, the temperature
is \emph{66} (Fahrenheit), the humidity is \emph{90\%} and the
windiness is \emph{true}. The value computed in~\cite{WittenFH11} is
$20.8\%$. The minor difference of $0.001$ with our outcome can be
attributed to (intermediate) rounding errors.

Our computation is done via \EfProb, using the
formula~\eqref{eq:inversion-by-likelihood}.

\xhead*{Acknowledgements}

The research
leading to these results has received funding from the European
Research Council under the European Union's Seventh Framework
Programme (FP7/2007-2013) / ERC grant agreement nr.~320571.
The first author (KC) is also supported by
ERATO HASUO Metamathematics for Systems Design Project
(No.~JPMJER1603), JST\@.
The main part of this work was done when the first author
was a PhD student at Radboud University, Nijmegen.


\end{document}